\documentclass[11pt]{article}
\pdfoutput=1

\usepackage[colorlinks=true,linkcolor=black,citecolor=black,urlcolor=black,breaklinks]{hyperref}

\hypersetup{
colorlinks = true,
linkcolor=black
}

\usepackage{custom_tex}

\begin{document}


\title{Optimal Multitask Linear Regression and Contextual  Bandits under Sparse Heterogeneity}

 \author{
 Xinmeng Huang\footnote{
    Graduate Group in Applied Math and Computational Science, Univ. of Pennsylvania. \texttt{xinmengh@sas.upenn.edu}.} \hfill
    Kan Xu\footnote{Department of Information Systems, Arizona State University. \texttt{kanxu1@asu.edu}.} \hfill 
    Donghwan Lee\footnote{Graduate Group in Applied Math and Computational Science, Univ. of Pennsylvania. \texttt{dh7401@sas.upenn.edu}.}\vspace{4mm}\\
 Hamed Hassani\footnote{Department of Electrical and Systems Engineering, Univ. of Pennsylvania. \texttt{hassani@seas.upenn.edu}.} \hspace{6mm}
 Hamsa Bastani\footnote{Department of Operations, Information, and Decisions, Univ. of Pennsylvania. \texttt{hamsab@wharton.upenn.edu}.}\hspace{6mm}
 Edgar Dobriban\footnote{Department of Statistics and Data Science, Univ. of Pennsylvania. \texttt{dobriban@wharton.upenn.edu}.}
 }

\date{\today}

\maketitle

\begin{abstract}
Large and complex datasets are often collected from several, possibly heterogeneous sources. 
Multitask learning methods improve efficiency by leveraging commonalities across datasets while accounting for possible differences among them. Here, we study multitask linear regression and contextual bandits under \emph{sparse heterogeneity}, where the source/task-associated parameters are equal to a global parameter plus a sparse task-specific term.  
We propose a novel two-stage estimator called \ours that leverages this structure by first constructing a  covariate-wise weighted median of the task-wise linear regression estimates and then shrinking the task-wise estimates towards the weighted median. 
Compared to task-wise least squares estimates,
\ours improves the dependence of the estimation error on the data dimension. Extensions of \ours to generalized linear models and constructing confidence intervals are discussed in the paper.
We then apply \ours to develop methods for sparsely heterogeneous multitask contextual bandits, obtaining improved regret guarantees over single-task bandit methods.
We further show that our methods are
minimax optimal by providing a number of lower bounds.
Finally, we support the efficiency of our methods by performing experiments on both synthetic data
and the PISA dataset on student educational outcomes from heterogeneous countries.
\end{abstract}

\tableofcontents
\medskip

\section{Introduction}
Large and complex datasets are often collected from multiple sources---such as from several locations---and with possibly varying data collection methods~\citep{gupta2018distributed,yuan2022revisiting}. This can result in both similarities and differences among the source-specific datasets. While some covariates have consistent effects on the response across all sources, others may have different effects on the response depending on the source. For instance, when predicting students' academic performance, the effects of socioeconomic status, language, and education policies may vary across regions. See Figure \ref{fig:coef_plt} for an illustration of the multi-country PISA educational attainment dataset \citep{oecd2019teaching}, described in more detail in Section~\ref{pisa}, where the regression coefficients of a few features vary strongly across countries \citep{oecd2019teaching}.
Ignoring this heterogeneity may introduce bias and lead to incorrect predictions and decisions. Similar situations arise in healthcare \citep{quinonero2008dataset}, demand prediction \citep{baardman2020detecting, van2012relationship}, and others. 
Therefore, instead of learning a shared model for all the data, it may help to develop models for each task.

\begin{figure}[!h]
    \centering
    \includegraphics[width=0.8\textwidth]{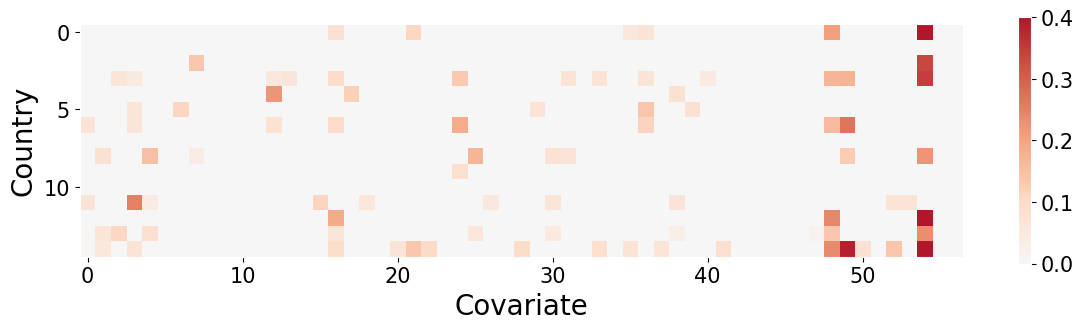}
    \caption{The differences in the least squares estimates of a measure of educational attainment in selected countries for the PISA dataset. See details in Appendix \ref{app:experi}.}
    \label{fig:coef_plt}
\end{figure}

Although complex datasets are often heterogeneous, commonalities also exist across data of the same type. 
Therefore, it may increase efficiency if we analyze them jointly.
This is often referred to as {\em multitask} analysis. 
Common multitask methods regularize the task-associated parameters via penalties \citep{evgeniou2004regularized,evgeniou2005learning,duan2022adaptive}, as well as cluster or pool datasets based on their similarity \citep{ben2010theory,crammer2008learning,Dobriban2018DistributedLR}. 
While these methods have demonstrated
improvements by factors equal to the number of tasks, these methods can perform significantly worse than single-task learning when the heterogeneity across tasks is severe.
To achieve larger theoretical improvements,  \cite{Tripuraneni2021ProvableMO,Du2021FewShotLV,Collins2021exploitingSR} 
consider  low-dimensional shared representations of task-specific models, 
while \cite{lounici2009taking,singh2020online} constrain the task-specific parameters to be sparse with a common support set.

Often, only a small subset of covariates have different effects in different data sources or in predicting different responses.
For instance, 
in the Expedia personalized recommendation dataset,
Figure 2 of \citet{bastani2021predicting} shows that 
out of $15$ customer- and hotel-specific features,
only the price has significantly different effects on predicting bookings and clicks.
To capture this phenomenon in a broader multitask regime, 
we consider the following sparse heterogeneity structure.
Given $M$ tasks, each task $m$ is associated with a parameter $\beta^{(m)}$, and $\{\beta^{(m)}\}_{m=1}^M$ differ only in a small number of coordinates.
In particular, they have the form $\beta^{(m)}=\beta^\star +\delta^{(m)}$,
for some unknown global parameter $\beta^\star$ representing the part of the parameters shared across tasks; and for unknown sparse parameters $\delta^{(m)}$---with few nonzero covariates---representing task-specific adjustments. 

This structure has garnered interest in a number of prior works: \cite{bastani2021predicting} combined a large proxy dataset and a small target dataset when the associated parameters differed in a sparse vector; 
\cite{xu2021group} considered group-wise sparse heterogeneity in matrix factorization of word embeddings;
while \cite{xu2021multitask} proposed methods for multiple sparsely heterogeneous contextual bandits. 

However, prior work leaves open a number of important problems.
In particular, it is not clear what the statistically optimal methods are under sparse heterogeneity, even in linear regression.
Prior work has shown that heterogeneity-aware methods improve the dimension-dependence of the estimation error rate, compared to heterogeneity-unaware or single-task methods.
However, it remains unclear whether better methods exist.  
In this paper, we resolve this problem by establishing---to our knowledge---a new lower bound for estimating several models with sparse heterogeneity, in both linear regression and contextual bandits. 
We also propose novel methods and show that they achieve the lower bounds.

\subsection{Contributions}
We consider estimating $M$ linear models under $s$-sparse heterogeneity, in both the offline and online scenarios. 
We highlight our contributions as follows:
\begin{itemize}
    \item For linear regression, we propose the \ours method---Median-based Multitask Estimator for Linear Regression---to jointly learn heterogeneous models: \ours first applies a weighted covariate-wise median to source-specific least squares regression estimators 
    to obtain an estimate shared across tasks, 
    and then obtains estimates of individual tasks by shrinking the individual estimates towards the shared estimate via thresholding.
    We provide upper bounds on the estimation error for each individual task using \ours, 
    showing the benefit of multiple tasks (large total sample size $n_{[M]}$) and the sparse heterogeneity (small $s$).
    For balanced datasets, our result improves the rates of pooled ordinary least squares (OLS) and the multitask method of \cite{xu2021multitask}  by factors of $d/s$ and $\sqrt{d/s}$, respectively.     Beyond 
    estimation in linear models,  the applications of \ours to generalized linear models, and statistical inference for task-wise parameters are then discussed.
    
    We also provide matching minimax lower bounds in the multitask setting with sparse heterogeneity, showing the optimality of \ours. 
    Our result generalizes the lower bound for single-task linear regression to allow data points with different noise variances, strengthening the one for sparse linear regression~\citep{raskutti2011minimax}.

    \item We formulate the asynchronous multitask setting for two types of contextual bandit problems 
    where each bandit instance has a certain probability of observing a context and taking an action at any time.
    We use \ours in the multitask bandit setting, leading to improved individual regret bounds, where the scaling in the context dimension $d$ from single-task methods is replaced by the level of heterogeneity. 
    In addition, we provide new minimax lower bounds for asynchronous multitask bandit problems.

    \item We support our methods with experiments on both synthetic data and the PISA educational attainment dataset. 
    Our empirical results support our theoretical findings, and show an improvement over existing methods.

\end{itemize}

\subsection{Related Works}
 We review the most closely related works here. 
 More literature is reviewed in Appendix~\ref{app:ref}.

\paragraph{Data Heterogeneity.}
Complex datasets are often obtained by aggregating data from heterogeneous sources; which may correspond to subpopulations 
with unique characteristics \citep{fan2014challenges, meinshausen2015maximin,marron2017big}.
Data heterogeneity can reduce the performance of standard methods 
designed for independent and identically distributed (i.i.d.) data 
in statistical inference \citep{guo2020inference,hu2022collaborative,Yuan2021RemovingDH} 
and various learning tasks \citep{zhao2016partially,gu2022weighted,mcmahan2017communication}. 
However,
it is sometimes possible to mitigate the effects of heterogeneity \citep{luo2022odach,yang2020analysis,zhang2019optimal,chen2022distributed,wang2019additive}.

\paragraph{Multitask Linear Regression \& Contextual Bandits. }
Studying multitask linear regression, \cite{Tripuraneni2021ProvableMO,Du2021FewShotLV} show improved generalization errors compared to the single-task  OLS, 
by considering a low-dimensional shared representation. 
A similar result holds for personalized federated learning \citep{Collins2021exploitingSR}. 
\cite{yang2019high} assume group-wise heterogeneity of parameters and propose to regularize the least squares objective, without a finite-sample theoretical analysis.

Starting from \cite{woodroofe1979one}, literature on contextual bandits has developed vastly 
\citep[see e.g.,][etc]{sarkar1991one,yang2002randomized,perchet2013multi,chen2021statistical,chen2022online,chen2022nearly,luo2022contextual}.
Multitask contextual bandit 
methods include regularizing the bandit instance parameters, 
and pooling data from related bandit instances \citep[see \eg,][]{soare2014multi, chu2011contextual,valko2013finite,cesa2013gang,deshmukh2017multi,gentile2014online,gentile2017context}.
One can also impose a shared prior distribution over bandit instances \citep{cella2020meta,kveton2021meta,bastani2022meta}. 
However, most resulting
regret bounds  for individual bandit instances can counter-intuitively deteriorate 
and can be worse than for single-bandit methods. 
Furthermore,
some methods require bandit instances to appear sequentially to learn the prior~\citep{lazaric2013sequential}.

Recently, \cite{xu2021multitask} propose methods with improved estimation error in linear regression and regret in contextual bandits, when tasks are sparsely heterogeneous. 
Our matching upper and lower bounds imply that their method is sub-optimal.

\paragraph{Transfer Learning.}  
Transfer learning aims to boost the learning performance on a particular task (typically with a small dataset) when given data from other sources. 
A few works have studied transfer learning in linear regression \citep{li2022transfer} and generalized linear models \citep{tian2022transfer,li2023estimation}, but mainly for parameters with $\ell_1$-bounded heterogeneity.

\subsection{Notation}
We introduce necessary notations here and refer to the full definitions in Appendix~\ref{app:ref}.
We use $:=$ or $\triangleq$ to introduce definitions.
For an integer $d\ge 1$, we write $[d]$ for $\{1,\dots,d\}$. We use $I_{d}$ to denote the $d\times d$ identity matrix. 
For a vector $v\in\RR^d$, we denote its entries as  $v_1,\dots,v_d$.
We also denote 
$\smash{\|v\|_p = (\sum_{k\in[d]} |v_k|^p)^{1/p}}$ for all $p>0$, with $\|v\|_0$ defined as the number of non-zero entries.
For any $\cI\subseteq [M]$, given weights $\{w_m\}_{m=1}^M$ (or sample sizes $\{n_m\}_{m=1}^M$), we write $W_{\cI}$ as $\sum_{m\in\cI}w_m$ and $n_{\cI}$ for $\sum_{m\in\cI}n_m$.
For a matrix $A\in\RR^{m\times n}$, we denote the $(i,j)$-th covariate of $A$ by  $[A]_{i,j}$ or $A_{i,j}$,
and the $i$-th row (resp., the $j$-th column) by $A_{i,\cdot}$ (resp., $A_{\cdot, j}$).
For two real numbers $a$ and $b$, we write $a\vee b$ and  $a\wedge b$ for $\max\{a,b\}$ and $\min\{a,b\}$, respectively.
For an event $E$, we write $\one(E)$ for the indicator of the event.
We use the Bachmann-Landau asymptotic notations $\Omega(\cdot)$, $\Theta(\cdot)$, $O(\cdot)$ to absorb constant factors, and use $\tilde{\Omega}(\cdot)$, $\widetilde{O}(\cdot)$ to also absorb logarithmic factors. 
Furthermore, we use probabilistic notations such as $O_P(a_{\{n_m\}_{m=1}^M})$ to denote quantities that are bounded by $a_{\{n_m\}_{m=1}^M}$ with overwhelming probabilities as $\min_{m\in[M]}{n_m}\to\infty$.
For a number $x\in\RR$, we use $(x)_{+}$ to denote its non-negative part, \ie, $x\one(x\geq 0)$.

\section{Multitask Linear Regression}\label{sec:multi-lr}
In this section, we study multitask linear regression under sparse heterogeneity with $M>0$ tasks, each of which
is associated  with an unknown task-specific parameter $\beta^{(m)}\in\RR^d$ for $m\in[M]$.
For a covariate vector $x^{(m)}_i$ associated with task $m$, 
the outcome follows the linear model
\begin{equation*}
    y_i^{(m)}=\langle x_i^{(m)}, \beta^{(m)}\rangle+\ep_i^{(m)},
\end{equation*}
where $\ep_i^{(m)}$ is noise satisfying conditions specified below.
We observe $n_m > 0$ i.i.d.~vectors of covariates for each task $m$;
the sample sizes $n_m$ may vary with $m$.
For each task $m$, we denote by $\vX^{(m)}\in\RR^{n_m\times d}$ the matrix whose rows are the observed features, 
and by $Y^{(m)}\in\RR^{n_m}$ the vector of corresponding observed outcomes.
Our goal is to use $\{(\vX^{(m)},Y^{(m)})\}_{m=1}^M$ to estimate the parameter $\beta^{(m)}$ for each task $m$. 
We denote by $\vX$ the $n_{[M]}\times d$ 
concatenation of all matrices $\{\vX^{(m)}\}_{m=1}^M$, where $n_{[M]}\triangleq \sum_{m=1}^Mn_m$ denotes the total sample size.
We consider the following condition,
which will enable us to pool information to improve estimation: 
\begin{assumption}[\sc Sparse heterogeneity]\label{asp:heterogneity}
There exists an \emph{unknown} $s \in \{0,\ldots, d\}$ and a common global parameter $\beta^\star\in\RR^d$ such that for each $m\in[M]$, $\|\beta^{(m)}-\beta^\star\|_0\le s$.
\end{assumption}
We remark that while a common upper bound $s$ is assumed in Condition~\ref{asp:heterogneity}, most of our methods also adapt to the case where the $\ell_0$-heterogeneity between $\beta^{(m)}$ and  $\beta^\star$ 
differs across tasks \ie, $\|\beta^{(m)}-\beta^\star\|_0\leq s_m$, and results depend on the individual values of $s_m$ as opposed to just their maximum; see Remark \ref{rmk:varying} for more details.
The heterogeneity level $s$ (or $\{s_m\}_{m=1}^M$) and the global parameter $\beta^\star\in\RR^d$ may not be identifiable, but our task-wise results apply simultaneously to all possible choices of $\{\beta^{(m)}\}_{m=1}^M$ and $s$ (or $\{s_m\}_{m=1}^M$) satisfying Condition~\ref{asp:heterogneity}. 
We will also require several additional standard assumptions. 
In particular,  in the main text, we consider Gaussian noise to simplify the analysis (Condition \ref{asp:noise}). We also provide similar results for \emph{Orlicz-norm-bounded noise}, covering sub-Gaussian and sub-exponential noises, in Appendix \ref{app:general-noise} at the cost of an auxiliary high-order term in $n$ in our rates.

\begin{assumption}[\sc Gaussian noise]\label{asp:noise}
For each $m\in[M]$, 
and some $\sigma_m\ge 0$, $m\in[M]$,
the noises $\{\ep_i^{(m)}\}_{i=1}^{n_m}$ are i.i.d. random variables with $\ep_i^{(m)}\sim \cN(0,\sigma_m^2)$.
\end{assumption}

We next introduce a customary condition on task-wise distributions.

\begin{assumption}[\sc Covariate distribution]\label{asp:design-mats}
There are constants $L\geq \mu>0$ such that 
for any $m\in[M]$, $\{x^{(m)}_i\}_{i=1}^{n_m}$  are the $L$-sub-Gaussian and  independently distributed covariates with zero mean and covariance $\Sigma^{(m)}\succeq \mu  I_{d}$,  where $\succeq$ denotes the Loewner order.
\end{assumption}
Throughout the paper, we focus on investigating the impact of sample sizes $\{n_m\}_{m=1}^M$, variances $\{\sigma_m^2\}_{m=1}^M$, dimension of covariates $d$ and heterogeneity $s$, while omitting the dependence on the other problem characteristics, including $\mu$ and $L$.
When leveraging heterogeneous datasets, 
the performance can be hampered by very small and noisy datasets
\citep[e.g.,][etc]{akbani2004applying,kotsiantis2006handling,chawla2010data}. 
To prove optimal estimation errors, we will require Condition \ref{asp:size} to mildly constrain the sample sizes of the input datasets.

\begin{assumption}[\sc Sample size constraint]\label{asp:size}
We assume $n_m\geq c\ln(d)d$ for a sufficiently large constant $c$.
For the case of homogeneous variances $\sigma_1^2=\cdots=\sigma_M^2$, there is a constant $c_{\rm s}\geq 1$ such that  $\ln((d/s) \wedge(n_{[M]}/\min_{m\in[M]}n_m))^2\max_{m\in[M]}n_m\leq c_{\rm s}n_{[M]}$ and $\max_{m\in[M]}\sqrt{n_m}\sum_{m=1}^M\sqrt{n_m} \leq c_{\rm s}n_{[M]}$. For the case of different variances, we require the same inequalities to hold for the rescaled sample sizes $\tilde n_m\triangleq n_m/\sigma_m^2$.
\end{assumption}
Condition \ref{asp:size} is satisfied by the ideal balanced case where $n_m=\Theta(n)$ and $\sigma_m=\Theta(\sigma)$ for some $n>0$ and $\sigma>0$, which implies   $\ln(n_{[M]}/\min_{m\in[M]}n_m)^2\max_{m\in[M]}n_m= O(\ln(M)^2n_{[M]}/M)$ and $\max_{m\in[M]}\sqrt{n_m}\sum_{m=1}^M\sqrt{n_m}= O(n_{[M]})$, leading to $c_{\rm s}=O(1)$. However, Condition~\ref{asp:size} also covers more skewed sample sizes, including the singly dominant case where $n_1=\Theta(nM^c)$ with some $c\in[0,1)$ and $n_m=\Theta(n)$ for $2\leq m\leq M$. 

\subsection{Algorithm Overview} 
\begin{algorithm}[t]
	\caption{\ours: Weighted-median-based Multitask Linear Regressors}
	\label{alg:molar}
	\begin{algorithmic}
		\STATE \noindent {\bfseries Input:} $\{(\vX^{(m)},Y^{(m)})\}_{m=1}^M$, thresholds $\{\gamma_m\}_{m=1}^M$, weights $\{w_m\}_{m=1}^M$
		\FOR{$m\in[M]$}
		\vspace{1mm}
		\STATE Let $\widehat{\beta}_{\mathrm{ind}}^{(m)}=(\vX^{(m)\top} \vX^{(m)})^{-1}\vX^{(m)\top} Y^{(m)}$ be the OLS estimator for dataset $(\vX^{(m)},Y^{(m)})$ 
		\vspace{1mm}
		\ENDFOR
	    \vspace{1mm}
		\STATE Let $\widehat{\beta}^\star = \wmed(\{\widehat{\beta}_{\mathrm{ind}}^{(m)}\}_{m=1}^M; \{w_m\}_{m=1}^M)$ be the covariate-wise weighted median
  \vspace{1mm}
		\FOR{$m\in[M]$ and $k\in[d]$}
		    \vspace{2mm}
                \STATE \textsf{/* Option I: hard thresholding */}
                \vspace{1mm}
		        \STATE $\widehat{\beta}_{{\rm MOLAR}, k}^{(m)}=\widehat{\beta}^\star_k$ \textbf{ if } $|\widehat{\beta}^\star_k-\widehat{\beta}^{(m)}_{\mathrm{ind},k}|\le \gamma_m\sqrt{[(\vX^{(m)\top} \vX^{(m)})^{-1}]_{k,k}} $ \textbf{ else } $\smash{\widehat{\beta}^{(m)}_{\mathrm{ind},k}}$  
		        \vspace{2mm}
          \STATE \textsf{/* Option II: soft thresholding */}
                \vspace{1mm}
		        \STATE $\widehat{\beta}_{{\rm MOLAR}, k}^{(m)}=\widehat{\beta}^\star_{k}+\mathsf{SoftThresholding}(\widehat{\beta}^{(m)}_{\mathrm{ind},k}-\widehat{\beta}^\star_k;\gamma_m\sqrt{[(\vX^{(m)\top} \vX^{(m)})^{-1}]_{k,k}} )$
		        \vspace{2mm}
		\ENDFOR
		\STATE \noindent {\bfseries Output:} $\{\widehat{\beta}_{\rm MOLAR}^{(m)}\}_{m=1}^M$
	\end{algorithmic}
\end{algorithm}

We now introduce the \ours method, which consists of two steps: \emph{collaboration} and \emph{covariate-wise shrinkage}. 
Here,  \ours mainly concerns the case where $n_m\gg d$ for all $m\in[M]$ in the main text and a variant applicable to $n_m< d$ is provided in Appendix~\ref{app:high-dim}.
In the first step, we estimate the global parameter $\beta^\star$ via the covariate-wise weighted median of the OLS estimates 
$\{\widehat{\beta}^{(m)}_\mathrm{ind}\}_{m=1}^M$,
where the weights are adjusted according to the sample sizes and noise variances of each task. 
Recall that a weighted median $\wmed(\{z_m\}_{m=1}^M; \{w_m\}_{m=1}^M)$ of the scalar variables $\{z_m\}_{m=1}^M$ and non-negative weights $\{w_m\}_{m=1}^M$ (not necessarily summing up to one) 
is defined as any of the minimizers to 
the function $z\mapsto\sum_{m=1}^M w_m|z-z_m|$. 
Scaling all weights $\{w_m\}_{m=1}^M$ by a common factor leads to the same weighted median. 
In particular, if $w_1=\cdots=w_M$, the weighted median recovers the classical median and thus the weights can be omitted for clarity.

There are two key insights in this step. 
First, since the heterogeneity is sparse,
for most coordinates, 
most OLS estimates are unbiased for $\beta^\star$. 
Hence, to estimate those covariates of the global parameter, 
we view the local estimates as potentially perturbed by outliers 
and leverage robust statistical methods to mitigate their heterogeneity-incurred influence. 
Second, the weighting mechanism allows OLS estimates from datasets with larger sizes and less noise to contribute more to the global estimate $\widehat \beta^\star$. 
Notably,
our work also presents a novel non-asymptotic analysis of weighted medians, which can be of independent interest.

In the second step---covariate-wise shrinkage---we detect mismatched covariates between task-wise OLS estimates $\{\widehat{\beta}^{(m)}_{\mathrm{ind}}\}_{m=1}^M$ and the global estimate $\widehat{\beta}^\star$.
Recall that the global and task-wise parameters differ in only a few coordinates. 
For covariates $k\in[d]$ such that $|\widehat{\beta}^\star_k-\widehat{\beta}^{(m)}_{\mathrm{ind},k}|$ is below
$\gamma_m \sqrt{v^{(m)}_k}$ 
for some threshold $\gamma_m$ and $v^{(m)}_k\triangleq \sqrt{[(\vX^{(m)\top}\vX^{(m)})^{-1}]_{k,k}}$, we may expect that $\beta^{(m)}_k=\beta^\star_k$.
Also, the estimate $\widehat{\beta}^\star_k$ of $\beta^\star_k$ can be more accurate than $\widehat{\beta}^{(m)}_k$ for $\beta^{(m)}_k$,
as it is estimated collaboratively from multiple datasets.
As a result,  we can assign $\widehat{\beta}^\star_k$ as the final estimate $\widehat{\beta}^{(m)}_{\mathrm{MOLAR},k}$ of $\beta^{(m)}_k$ if $|\widehat{\beta}^\star_k-\widehat{\beta}^{(m)}_{\mathrm{ind},k}| \le 
\gamma_m \sqrt{v^{(m)}_k}$.
On the other hand, for the covariates where the global and local parameters differ, \ie, $\beta^{(m)}_k\neq \beta_k^\star$,  the threshold is more likely to be exceeded. 
In this case, 
we keep the local estimate $\smash{\widehat{\beta}^{(m)}_{\mathrm{ind},k}}$ as $\smash{\widehat{\beta}^{(m)}_{\mathrm{MOLAR},k}}$. 
This can be viewed as shrinkage of $\smash{\widehat{\beta}^{(m)}_{\mathrm{ind},k}}$ towards 
$\widehat{\beta}^\star_k$ 
via hard thresholding.
Alternatively, to allow a smooth transition, we can also conduct a soft thresholding step where the final estimate $\smash{\widehat{\beta}_{{\rm MOLAR},k}^{(m)}}$ shifts $\widehat{\beta}_{{\rm ind},k}$ by $\gamma_m\sqrt{v_k^{(m)}}$ when $|\widehat{\beta}^\star_k-\widehat{\beta}^{(m)}_{\mathrm{ind},k}| > \gamma_m \sqrt{v^{(m)}_k}$. The latter can be viewed as shrinkage via the soft thresholding operator $x\mapsto\textsf{SoftThresholding}(x;\lambda):=\mathrm{sign}(x)(|x|-\lambda)_+$ defined for any $x\in\RR$ and $\lambda\ge 0$.
The two options have the same theoretical guarantees but can differ slightly in practice.

While using two stages, 
our method does not require sample splitting. 
\ours  requires neither the knowledge of the heterogeneity bounds $\|\beta^{(m)}-\beta^\star\|_0$  nor of the  support sets  of  $\{\beta^{(m)}-\beta^\star\}_{m=1}^M$.

\subsection{Analysis of Estimation Error under Gaussian Noise}\label{sec:ana-gau}
In this subsection, we provide theoretical results for \ours for a Gaussian noise; more general noise is considered in the Appendix.
The local OLS estimates $\widehat{\beta}_{\mathrm{ind}}^{(m)}$ 
can be written as
$\smash{\widehat{\beta}_{\mathrm{ind}}^{(m)}=(\vX^{(m)\top} \vX^{(m)})^{-1}\vX^{(m)\top} Y^{(m)}= {\beta}^{(m)}+(\vX^{(m)\top} \vX^{(m)})^{-1}\vX^{(m)\top} \boldsymbol{\ep}^{(m)}}$, 
where $\boldsymbol{\ep}^{(m)}\sim \cN(0, \sigma_m^2 I_{n_m})$. 
Thus, denoting $[(\vX^{(m)\top}\vX^{(m)})^{-1}]_{k,k}$ as $v_k^{(m)}$
for any $k\in[d]$ and $m\in [M]$, we have
\begin{equation}\label{lem:ind-est}
\widehat{\beta}^{(m)}_{\mathrm{ind},k} \mid \vX^{(m)}\sim \cN(\beta^{(m)}_k, v^{(m)}_k\sigma_m^2).
\end{equation}
For each $k\in[d]$,
let
$\cB_k\triangleq \{m\in[M]:\,\beta^{(m)}_k\neq \beta^\star_k\}$ be the set of \emph{unaligned tasks} at covariate $k$. 
For any $0<\eta\le 1$,  we define the set of \emph{$\eta$-well-aligned covariates} as
\begin{equation}\label{ieta}
    \cI_\eta:=\left\{k\in[d]:\,\sum_{m=1}^M\one(m\in\cB_k) w_m<\eta W_{[M]}\right\},
\end{equation}
where $W_{\cB_k}\triangleq \sum_{m\in \cB_k}w_m$ and  $W_{[M]}\triangleq \sum_{m=1}^Mw_m$.
The $\eta$-well-aligned covariates are only used in the proof, not in our algorithm.
At each covariate $k\in[d]$, for $m\in[M]\backslash \cB_k$, by \eqref{lem:ind-est}, $\smash{\widehat{\beta}^{(m)}_{{\rm ind},k}}$ is an unbiased estimate of $\beta^\star_k$. 
When $W_{\cB_k}$ is relatively small compared to $W_{[M]}$,  
the set 
$\smash{\{\widehat{\beta}^{(m)}_{{\rm ind},k}\}_{m\in \cB_k}}$ of estimates possibly biased for $\beta^{\star}_k$ have small weights.
We will show that the weighted median can estimate $\beta^\star_k$ accurately, despite the biased subset.

We first bound the estimation error of $\widehat{\beta}^\star$ for $\eta$-well-aligned covariates. To this end, we need to characterize the estimation error of the weighted median of Gaussian inputs with non-identical means and variances, as shown in Appendix \ref{app:median-Gaussian}.
Applying this to each covariate $k\in\cI_\eta$ with $\eta\le 1/5$,  we find the following bounds for the estimation error of $\widehat{\beta}^\star$.

\begin{proposition}[\sc Error bound for well-aligned coordinates]\label{prop:center-covariate}
     Taking\footnote{Throughout the paper, we assume $\{\sigma_m\}_{m=1}^M$ are known as they can be easily estimated using the OLS-based formula $\widehat \sigma_m^2 = \|Y^{(m)}-\vX^{(m)}\widehat \beta^{(m)}_{\rm ind}\|_2^2/(n_m-d)$. Experiments with estimated variances can be found in Appendix~\ref{app:robust_offline}.} $w_m= n_m/\sigma_m^2$ for $m\in[M]$, under Conditions~\ref{asp:noise}, \ref{asp:design-mats}, and \ref{asp:size}, for any $0<\eta \le \frac{1}{5}$, $k\in\cI_\eta$,
it holds with probability at least $1-O((\min_m \tilde n_M/\tilde n_{[M]})\vee (s/d))-O(Mde^{-c\min_m n_m})$ that
\begin{equation*}
    |\widehat{\beta}^\star_k-\beta^\star_k|=\widetilde {O}\left(\sigma_w\frac{W_{\cB_k}+\left(\sum_{m\in \cB_k^c}w_m^2\right)^{1/2}}{ W_{[M]}}\right),
\end{equation*}
where $\sigma_w:= W_{[M]}^{-1}\sum_{m=1}^Mw_m(\sigma_m/\sqrt{n_m})$ is the  weighted average of standard deviations. 
In particular, in the regime of balanced variances where $\sigma_m=\Theta(\sigma)$ for some $\sigma>0$, we have
\begin{equation*}
    |\widehat{\beta}^\star_k-\beta^\star_k|=\widetilde {O}\left(\frac{n_{\cB_k} \sigma}{n_{[M]}\sqrt{\max_{m\in[M]}n_m}}+\frac{\sigma}{\sqrt{n_{[M]}}}\right).
\end{equation*}
\end{proposition}

\begin{remark}[Comparison with \cite{xu2021multitask}]\label{rmk:vjosdfs}
When $n_m=\Theta(n)$ and $\sigma_m=\Theta(\sigma)$ for some $n>0$ and $\sigma>0$ and  all $m\in[M]$, 
\cite{xu2021multitask}  estimate $\beta^\star$ 
through a trimmed mean of the individual OLS estimates.
Then, a Lasso-based shrinkage is used to estimate the local parameters  $\{\beta^{(m)}\}_{m=1}^M$. 
The Lasso step is more computationally expensive than our covariate-wise procedures. 
Moreover,
the trimmed mean is less effective than the median in handling sparse heterogeneity. First, setting the fraction of the data trimmed $\omega$, taken as $\sqrt{s/d}$ in \citet[Corollary 1]{xu2021multitask}, requires knowing the sparse heterogeneity level $s$. 
In contrast, the weighted median is applicable to all potential values of $s$. 
Second, using a single trimming proportion $\omega$ is suboptimal,
as trimming fewer local estimates for covariates with more aligned tasks can improve accuracy. 
The best choice of $\omega =\sqrt{s/d}$ yields 
$ \|\widehat{\beta}^{(m)}-\beta^{(m)}\|_1=\widetilde{O}_P(({d}/{\sqrt{M}}+\sqrt{sd}){\sigma}/{\sqrt{n}})$, which is larger than the minimax optimal rate by a factor of $\sqrt{d/s}$ (see Theorem \ref{thm:ls-fixed-design} and \ref{thm:lower1} for optimality). 
\end{remark}

Proposition \ref{prop:center-covariate} provides an upper bound relating to the weighted frequency $W_{\cB_k}/W_{[M]}$ of misalignment. 
This result cannot be obtained by directly applying concentration inequalities to the estimates, due to heterogeneity.
Instead, we analyze the concentration of the empirical weighted $(1/2\pm{W_{\cB_k}}/{W_{[M]}})$-quantiles to mitigate the influence of heterogeneity. 
The constant $1/5$, which restricts the heterogeneity, 
is not essential and is chosen for clarity. 
With more cumbersome calculations, it can be replaced with any number below ${1}/{2}$. 

While estimating $\beta^\star$ is not our main goal, based on Proposition \ref{prop:center-covariate}, we readily obtain the following bound for estimating $\beta^\star$ by choosing $\eta =\max_{k\in[d]}W_{\cB_k}/W_{[M]} $, 
so $[d]=\cI_\eta$,
and summing up the errors over all covariates. Noting that $\sum_{k\in[d]}W_{\cB_k}/W_{[M]}=s$,
Corollary~\ref{cor:uniform-hetero} reveals that the global parameter can be accurately estimated if heterogeneity happens roughly uniformly across all covariates.
\begin{corollary}[\sc Error bound for global parameter]\label{cor:uniform-hetero} 
Taking $w_m= n_m/\sigma_m^2$  for all $m\in[M]$, under Conditions \ref{asp:noise}, \ref{asp:design-mats}, and \ref{asp:size}, let $\widehat{\beta}^\star$ be obtained from Algorithm \ref{alg:molar} and suppose $W_{\cB_k}/W_{[M]}=O(s/d)$ for any $k\in [d]$. 
For any $p\in\{1,2\}$, it holds with
probability
$1-O((\tilde n_M/\tilde n_{[M]})\vee (s/d))-O(Mde^{-c\min_m n_m})$ that
\begin{equation*}
        \|\widehat{\beta}^\star-\beta^\star\|_p^p=\widetilde{O}\left(\sigma_w^p\left(\frac{s^p}{d^{p-1}}+\frac{d\left(\sum_{m=1}^Mw_m^2\right)^{1/2}}{W_{[M]}}\right)\right),
\end{equation*} 
where $\sigma_w:= W_{[M]}^{-1}\sum_{m=1}^Mw_m(\sigma_m/\sqrt{n_m})$ . In particular,
in the regime of balanced variances where $\sigma_m^2=\Theta(\sigma^2)$ for some $\sigma>0$, we have
\begin{equation*}
    \|\widehat{\beta}^\star-\beta^\star\|_p^p=\widetilde{O}\left(\frac{s^p\sigma^p}{\max_{m\in[M]}n_m^{p/2}d^{p-1}}+\frac{d\sigma^p}{n_{[M]}^{p/2}}\right).
\end{equation*}
\end{corollary}
Proposition \ref{prop:center-covariate} above shows that the coefficients of the well-aligned covariates are accurately estimated. 
Thus one can use the global estimate $\widehat{\beta}^\star$ for the well-aligned covariates where the global parameter $\beta^\star$ also aligns with the local parameter $\beta^{(m)}$. 
The coefficients of the remaining covariates, which are either poorly aligned or do not satisfy $\beta^\star_k=\beta^{(m)}_k$, could be estimated by individual OLS estimates. 
In Theorem \ref{thm:ls-fixed-design} below, 
we argue that  with high probability, 
properly chosen thresholds achieve this. 
The proof is in Appendix \ref{app:ls-fixed-design}.
\begin{theorem}[\sc Error bound for task-wise parameters]\label{thm:ls-fixed-design}
Under Conditions~\ref{asp:heterogneity}, \ref{asp:noise}, \ref{asp:design-mats} and \ref{asp:size},
taking
$\gamma_m =c_1\sqrt{\ln((n_{[M]}/n_m)\wedge d)}\sigma_m$  for all $m\in[M]$ with  $c_1\geq 1$ being constant,
with $\widehat{\beta}^{(m)}_{\mathrm{MOLAR}}$  from Algorithm \ref{alg:molar} using either Option I or II,
it holds for any $p\in\{1,2\}$, $m\in[M]$ that 
\begin{equation}\label{eqn:vmcoxbmnj5}
    \|\widehat{\beta}^{(m)}_{\mathrm{MOLAR}}-\beta^{(m)}\|_p^p =\widetilde{O}_P\left(\frac{s\sigma_m^p}{n_m^{p/2}}+\frac{d}{
    \left(\sum_{m=1}^Mn_m/\sigma_m^2\right)^{p/2}}\right).
\end{equation}In particular,
in the regime of balanced variances where $\sigma_m^2=\Theta(\sigma^2)$ for some $\sigma>0$, we have
\begin{equation}\label{eqn:vjsiodfdvcx}
     \|\widehat{\beta}^{(m)}_{\mathrm{MOLAR}}-\beta^{(m)}\|_p^p =\widetilde{O}_P\left(\frac{s\sigma^p}{n_m^{p/2}}+\frac{d\sigma^p}{n_{[M]}^{p/2}}\right).
\end{equation}
\end{theorem}

\begin{remark}[\sc Varying heterogeneity levels]\label{rmk:varying}
While we assume the task-wise heterogeneity is constrained by a common parameter $s$ in Condition~\ref{asp:heterogneity} for simplicity, we remark that similar theoretical results hold under varying heterogeneity levels where $\|\beta^{(m)}-\beta^\star\|_0\leq s_m$ for some $s_m\in\{0,1,\dots,d\}$ and all $m\in[M]$. 
Considering $\sigma_m=\Theta(\sigma)$ for illustration, 
where we have $\sum_{k\in[d]}W_{\cB_k}/W_{[M]}\asymp \sum_{m\in[M]}n_ms_m/n_{[M]}=:\bar s_w$, Our theory implies
\begin{equation*}
\|\widehat{\beta}^{(m)}_{\mathrm{MOLAR}}-\beta^{(m)}\|_p^p =\widetilde{O}_P\left(\frac{s_m\vee \bar s_w\sigma^p}{n_m^{p/2}}+\frac{d\sigma^p}{
    n_{[M]}^{p/2}}\right).
\end{equation*} Therefore, a few outlier tasks with large $\|\beta^{(m)}-\beta^\star\|_0$ (at most $d$) do not heavily impact the ultimate task-wise estimation errors in \ours, 
revealing its robustness.
\end{remark}

\begin{remark}[\sc Extensions of \ours]
    \ours can be extended 
    beyond parameter estimation in linear models. 
    For example, \ours is readily extended to generalized linear models by adjusting the individual maximum likelihood estimates (MLEs) $\{\widehat \beta^{(m)}_{\rm ind}\}_{m=1}^M$ and adjusting the data matrix $\vX^{(m)\top}\vX^{(m)}$ with the inverse link function. A similar guarantee can be established for sufficiently large sample sizes thanks to the asymptotic normality of task-wise MLEs. 
    For details, we refer to Appendix~\ref{app:glm}.
    
    In addition,  one can construct confidence intervals for task-wise parameters $\{\beta^{(m)}\}_{m=1}^M$ 
    by leveraging the improved concentration of $\widehat \beta^\star_k$ for well-aligned covariates $k$.
    These can have shorter lengths than the canonical single-task OLS intervals; 
    see Appendix~\ref{app:conf-interval}.
\end{remark}

The upper bound in \eqref{eqn:vjsiodfdvcx} consists of two terms. Taking 
$p=2$ for illustration,
the first term---$s/n_m$---is independent of the dimension $d$, and is a factor $d/s$ smaller than the minimax optimal rate $d/ n_m$  of estimation in a single linear regression task \citep{lehmann2006theory}. 
 Meanwhile, the second, dimension-dependent, term---$d/n_{[M]}$---is $n_{[M]}/n_m$ times smaller than  $d/n$, since it depends on the \emph{total sample size} $n_{[M]}$ used collaboratively. 
 This brings a significant benefit under sparse heterogeneity, \ie, when $s\ll d$. 
 Therefore, our method has a factor of $\min\{n_{[M]}/n_m,d/s\}$ improvement in accuracy, compared to the optimal rate for a single
 linear regression task.

Algorithm \ref{alg:molar} is applicable 
 to any value of heterogeneity level $s$. 
 In particular, when the heterogeneity is dense, \ie, $s=\Omega(d)$, our result recovers the optimal estimation error rate  $d/n_m$ in single-task linear regression. 
Therefore, the collaboration mechanism in Algorithm \ref{alg:molar} does not harm the rate, regardless of the level of heterogeneity.

\begin{table}[t]
\centering 
\caption{\small Bounds on the estimation error $\|\widehat{\beta}^{(m)}-\beta^{(m)}\|_1$ of various methods under balanced variances \ie, $\sigma_m^2=\Theta(\sigma^2)$ for all $m\in[M]$:
$\beta^{(m)}$ is the ground-truth parameter,
$\widehat{\beta}^{(m)}$ is the estimator,
${\delta}^{(m)}= \beta^{(m)}-\beta^\star$ is a non-vanishing measure of heterogeneity. 
The standard regime shows the results for balanced datasets, \ie, $n_m=\Theta(n)$ for all $m\in[M]$. 
The data-poor regime shows the results for transfer learning in the $(M+1)$-th task with a potentially small dataset: notably,
\ours and the minimax lower bound requires $n_{M+1}=O(n_{[M]} (s/d)^2)$ while the others require $n_{M+1}=O(\min_{m\in[M]}n_{m}/d^2)$.  See \citet[Sec 3.6]{xu2021multitask} for more details about the baseline methods.
Numerical constants and logarithmic factors are omitted for clarity.
}

\label{tab:lr-comparsion}
\begin{tabular}{lccc}
\toprule
Method & Standard Regime &  Data-poor Regime \\
\midrule
Individual OLS~\citep{lehmann2006theory}\hspace{-3mm} & ${\sigma d}/{\sqrt{n}}$  &  ${\sigma d}/{\sqrt{n_{M+1}}}$ \\
Individual LASSO~\citep{tibshirani1996regression} & ${\sigma d}/{\sqrt{n}}$  &  ${\sigma d}/{\sqrt{n_{M+1}}}$ \\
Global OLS~\citep{Dobriban2018DistributedLR} & $\hspace{-2mm}\|{\delta}^{(m)}\|_1+{\sigma d}/{\sqrt{ Mn}}\hspace{-2mm}$  &  $\hspace{-2mm}\|{\delta}^{(M+1)}\|_1+{\sigma }/{\sqrt{ n_{[M]}}}\hspace{-2mm}$ \\
Robust Multitask~\citep{xu2021multitask} & $ {\sigma \sqrt{sd}}/{\sqrt{n}}+{\sigma d}/{\sqrt{Mn}}$\hspace{-3mm} &  ${\sigma s}/{\sqrt{n_{M+1}}}$ \\
\midrule
\ours (Theorem~\ref{thm:ls-fixed-design}) & ${\sigma {s}}/{\sqrt{n}}+{\sigma d}/{\sqrt{Mn}}$ &  ${\sigma{s}}/{\sqrt{n_{M+1}}}$ \\
Lower Bound (Theorem~\ref{thm:lower1}) & ${\sigma{s}}/{\sqrt{n}}+{\sigma d}/{\sqrt{Mn}}$ &  ${\sigma{s}}/{\sqrt{n_{M+1}}}$ \\
\bottomrule
\end{tabular}
\end{table}

Theorem \ref{thm:ls-fixed-design} also implies that the collaborative estimator $\widehat{\beta}^\star$ is useful in transfer learning \citep{tian2022transfer,li2022transfer}, where,  given a number of ``source'' tasks with large samples, the goal is to maximize performance on a specific ``target''
task with a small sample.

\begin{corollary}[\sc Transfer learning for a data-poor task]\label{cor:data-poor}In the regime of balanced variances,
    when a new $(M+1)$-th task  has a small dataset with $n_{M+1}=O(n_{[M]} (s/d)^{2/p})$, the $\ell_1$ and $\ell_2$ estimation errors for the task parameter using \ours are  $\widetilde{O}_P(\sigma s/\sqrt{ n_{M+1}})$  and $\widetilde{O}_P(\sigma^2 s/ n_{M+1})$, respectively.
\end{corollary}

 The rates  in Corollary \ref{cor:data-poor}  do not depend on the feature dimension $d$; in contrast to a linear dependence for the individual OLS estimate. 
In Table \ref{tab:lr-comparsion}, we compare  \ours with the method from \cite{xu2021multitask} and other baseline methods, in terms of the $\ell_1$ estimation error, in both the standard regime where $n_1=\cdots=n_{M}=n$ and the data-poor regime where $n_{M+1}$ is small. 
We find that \ours outperforms all methods compared in both regimes.

\subsection{Lower Bound}\label{sec:lb-mul-lr}
To complement our upper bounds, we provide 
minimax lower bounds for our multitask linear regression task under sparse heterogeneity. 
We consider 
the best estimators $\{\widehat{\beta}^{(m)}\}_{m=1}^M$ that leverage heterogeneous datasets, 
in the worst-case sense over all global parameters $\beta^\star \in\RR^d$, 
task-wise parameters $\{\beta^{(m)}\}_{m=1}^M$ each in $\BB_{s}(\beta^\star):=\{\beta\in\RR^d:\|\beta-\beta^\star\|_0\le s\}$, as well as covariance matrices $\{\Sigma^{(m)}\}_{m=1}^M$ each in $\mathcal{A}_{\mu,L}\triangleq \{\Sigma\in\RR^{d\times d}:\Sigma \mbox{ positive semi-definite,}$ $ \mu I_d \preceq  \Sigma \preceq LI_d\}$:
\begin{equation}\label{eqn:vhoeqwfq}
\mathcal{M}(s,d,\mu,L,m,p, \{n_m\}_{m=1}^M\},\{\sigma_m\}_{m=1}^M)
     := \inf_{\widehat{\beta}^{(m)}}
    \sup_{\substack{\beta^\star \in\RR^d,\;\{\beta^{(m)}\}_{m=1}^M\subseteq \BB_s(\beta^\star)\\ \{\Sigma^{(m)}\}_{m=1}^M \subseteq \cA_{\mu,L}}}\EE[\|\widehat{\beta}^{(m)}-{\beta}^{(m)}\|_p^p].
\end{equation}
The minimax risk \eqref{eqn:vhoeqwfq} characterizes the best possible worst-case estimation error under our heterogeneous multitask learning model.

Since the supremum is taken over all  parameters $\{\beta^{(m)}\}_{m=1}^M\subseteq\RR^d$ such that $\|\beta^{(m)} -\beta^\star\|_0\le s$ for some $\beta^\star \in \RR^d$, we can consider  two representative cases. 
First, the \emph{homogeneous} case where $\beta^{(1)}=\cdots=\beta^{(m)}=\beta^\star$;
which reduces to a single  linear regression task with $n_{[M]}$ data points and varying noise variances, 
leading to a lower bound  $\Omega_P(d/(\sum_{m=1}^Mn_{m}/\sigma_m^2)^{p/2})$ for the minimax estimation error. 
Second, in the independent \emph{$s$-sparse} case where  $\beta^\star = 0$ and $\|\beta^{(m)}\|_0\le s$ for all $m\in[M]$, clearly $\| \beta^{(m)} - \beta^\star \|_0 \leq s$  and thus  $\{\beta^{(m)}\}_{m=1}^M\subseteq\mathbb{B}_s(\beta^\star)$.
By constructing independent priors for $\{\beta^{(m)}\}_{m=1}^M$, we show that only data points sampled from the model with parameter $\beta^{(m)}$ are informative for estimating $\beta^{(m)}$. 
We then show a lower bound of $\Omega(s \sigma_m^p/n_m^{p/2})$ for the minimax risk. In particular, our lower bound for the $\ell_1$ error strengthens the existing lower bound of order $\Omega(s^{1/2} \sigma_m/n_m^{1/2})$ for sparse linear regression \citep{raskutti2011minimax}.
Combining the two cases, we prove the following lower bound. 
The formal proof is in Appendix \ref{app:lower1}.
\begin{theorem}[\sc Minimax lower bound for linear regression under sparse heterogeneity]\label{thm:lower1}
For any $m\in[M]$, $p\in\{1,2\}$, it holds that
\begin{equation*}
    \mathcal{M}(s,d,\mu,L,m,p, \{n_m\}_{m=1}^M\},\{\sigma_m\}_{m=1}^M)= \Omega\left(\frac{s\sigma_m^p}{n_m^{p/2}}+\frac{d}{\left(\sum_{m=1}^Mn_m/\sigma_m^2\right)^{p/2}}\right).
\end{equation*}
\end{theorem}
Theorems \ref{thm:lower1} and  \ref{thm:ls-fixed-design} imply that \ours is minimax optimal up to logarithmic terms.

\section{Linear Contextual Bandit} \label{sec:multi-lcb}
In this section, we study multitask linear contextual bandits as an application of our \ours method. 
A contextual bandit problem \citep{woodroofe1979one}
consists of data collection over several rounds $t=1, \ldots , T$, where $T>0$ is referred to as the time horizon.
At each round,
an analyst observes a context---covariate---vector, 
and based on all previous observations, chooses one of $K$ options---referred to as arms---observing the associated reward.
The goal is to minimize the \emph{regret} compared to the best possible choices of arms.
In linear
contextual bandits, 
two settings are widely studied.

In the first setup, referred to as {\sf Model-C} by \cite{Ren2020DynamicBL}, at each round $t\in[T]$, the analyst  first observes a set of $K$ $d$-dimensional  contexts $\{x_{t,a}\}_{a\in[K]}$ in which $x_{t,a}$ is associated with arm $a$. If the analyst selects action $a\in[K]$, then a reward $y_{t}=\langle x_{t,a}, \beta\rangle +\ep_t$ is earned,
where $\beta \in\RR^d$ is the
 unknown parameter vector associated with the bandit and $\{\ep_t\}_{t=0}^\infty$ is a sequence of i.i.d.~noise variables. 
In the second setup, referred to as {\sf Model-P} by \cite{Ren2020DynamicBL}, at each round $t\in[T]$, the analyst instead only observes a single $d$-dimensional context $x_t$. Then if action $a\in[K]$ is chosen, the analyst earns the reward $y_{t}=\langle x_{t}, \beta_a\rangle +\ep_t$,
where $\beta_a \in\RR^d$ is the unknown
 parameter vector associated with arm $a$ and $\{\ep_t\}_{t=0}^\infty$ is still a sequence of i.i.d.~noise variables. 

Both models have been studied: 
for instance {\sf Model-C}  in \cite{Han2020SequentialBL,oh2021sparsity,Ren2020DynamicBL} and {\sf Model-P} in \cite{bastani2020online,bastani2021mostly}. 
We extend these models to the multitask setting. 
We present
{\sf Model-C} here, 
and state a parallel set of results under {\sf Model-P} in Appendix \ref{app:model-p}. 
We consider $M$ $K$-armed bandit instances, 
where the $ m$-th  is associated with a parameter $\beta^{(m)}\in\RR^d$.
Each bandit $m\in[M]$ has an activation probability $p_m\in[0,1]$.
At each round $t$ within the time horizon $T$,
each bandit $m$ is independently activated with  probability $p_m$; 
and we observe contexts for the activated bandits. 
The parameters $\{p_m\}_{m=1}^M$ model the frequency 
of receiving contexts. 
When  $p_1=\dots=p_m =1$, contexts for all bandit instances are always observed. We denote $\sum_{m\in\cI}p_m$ as $p_{\cI}$ for any $\cI\subseteq [M]$. Without loss of generality, we assume $p_1\geq \cdots \geq p_M$.

The analyst observes a set of $K$ $d$-dimensional contexts---covariates, feature vectors---$\{x_{t,a}^{(m)}\}_{a\in[K]}$ for each activated bandit instance $m\in\cS_t$ in the set $\cS_t$ of  activated  bandit instances at time $t$. 
The contexts  $\{x_{t,a}^{(m)}\}_{a\in[K]}$, for each $m\in\cS_t$ in each round are sampled independently.
Using the observed contexts, and all previously observed data,
the analyst 
selects actions $a_t^{(m)}\in[K]$ for each activated bandit instance $m\in\cS_t$ and earns the reward $y_{t}^{(m)}=\langle x_{t,a_t^{(m)}}^{(m)},\beta^{(m)}\rangle +\ep_{t}^{(m)}\in\RR$, 
where $\ep_{t}^{(m)}$ are, for $t\in [T], m\in \cS_t$, i.i.d.~noise variables.
To apply \ours in this setting, we
require  Condition  \ref{asp:shb} over all bandit parameters. 
Compared to Condition \ref{asp:heterogneity} for linear regression,
this requires all parameters to have a bounded $\ell_2$ norm, 
as is common in the area (see \eg, \cite{Han2020SequentialBL,bastani2020online}).

\begin{assumption}[\sc Sparse heterogeneity \& Boundedness]\label{asp:shb}
    There is an unknown global parameter $\beta^\star\in\RR^d$ and 
    value $s \in \{0,1,\ldots, d\}$ 
    such that $\|\beta^{(m)}-\beta^\star\|_0\le s$ for any $m\in[M]$. Further, $\|\beta^{(m)}\|_2\le 1$ for all $m\in[M]$. 
\end{assumption}
As before, our analysis applies to each $\beta^\star,s $ for which the condition holds.

\begin{assumption}[\sc Frequency constraint]\label{asp:freq}
We have $\min_{m\in[M]}(p_{1}\vee (p_{[M]} /m))/p_{m} \le c_{\rm f}$ for an absolute constant $c_{\rm f}$, where $p_{[M]}\triangleq \sum_{m=1}^M p_m$ .
\end{assumption}

Following previous literature studying {\sf Model-C}, we state other standard conditions as follows. 
The Gaussian noise Condition \ref{asp:gsn} is used only for simplicity, 
as our results can be extended to more general noise with a bounded Orlicz norm by leveraging the results of the offline linear regression with general noise in Appendix \ref{app:general-noise}. 

\begin{assumption}[\sc Gaussian noise]\label{asp:gsn}
    The noise variables $\{\ep_t^{(m)}\}_{t=1}^\infty$ are i.i.d.~$\cN(0,1)$ variables.
\end{assumption}

Next, we require the contexts to be sub-Gaussian.
Since bounded contexts are sub-Gaussian, Condition \ref{asp:covariate-subg} is weaker than 
the assumption of bounded contexts commonly used in the literature on contextual bandits \citep{xu2021multitask,bastani2021mostly,bastani2020online,kim2019doubly}.

\begin{assumption}[\sc Sub-Gaussianity]\label{asp:covariate-subg}
   There is $L\ge 0$, such that  
    for each $m\in[M]$, $t\in [T]$, and $a\in[K]$, $x_{t,a}^{(m)}$ is $L$-sub-Gaussian.
\end{assumption}

We next consider
Condition \ref{asp:diverse},
which ensures sufficient exploration even with a greedy algorithm \citep{Ren2020DynamicBL}. 

\begin{assumption}[\sc Diverse context]\label{asp:diverse}
    There are  positive constants $\mu$ and $c_x$ such that for any $\beta\in\RR^d$, vector $v\in\RR^d$, and $m\in[M]$,
    it holds that $\PP(\langle x^{(m)}_{t,a^\star}, v\rangle^2\ge \mu)\ge c_x$ where $a^\star =\underset{{a\in[K]}}{\operatorname{argmax}} \langle x_{t,a}^{(m)} ,\beta\rangle$ and the probability  is taken over the joint distribution of $\{x_{t,a}^{(m)}\}_{a=1}^K$.
\end{assumption}

\begin{remark}
     The ``diverse context'' condition, which is widely used in the literature on bandits and reinforcement learning,  simplifies the algorithmic design and the corresponding proof. 
     The condition or equivalent settings have been used in \eg,  \citet[Lemma 1]{bastani2021mostly},  \citet[Condition 4]{oh2021sparsity}, \citet[Condition 2]{cella2022multi}, \citet[Definition 3.1]{hao2021online}, 
     \citet[Condition 2.2 (c)]{chakraborty2023thompson}, 
     \cite{Han2020SequentialBL}. 
     Condition \ref{asp:diverse} encompasses many classical distributions, \eg, Gaussian contexts where $x_{t,a}\sim \cN(0,\Sigma)$ with $\Sigma\succeq 16 \mu I_d$ for each $a\in [K]$, and sub-Gaussian distributions.
    See \citet[Section 2.3]{Ren2020DynamicBL} for more discussion.
     In Conditions \ref{asp:covariate-subg} and \ref{asp:diverse}, we do not require $\{x_{t,a}^{(m)}\}_{a=1}^K$ to be independent across actions $a\in[K]$.  

     On the other hand, we remark that the condition is made only for simplicity and to show that \ours can be applied in an important online setting. 
     One can readily extend \oursb to other bandit algorithms with exploration phases (see \eg, \cite{bastani2020online,xu2021multitask}).
\end{remark}

\subsection{Algorithm Overview}
\begin{algorithm}[t]
	\caption{MOLARBandit: Multitask Bandits with \ours estimates}
	\label{alg:bc-bandit}
	\begin{algorithmic}
		\STATE \noindent {\bfseries Input:} Time horizon $T$, $\widehat{\beta}^{(m)}_{-1}=0$, $\vX^{(m)}_{q}=\emptyset$, and $Y^{(m)}_{q}=\emptyset$  for $m\in[M]$,  
  initial batch size $H_0$ and batch $\cH_0=[H_0]$, number of batches $Q= \lceil \log_2(T/H_0)\rceil$
  \vspace{1mm}
        \STATE Define batches $\cH_q = \{t:2^{q-1}H_0<t \le \min\{2^q H_0, T\}\}$, for $q=1,\dots, Q $
        \FOR{$t=1,\cdots,T$}
		\FOR{each bandit in parallel}
        \vspace{1mm}
        \IF{$t\in\cH_q$ \textbf{and} bandit instance $m$ is activated}
        \vspace{1mm}
		\STATE Choose $a_t^{(m)}=\arg\max_{a\in[K]}\langle x_{t,a}^{(m)},\widehat{\beta}_{q-1}^{(m)}\rangle$,  breaking ties randomly, and gain reward $y_{t}^{(m)}$
  \vspace{1mm}
		\STATE Augment observations $\vX_{q}^{(m)}\,\leftarrow\,[\vX_{
  q}^{(m)\top}, x_{t,a_t^{(m)}}^{(m)}]^\top$ and $Y_{q}^{(m)}\,\leftarrow\,[Y_{q}^{(m)\top},y_t^{(m)}]^\top$
  \vspace{1mm}
        \ENDIF
		\ENDFOR
  \IF{$t= 2^q H_0$, \ie, batch $\cH_q$ ends}
        \vspace{1mm}
        \STATE Let $n_{m,q}=|Y_{q}^{(m)}|$ and $\cC_q=\{m\in[M]:n_{m,q}\ge 2C_{\rm b}(\ln(MT)+d\ln(L\ln(K)/\mu))\}$ 
        \vspace{1mm}
		\STATE Call MOLAR$(\{(\vX^{(m)}_{q},Y^{(m)}_{q})\}_{m\in\cC_q})$ to obtain $\{\widehat{\beta}^{(m)}_{q}\}_{m\in \cC_q}$
        \vspace{1mm}
        \FOR{$m\in[M]\backslash \cC_q$}
        \vspace{1mm}
        \STATE Let $\widehat{\beta}^{(m)}_{q}=\widehat{\beta}^{(m)}_{q-1}$, $\vX_{{q+1}}^{(m)}=\vX_{{q}}^{m}$, and $Y_{{q+1}}^{(m)}=Y_{{q}}^{m}$ 
        \vspace{1mm}
        \ENDFOR
        \ENDIF
  \ENDFOR 
	\end{algorithmic}
\end{algorithm}

We now introduce our \oursb algorithm, see Algorithm \ref{alg:bc-bandit}. \oursb manages multiple bandit instances in a batched way, as in \cite{Han2020SequentialBL,Ren2020DynamicBL}.
We split the time horizon $T$ into batches that double in length, \ie, $|\cH_{q}|=2^{q-1}|\cH_0|$ for all $q\ge 1$, 
yielding $Q=O(\log_2(T/|\cH_0|))$ batches. 
Within each batch $\cH_q$, 
each bandit instance $m$ leverages 
the current estimate $\widehat{\beta}_q^{(m)}$ of the parameter $\beta^{(m)}$ without further exploration. 
These estimates are updated at the end of a batch, based on all previous observations. 

Compared to existing methods,
\oursb has two novel features: 
\emph{novel estimates} and \emph{fine-grained collaboration}.
First,
in the single-bandit regime,
the estimate $\widehat{\beta}^{(m)}$ is often obtained by using OLS \citep{goldenshluger2013linear}, LASSO \citep{bastani2020online}, and ridge regression \citep{Han2020SequentialBL}. 
In the multi-bandit regime, we use our \ours estimators to improve accuracy based on all instances. 
If used at each time step, \ours
induces strong correlations between the observations,
as well as between the estimates $\{\widehat{\beta}_{\rm ind}^{(m)}\}_{m=1}^M$ from  Algorithm \ref{alg:molar}, across all instances. 
This makes the median potentially inaccurate (see Lemma \ref{lem:median-Gaussian}). 
As a remedy, batching ensures that our 
estimates in the current batch are independent conditional on the observations from previous batches. 
Batching  
also reduces computational costs.

Second, as indicated by Condition \ref{asp:design-mats} from Section \ref{sec:ana-gau}, \ours requires the eigenvalues of the non-centered empirical covariance matrices of the datasets in the collaboration to be lower bounded. 
We will show that this holds with high probability
even when the arms are adaptively chosen
(Lemma \ref{lem:eig-lb}  in Appendix \ref{app:eig-lb}) when the sample size $n_{m,q}$ is of order $\widetilde{\Omega}(d)$.
However, this may fail within a small batch $\cH_q$ for bandit instances that rarely observe contexts because of  their small activation probabilities, so  $m\notin\cC_q$, 
with 
\begin{equation}\nonumber
    \label{cq}
\cC_q=\{m\in[M]:n_{m,q}\ge 2C_{\rm b}(\ln(MT)+d\ln(L\ln(K)/\mu))\},
\end{equation}
for $C_{\rm b}$ defined in Lemma \ref{lem:eig-lb}. 
Thus we neither involve these instances in \ours, 
nor update their parameter estimates until entering a large batch. 
In these instances,
the observations are not used to update estimates, 
and are merged into future batches.

\subsection{Regret Analysis}
Due to the differences in the activation probabilities $\{p_m\}_{m=1}^M$, the number of observed contexts and the regret of each bandit instance 
can vary greatly. 
Therefore,
we consider the following form of individual regret: 
given a time horizon $T\ge 1$ and a specific algorithm $A$ that produces action trajectories $\{a_t^{(m)}\}_{t\in[T],m\in[M]}$, we define the cumulative regret for each instance $m\in[M]$ as
\begin{equation*}
    R_T^{(m)}(A):=\sum_{t=1}^T\EE\left[\max_{a\in[K]}\langle x_{t,a}^{(m)}-x_{t,a_t^{(m)}}^{(m)},\beta^{(m)}\rangle\one(m\in\cS_t)\right]
\end{equation*}
where $\cS_t$ is the random set of activated bandits at time $t$. 

After showing that the empirical
covariance matrices are well-conditioned  with high probability at the end of each batch in Lemma \ref{lem:eig-lb}, 
we leverage results from Section \ref{sec:ana-gau} to show that the \ours estimates $\{\widehat{\beta}^{(m)}_q\}_{m\in\cC_q}$ are  accurate. 
This also requires controlling $\{n_{m,q}\}_{m\in\cC_q}$ to meet the sample size constraint (Condition \ref{asp:size}) with a high probability, based on Condition \ref{asp:freq}. 
This result is included in Lemma \ref{lem:bandit-est} and is proved in Appendix \ref{app:bandit-est}.
\begin{lemma}[\sc Parameter estimation bound for heterogeneous bandits]\label{lem:bandit-est}
    Under Conditions \ref{asp:shb}-\ref{asp:diverse}, for any $0\le q<Q$, and\footnote{We choose the largest index achieving the minimum.} $\tau =\arg\min_{m\in[M]}(p_{1}\vee p_{[M]} /m))/p_{m}$, if $|\cH_q|\ge 2C_{\rm b}(\ln(MT)+d\ln(L\ln(K)/\mu))/p_{\tau}$ with $C_{\rm b}$ defined in Lemma \ref{lem:eig-lb}, it holds with probability at least $1-2/T$  (over the randomness of $\{\vX_q^{(m)}\}_{m=1}^M$)
    that for all $m\in\cC_q$,
    \begin{equation*}
        \EE[\|\widehat{\beta}^{(m)}_q-\beta^{(m)}\|_2^2 \mid (\vX^{(m)}_{q},Y^{(m)}_{q})_{m\in\cC_q}]=\widetilde{O}\left(\frac{1}{ |\cH_q|}\left(\frac{s}{p_m}+\frac{d}{p_{[M]}}\right)\right),
    \end{equation*}
    where
    logarithmic factors and quantities  depending only on $c_x$, $c_{\rm f}$  are absorbed into $\widetilde{O}(\cdot)$.
\end{lemma}

Based on Lemma \ref{lem:bandit-est}, we can bound the individual regret as follows; with a proof in Appendix \ref{app:bc-bandit}.
\begin{theorem}[\sc Individual regret upper bound for heterogeneous bandits]\label{thm:bc-bandit}
Under Conditions \ref{asp:shb}-\ref{asp:diverse}, 
the expected regret of \oursb, for any $T\ge 1$ and $1\le H_0\le d$, is bounded as
\begin{equation}\label{eqn:vnidsczxcvs}
    \EE[R_T^{(m)}]=
   \widetilde{O}\left(  d\wedge( Tp_m)+
    \sqrt{\left(s+\frac{dp_m}{p_{[M]}}\right)Tp_m}\right)
\end{equation}
where logarithmic factors as well as quantities depending only on $c_x$, $c_{\rm f}$, $L$, and $\mu$  are absorbed into $\widetilde{O}(\cdot)$.
In particular, if contexts are observed for all bandits at all times, so $p_1=\cdots=p_m=1$, \eqref{eqn:vnidsczxcvs} implies
\begin{equation*}
    \EE[R_T^{(m)}]=
   \widetilde{O}\left( d\wedge T+
    \sqrt{\left(s+\frac{d}{M}\right)T}\right).
\end{equation*}
\end{theorem}

For a single contextual bandit without collaboration,  
when $T=\Omega(d^2)$, the minimax optimal  regret bound is $\widetilde{\Theta}(\sqrt{dT})$~\citep{Han2020SequentialBL,chu2011contextual,auer2002using}. Theorem \ref{thm:bc-bandit} implies a regret bound of order $\widetilde{O}(\sqrt{(s+d/M)T})$, when $T=\Omega( d^2/(s+d/M))$, which shows a factor of $\min\{M,d/s\}^{1/2}$ improvement. 
Similarly to Theorem \ref{thm:ls-fixed-design}, 
Algorithm \ref{alg:bc-bandit} is applicable to 
any heterogeneity level $s$. 
When the heterogeneity is non-sparse, \ie, $s=\Omega(d)$, 
Theorem \ref{alg:bc-bandit} recovers the optimal regret $\widetilde{\Theta}(\sqrt{d T})$ for
a single bandit. 
Thus, the collaboration mechanism in Algorithm \ref{alg:bc-bandit} is always benign, regardless of the heterogeneity.

\subsection{Lower Bound}
To complement our upper bound, we also present a minimax regret lower bound that characterizes the fundamental learning limits
in multitask linear contextual bandits under sparse heterogeneity. Similar to Theorem \ref{thm:lower1}, Theorem \ref{thm:lb-bandit} also leverages two representative cases:
the homogeneous case where $\beta^{(1)}=\cdots=\beta^{(m)}=\beta^\star$ and the $s$-sparse case where $\beta^\star =0$ and $\|\beta^{(m)}\|_0\le s$ for all $m\in[M]$. The two cases yield a lower bound of orders $\Omega(\sqrt{{dTp_m^2}/{p_{[M]}}})$ and $\Omega(\sqrt{{sTp_m}})$, respectively. For each case, the proof outline is similar to the lower bound from \cite{Han2020SequentialBL} in the single-bandit regime by additionally incorporating the probabilistic activations.
We set a uniform prior for the parameters 
and translate the regret to a measure characterizing the difficulty of distinguishing two distributions. 
The difficulty---and therefore the regret---is then quantified and bounded via Le Cam's method \citep{Tsybakov2008IntroductionTN}. 
Since our task-specific regret in the multitask regime is  different  from the standard regret in the single-task regime, 
the proof requires some novel steps to handle the activation sets $\cS_t$;
see Appendix \ref{app:lb-bandit}.

\begin{theorem}[\sc Individual regret lower bound for heterogeneous bandits]\label{thm:lb-bandit}
Given any $1\le s\le d$ and $\{p_m\}_{m=1}^M\subseteq[0,1]$, for any $m\in[M]$, when $T\ge \max\{(d+1)/p_{[M]}, (s+1)/p_m\}/(16L)+1$,  there exist   $\{\beta^{(m)}\}_{a\in[K],m\in[M]}$ satisfying Condition \ref{asp:shb},
and distributions of contexts satisfying Condition
\ref{asp:covariate-subg} and \ref{asp:diverse}, 
such that for any online Algorithm $A$ and for any $m\in[M]$,
\begin{equation*}
    \EE[R_T^{(m)}(A)]=\Omega\left(\sqrt{\left(s+\frac{dp_m}{p_{[M]}}\right)Tp_m}\right).
\end{equation*}
In particular, when $p_m=1$ for all $m\in[M]$,
$\EE[R_T^{(m)}(A)]=\Omega\left(\sqrt{\left(s+d/M\right)T}\right)$.
\end{theorem}

Theorem~\ref{thm:lb-bandit} and \ref{thm:bc-bandit} imply that \oursb is minimax optimal when $T =\Omega( d^2/(s + d/M ))$,  up to logarithmic terms.

\begin{figure}[t]
    \centering
    \includegraphics[height = 0.25\textwidth]{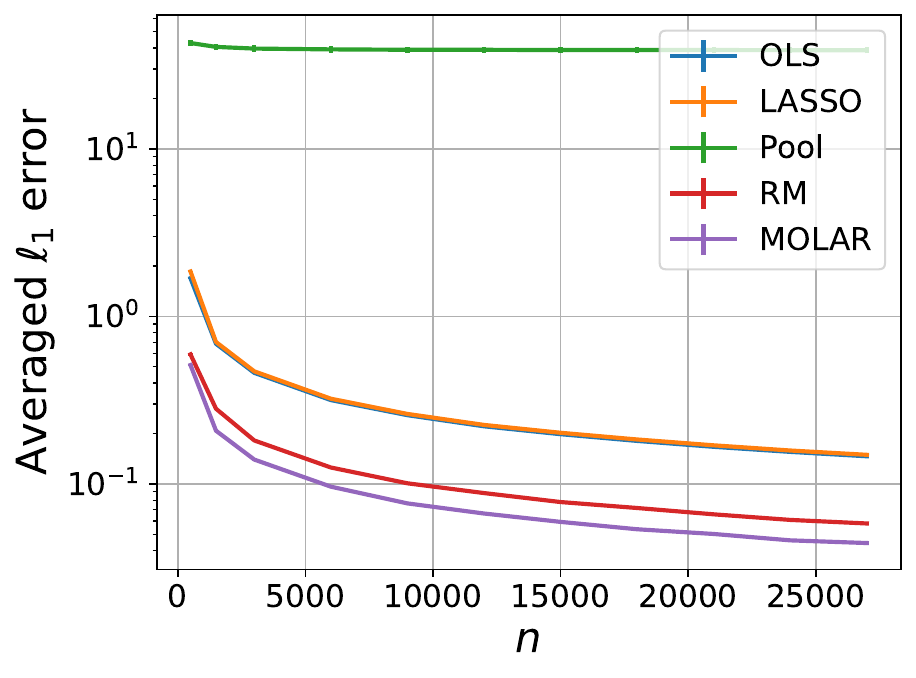}
    \hspace{-3mm}
    \includegraphics[height = 0.25\textwidth]{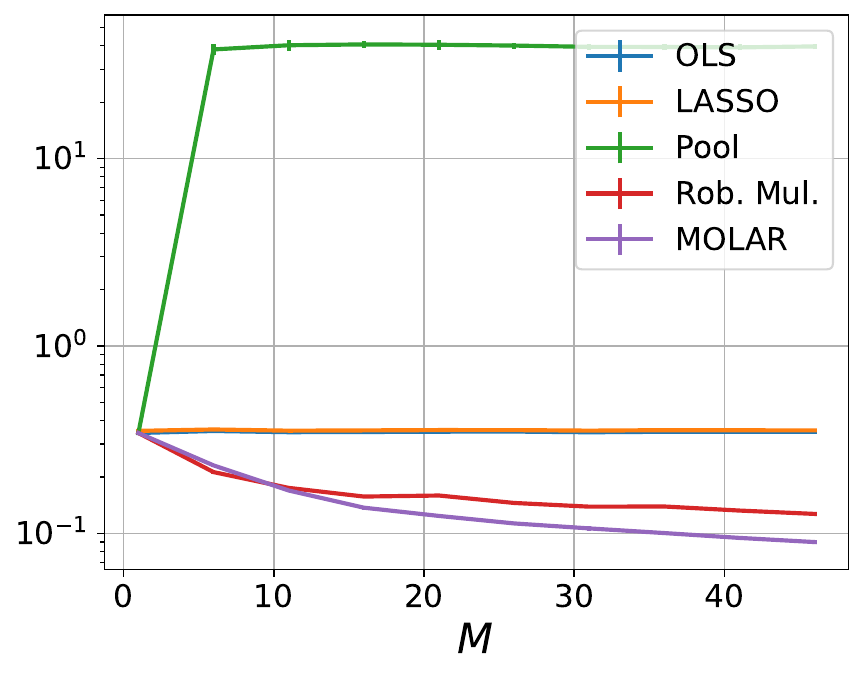}
    \hspace{-3mm}
    \includegraphics[height = 0.25\textwidth]{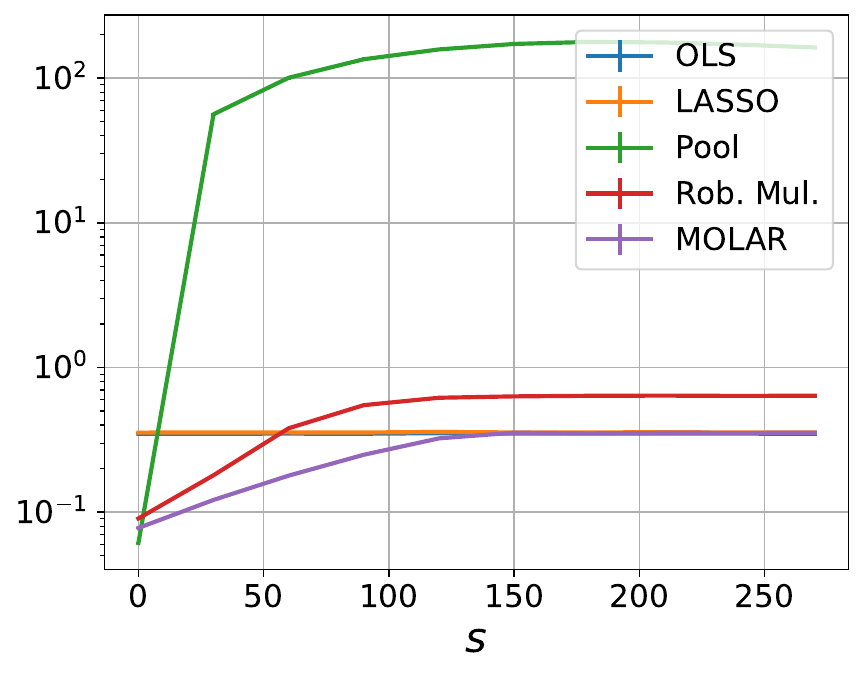}
    \caption{Average $\ell_1$ estimation error for multitask linear regression. (Left):  Fixing $s = 20, \,M = 30$ and varying $n$.
  (Middle): Fixing $s = 20,\,n = 5,000$ and varying $M$.
  (Right): Fixing  $M = 30, \,n = 5,000$ and varying $s$.
  The standard error bars  (barely visible) are obtained from ten independent trials. 
  }
    \label{fig:synthetic-ls}
\end{figure}

\section{Experiments}
In this section, we evaluate the performance of our method in both offline and online scenarios with synthetic and empirical datasets. We provide an overview of the experiments here and provide the remaining details in Appendix \ref{app:experi}.
For linear regression, we evaluate individual OLS estimates, LASSO estimates, a global OLS estimate via data pooling, the robust multitask estimate \citep{xu2021multitask}, and our \ours estimates, denoted by OLS, LASSO, Pool, RM, and  \ours below. 
For contextual bandits, we evaluate
the OLS Bandit \citep{goldenshluger2013linear}, LASSO Bandit \citep{Ren2020DynamicBL}, Trace Norm Bandit \citep{cella2022multi}, Robust Multitask Bandit \citep{xu2021multitask}, and our \oursb methods over multiple \textsf{Model-C} bandit instances\footnote{A few multitask bandit are not obviously applicable to
our setup: \cite{soare2014multi, gentile2014online} aggregate data from similar yet heterogeneous instances, leading  to linear growth in regret; \citep{kveton2021meta,cella2020meta,bastani2022meta} consider Bayesian meta-learning that require instances to be observed sequentially
rather than simultaneously to construct a  prior for instances.}, 
denoted by OLSB, LASSOB, TNB, RMB, and  MOLARB below. 
OLSB and LASSOB act by treating $M$  bandit instances independently, either via
OLS or LASSO. 
Trace Norm Bandit is a state-of-the-art multitask bandit method that leverages trace---nuclear---norm regularized estimates, 
improving accuracy when the parameters span a linear space of rank smaller than $d$. 
More experimental details and results can be found in Appendix~\ref{app:experi}.

\subsection{Numerical simulations}\label{sec:sync}

\textbf{Linear Regression. }\label{sec:synthetic-lr}
We first 
randomly sample $\beta^\star $ from the uniform distribution over the $(d-1)$-dimensional sphere $\mathbb{S}^{d-1}$ where $d=300$.
From the $d$ covariates of $\beta^\star$, we draw $s$ covariates uniformly at random 
and randomly assign new values sampled from the standard Gaussian distribution with re-normalization to preserve $\|\beta^{(m)}\|_2=\|\beta^\star\|$ for all $m\in[M]$.
We repeat this procedure $M$ times to obtain sparsely perturbed parameters $\{\beta^{(m)}\}_{m=1}^M\subseteq  \mathbb{S}^{d-1}$.
Then, $M$ datasets  $\{x_i^{(1)}\}_{i=1}^n,\dots,\{x_i^{(m)}\}_{i=1}^n$ with i.i.d.~$\cN(0, I_{d})$ features
are sampled for each task $m \in [M]$, each containing $n$ data points\footnote{We conduct similar experiments for correlated covariates and disparate task-wise sample sizes in Appendix~\ref{app:correlated-disparse}.}. 
The outcomes $\{y_i^{(1)}\}_{i=1}^n,\dots,\{y_i^{(m)}\}_{i=1}^n$ are set as $y_i^{(m)} = \langle x_i^{(m)}, \beta^{(m)}\rangle +\ep_i^{(m)}$ where $\ep_i^{(m)}$ are i.i.d.~$\cN(0, \sigma^2)$ noise with $\sigma= 0.1$. 
We conduct the simulations by varying the sample size $n$, the number of tasks $M$, and the number of heterogeneous covariates $s$. 
As the datasets have equal sample sizes, 
we take the averaged $\ell_1$ error $\frac{1}{M}\sum_{m=1}^{M}\|\widehat{\beta}^{(m)}-\beta^{(m)}\|_1$ as the performance metric.

 Since we have  sparse heterogeneity, 
 we expect RM and \ours
to outperform baseline methods, 
which is corroborated by the experimental results from Figure \ref{fig:synthetic-ls}. 
For \ours, 
the estimation error decreases as  $n$ and $M$ increase and  $s$ decreases, as revealed by Theorem \ref{thm:ls-fixed-design}. 
Furthermore, \ours outperforms baseline methods for most values of $n$, $M$, and $s$.
Other methods outperform \ours when $s$ is sufficiently large,
as shown in the right panel of Figure \ref{fig:synthetic-ls}, which highlights the crucial role of sparse heterogeneity. 
However, we remark that when $s$ is close to $d$,  the parameters $\{\beta^{(m)}\}_{m=1}^M$ are highly different, 
and no multitask approaches can provably outperform the individual OLS estimates.

Figure \ref{fig:synthetic-ls} also supports the theoretical predictions from  Table \ref{tab:lr-comparsion}. 
OLS  does not pool data 
and thus its estimation error does not vary 
with $M$ and $s$.
Since the  parameters are not sparse, 
we see no benefit in
using LASSO over OLS, and thus, their curves almost overlap. Also, due to the heterogeneity, pooling all datasets introduces a non-vanishing bias when $M>1$; 
even for large $n$ and $M$. 
Furthermore, these estimation errors grow as $s$ grows, due 
to the increasing heterogeneity. 
RM, while addressing heterogeneity, 
performs worse than \ours due to its sub-optimality discussed earlier.

We also perform simulations to support our theoretical results about the rate of the estimation error of \ours, presented in Theorem \ref{thm:ls-fixed-design}. We fix $n=10,000$ and generate $\{\beta^{(m)}\}_{m=1}^M$ in different dimensions $d$ but with a constant ratio $\rho = s/d$. As can be seen in Figure \ref{fig:rate_verify}, when $M$ and $\rho = s/d$ are fixed, the estimation error of \ours grows linearly as $d$ increases. 
The slopes of the curves, corresponding to $\Theta(\sigma (\rho+1/\sqrt{M})/\sqrt{n})$ in Theorem \ref{thm:ls-fixed-design}, decrease as $M$ grows and $\rho$ decays.
This aligns with our theoretical results.

\begin{figure}
    \centering
    \includegraphics[height = 0.25\textwidth]{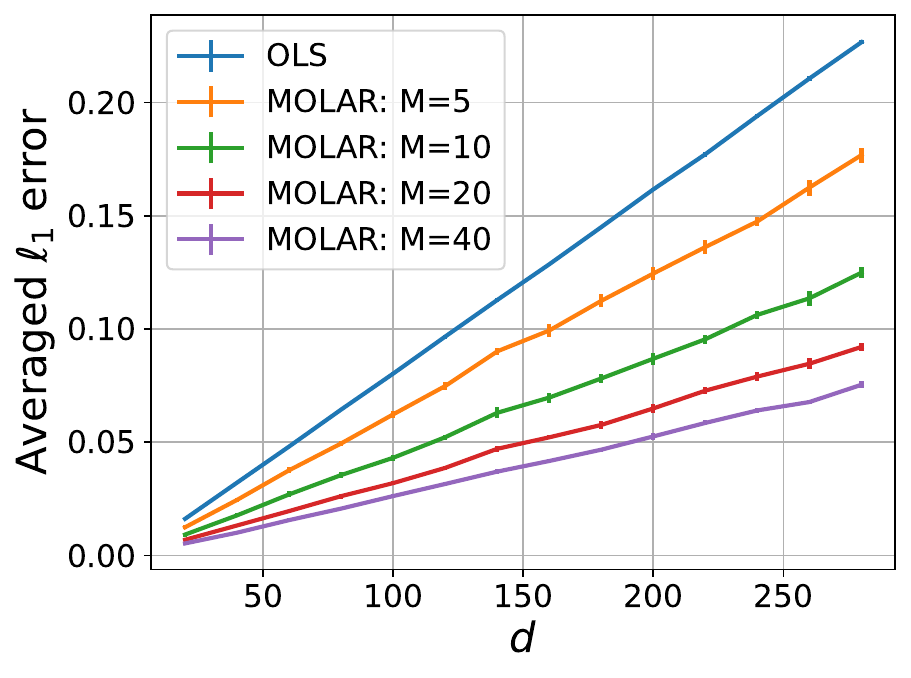}
    \hspace{10mm}
    \includegraphics[height = 0.25\textwidth]{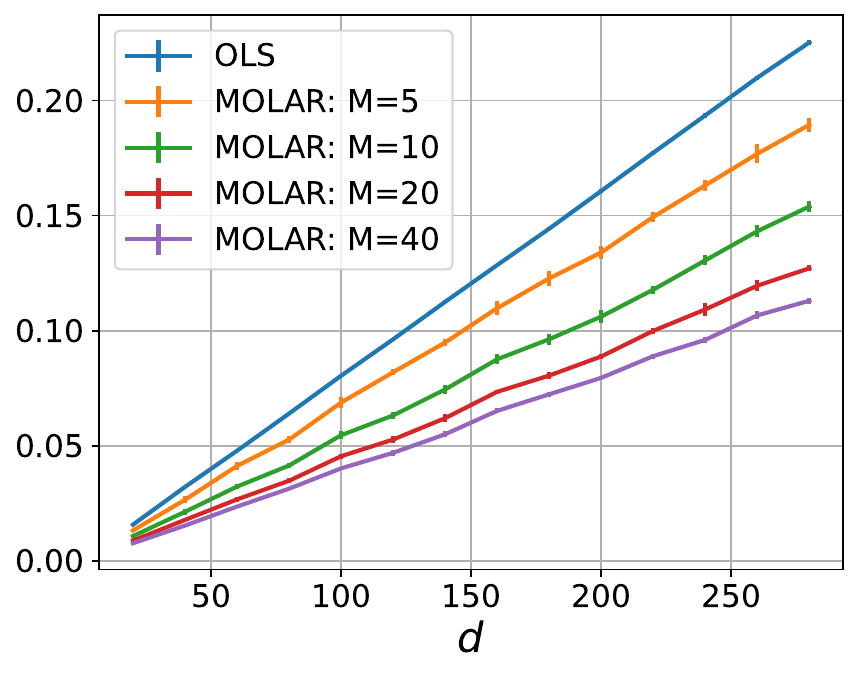}
    \caption{Average $\ell_1$ estimation error for multitask linear regression with $n=10,000$ and $\rho=s/d$ fixed. (Left): $\rho=0.1$.
  (Right): $\rho=0.2$.
  The standard error bars are obtained from ten independent trials.}
    \label{fig:rate_verify}
\end{figure}

\vspace{2mm}
\noindent\textbf{Linear Contextual Bandits. } 
We
set $(d,s,M,K)= (30,2,20, 3)$ and randomly sample the activation probabilities $\{p_m\}_{m=1}^M$ from the uniform distribution on $[0,1]$. 
We then sample the sparsely heterogeneous parameters $\{\beta^{(m)}\}_{m=1}^M\subseteq  \mathbb{S}^{d-1}$  as in the linear-regression experiments. For any activated bandit $m$ at time $t$,  we independently sample the contexts $\{x_{t,a}^{(m)}\}_{a\in[K]}$  from $\cN(0,I_d)$ and the  sample reward noise $\ep_t^{(m)}\sim\cN(0, \sigma^2)$ with $\sigma= 0.5$.

We consider instances with a large, medium, and small activation probability respectively, as shown in Figure \ref{fig:synthetic-b}.
We observe that MOLARB outperforms all baseline methods. 
The advantage over OLSB and LASSOB is substantial, 
as they do not leverage collaboration across tasks. 
We observe that RMB and TNB outperform OLSB and LASSOB due to regularization.
TNB is slightly less accurate because the parameters do not necessarily have a low-rank structure.
Moreover, the difference
between OLSB and LASSOB is small, as the parameters are not sparse. 
Also, the cumulative regret of instances with smaller activation probabilities is lower than of the
ones with larger activation probabilities, due to fewer rounds of decision-making.
Further,  \oursb is computationally more efficient than LASSOB, TNB, and RMB which require solving optimization problems in each update.

\begin{figure}[t]
    \centering
    \includegraphics[height = 0.25\textwidth]{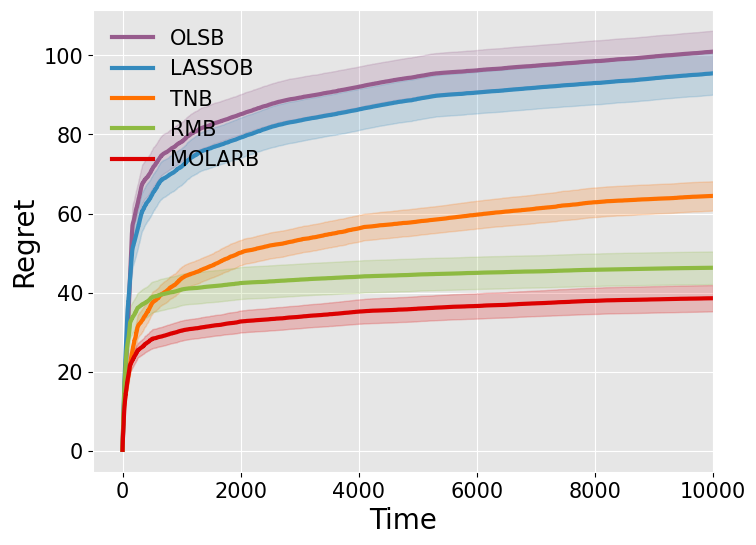}
    \hspace{-3.5mm}
    \includegraphics[height = 0.25\textwidth]{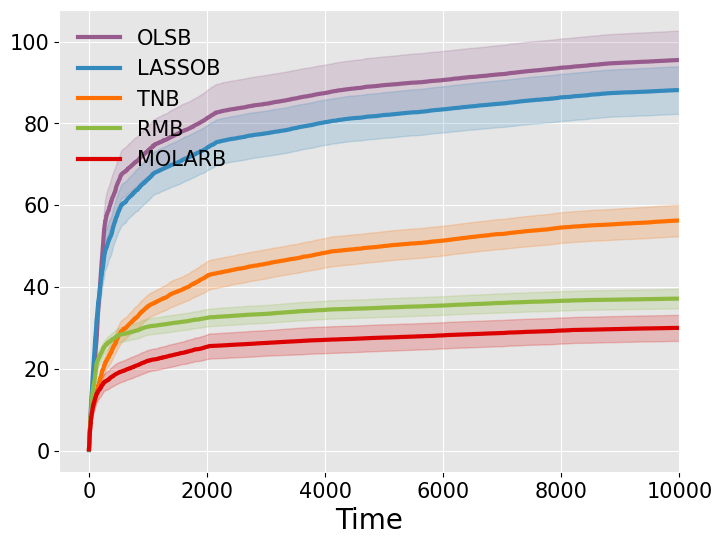}
    \hspace{-3.5mm}
    \includegraphics[height = 0.25\textwidth]{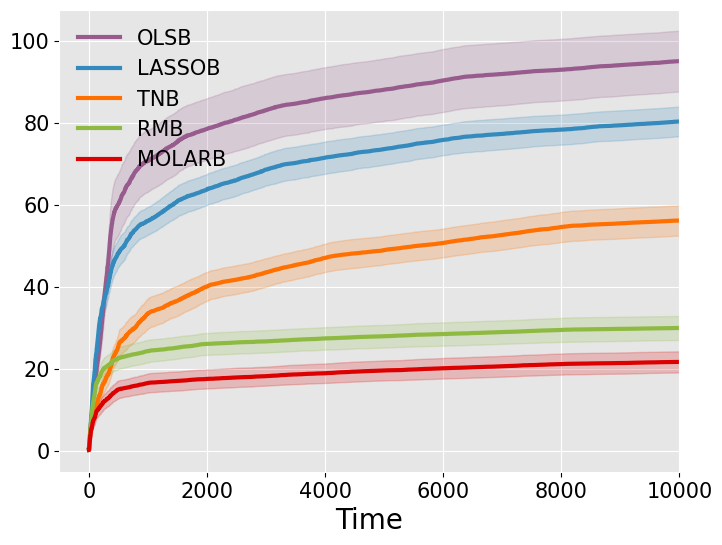}
    \caption{Regret $R_T^{(m)}$ of instances with activation probability $0.778$ (Left), $0.466$ (Middle), $0.318$ (Right), respectively, where
    shaded regions depict the corresponding
95\% normal confidence intervals based on standard errors calculated over twenty independent trials. 
    }
    \label{fig:synthetic-b}
\end{figure}

\subsection{PISA Dataset}\label{pisa}
The \emph{Programme for International Student Assessment} (PISA) is a large-scale international study conducted by the Organisation for Economic Co-operation and Development (OECD).
The study aims to evaluate the quality of education systems around the world by assessing the skills and knowledge of 15-year-old students in reading, mathematics, and science. 
This dataset has been widely used to gain insights into the impact of factors including teaching practices \citep{oecd2019teaching},  gender \citep{stoet2018gender}, and socioeconomic status \citep{kline2019socioeconomic} on student academic performance.

In this experiment, we use a part of the PISA2012 data across $M=15$ countries 
to learn linear predictors for individual countries,
treating each country as a  task.
After basic preprocessing detailed in the Appendix \ref{app:experi}, we have $57$ student-specific features and a continuous response assessing students' mathematics ability---the variable ``PV1MATH'', standardized.  
See Appendix \ref{app:experi} for additional experimental details including fractions of data used for training, validation, and testing,   hyperparameter choices, and robustness checks.
Figure \ref{fig:coef_plt} plots the differences in the coefficients across countries. The structure of sparse heterogeneity appears to be reflected in the dataset.

Our experiments include linear regression and contextual bandits. In linear regression, we estimate the linear coefficients of the processed features for predicting the response with the aforementioned estimation methods. 
We simulate the setup of \textsf{Model-C} as follows. 
We read the records of two students from each country with an activation probability proportional to its sample size.
The goal in each round is to select the student with a better ability in mathematics.
Recall that in \textsf{Model-C}, a $K$-armed bandit observes $K$ contexts at a time, with a shared parameter across the contexts generating the rewards.
Accordingly, there are $K=2$ arms and the reward is the mathematics score.
 Since the two data points are randomly drawn from the dataset of the same country without replacement, the population parameters are clearly identical.

\begin{table}[h]
\centering
\caption{The $\ell_1$ estimation errors and the averaged relative error (A.R.E.) on the PISA dataset, over $100$ independent random data splits.}
\begin{tabular}{lccccc}
\toprule
       Country              & OLS           & LASSO         & Pool          & RM     & \ours         \\\hline
Mexico      & $1.35\pm0.02$  &  $1.57\pm0.02$  &  $2.11\pm0.01$  &  { $\bf1.22\pm0.01$}  &  $1.32\pm0.01$  \\
Italy      & $2.05\pm0.04$  &  $1.59\pm0.01$  &  $2.34\pm0.01$  &  {$\bf1.53\pm0.02$}  &  $1.61\pm0.02$  \\
Spain      & $2.00\pm0.04$  &  $1.85\pm0.02$  &  $2.64\pm0.02$  &  $\bf1.66\pm0.02$  &  $\bf1.67\pm0.02$  \\
Canada      & $2.09\pm0.03$  &  $2.02\pm0.02$  &  $2.78\pm0.01$  &  $2.05\pm0.03$  &  $\bf1.78\pm0.03$  \\
Brazil      & $1.85\pm0.03$  &  $1.80\pm0.02$  &  $2.60\pm0.01$  &  $\bf1.60\pm0.02$  &  $1.76\pm0.02$  \\
Austrilia      & $2.52\pm0.04$  &  $1.92\pm0.02$  &  $2.15\pm0.01$  &  $1.99\pm0.02$  &  $\bf1.76\pm0.02$  \\
UK      & $2.53\pm0.03$  &  $2.15\pm0.02$  &  $2.32\pm0.01$  &  $1.93\pm0.02$  &  $\bf1.70\pm0.02$  \\
UAE      & $2.60\pm0.05$  &  $2.61\pm0.03$  &  $3.08\pm0.01$  &  $2.36\pm0.03$  &  $\bf2.19\pm0.03$  \\
Switzerland      & $2.94\pm0.04$  &  $2.66\pm0.03$  &  $3.29\pm0.01$  &  $2.88\pm0.03$  &  $\bf2.59\pm0.03$  \\
Qatar        & $2.85\pm0.04$  &  $2.63\pm0.04$  &  $4.24\pm0.01$  &  $\bf2.49\pm0.04$  &  $\bf2.56\pm0.03$  \\
Colombia      & $2.89\pm0.06$ &  $2.36\pm0.02$ &  $2.91\pm0.01$ &  $\bf2.06\pm0.02$ &  $2.25\pm0.02$ \\
Finland      & $3.41\pm0.05$ &  $\bf 2.29\pm0.02$ &  $2.97\pm0.01$ &  $2.65\pm0.03$ &  $2.49\pm0.03$ \\
Belgium      & $3.68\pm0.06$ &  $2.87\pm0.03$ &  $2.91\pm0.01$ &  $2.95\pm0.04$ &  $\bf2.55\pm0.03$ \\
Denmark      & $3.48\pm0.06$ &  $2.70\pm0.03$ &  $2.56\pm0.01$ &  $2.23\pm0.03$ &  $\bf1.87\pm0.02$ \\
Jordan      & $3.01\pm0.05$ &  $2.57\pm0.03$ &  $2.66\pm0.01$ &  $2.28\pm0.03$ &  $\bf2.08\pm0.03$ \\\hline
A.R.E.     & $100\%$   &  $85.72\pm0.37\%$    &  $106.05\pm0.47\%$   &  $81.30\pm0.33\%$    &  $\bf76.93\pm0.34\%$ \\
\bottomrule
\end{tabular}
\label{tab:pisa-off}
\end{table}

For offline experiments, we show the $\ell_1$ estimation errors of all methods on all individual countries in Table \ref{tab:pisa-off}.
The best result, \ie, outperforming all others regardless of standard deviations, is in bold font. 
For a global comparison, we define the averaged relative error $\sum_{m=1}^{(m)}\|\widehat{\beta}^{(m)}-\beta^{(m)}\|_1/(\sum_{m=1}^{(m)}\|\widehat{\beta}_{\rm ind}^{(m)}-\beta^{(m)}\|_1)$ where $\widehat{\beta}_{\rm ind}^{(m)}$ is the individual OLS estimate over task $m$. 
\ours outperforms other methods on most tasks and is the best in terms of the global error metric. 

For online experiments, we present the results for the tasks associated with Canada, UAE, and Denmark---which have activation probabilities $p_m= 0.64, \,0.32,\,0.22$, respectively---in Figure \ref{fig:pisa-on}. More figures can be found in Appendix~\ref{app:pisa_more}.
Again, we see that 
\oursb performs favorably compared to the baselines. 
This is more pronounced 
for the tasks with smaller activation probabilities, 
as suggested by the theory. 

\begin{figure}[t]
    \centering
    \includegraphics[height = 0.25\textwidth]{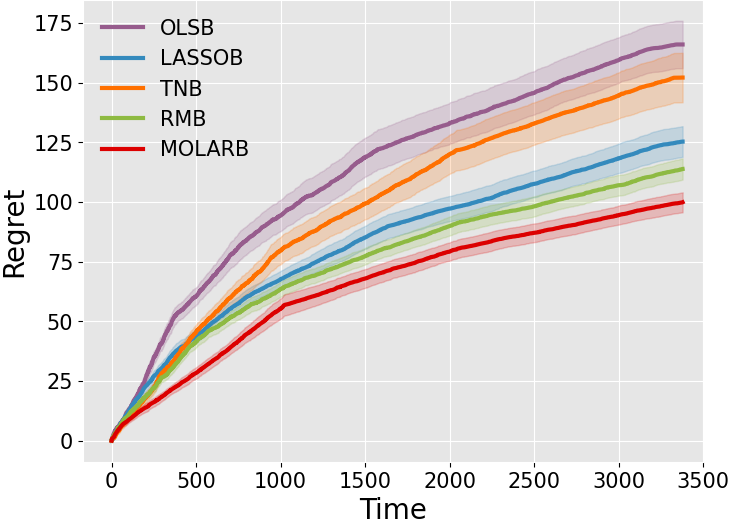}
    \hspace{-3.5mm}
    \includegraphics[height = 0.25\textwidth]{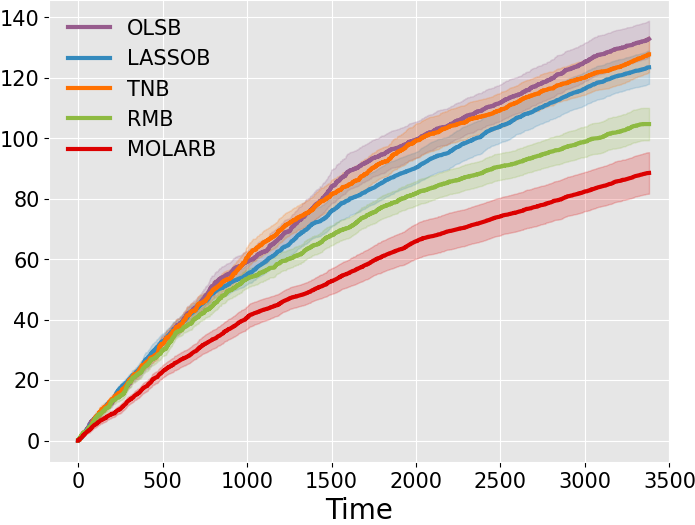}
    \hspace{-3.5mm}
    \includegraphics[height = 0.25\textwidth]{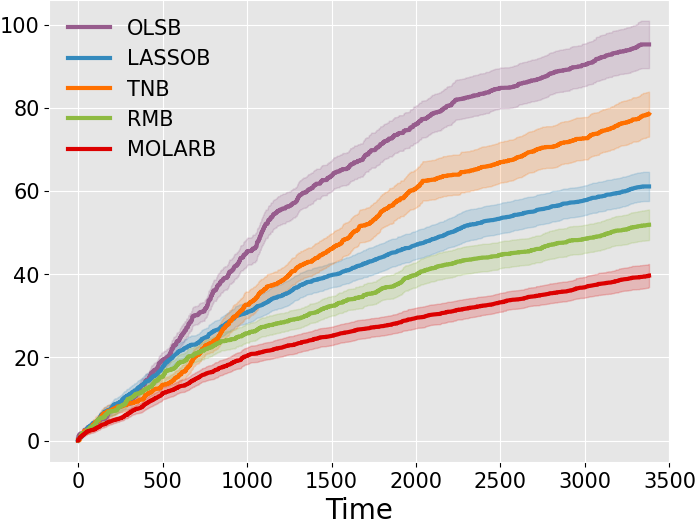}
    \caption{Regret $R_T^{(m)}$ of Canada, UAE, and Denmark of the PISA dataset. The
    shaded regions depict the corresponding
95\% normal confidence intervals based on standard errors calculated over twenty independent trials.}
    \label{fig:pisa-on}
\end{figure}

\section{Conclusion}
We consider multitask learning under sparse heterogeneity in both linear regression and linear contextual bandits.
For linear regression, we propose the \ours algorithm that collaborates on multiple datasets, improving accuracy compared to existing multitask methods. 
Applying \ours to linear contextual bandits, we also improve current regret bounds for individual bandit instances. To complement the upper bounds,
we establish lower bounds for multitask linear regression and contextual bandits, 
justifying the optimality of the proposed methods.
Our methods are also extended to generalized linear models and the construction of confidence intervals.
Our experimental results support our theoretical findings.
Future directions include investigating problem-specific 
optimal methods whose rate depends on $\{\beta^{(m)}\}_{m=1}^M$.

\section*{Acknowledgements}

We are grateful to Zhimei Ren for helpful discussions. Xinmeng Huang was supported in part by the NSF DMS 2046874 (CAREER), NSF CAREER award CIF-1943064;
 Donghwan Lee was supported in part by ARO W911NF-20-1-0080, DCIST, Air Force Office of Scientific Research Young Investigator Program (AFOSR-YIP) \#FA9550-20-1-0111 award.

{\small
\bibliography{references}
\bibliographystyle{plainnat-abbrev}
}

\newpage
\appendix

\section{More Related Works \& Notations}\label{app:ref}

\paragraph{Multitask Learning. }
When the data 
has components corresponding to multiple domains (also referred to as tasks or sources), 
multitask learning aims to develop methods that borrow information across tasks  \citep{caruana1998multitask}.
Multitask learning can be beneficial when the task-associated parameters are close in some sense, \eg, in the $\ell_2$ norm, or follow a common prior distribution \citep{raina2006constructing,hanneke2022no}. 
Popular multitask methods include regularizing the parameters to be estimated towards a common parameter---through ridge \citep{evgeniou2004regularized,hanzely2020lower}, $\ell_2$ \citep{duan2022adaptive}, kernel ridge \citep{evgeniou2005learning} penalties, etc---and 
clusteringpooling datasets based on similarity metrics  \citep{ben2010theory,crammer2008learning,Dobriban2018DistributedLR}. 
One can further leverage certain shared structures to improve rates of estimation. 
\cite{Tripuraneni2021ProvableMO,Du2021FewShotLV,Collins2021exploitingSR} study a low dimensional shared representation of task-specific models. 
\cite{lounici2009taking,singh2020online} consider the parameters for each task to be sparse and share the same support. \cite{bastani2021predicting,xu2021multitask,huang2022collaborative}
motivate and study sparse heterogeneity.


\paragraph{Robust Statistics \& Learning. }
In robust statistics and learning \citep{huber1981robust,hampel2011robust} many methods have been developed that are resilient to unknown data corruption \citep{rousseeuw1991tutorial,Minsker2015GeometricMA}.
From the optimization perspective, methods to robustly aggregate gradients of the loss functions have been developed \citep{Su2016FaultTolerantMO,Blanchard2017ByzantineTolerantML,Yin2018ByzantineRobustDL}. 
Our setting is different and requires a novel analysis.

\paragraph{Notations.}
We use $:=$ or $\triangleq$ to introduce definitions.
For an integer $d\ge 1$, we write $[d]$ for both $\{1,\dots,d\}$ and $\{e_1,\dots,e_d\}\subseteq \RR^d$, where $e_k$ is the $k$-th canonical basis vector of $\RR^d$. We use $I_{d}$ to denote the $d\times d$ identity matrix. 
We let $\mathbb{B}_d$ denote the unit Euclidean ball centered at the origin in $\RR^d$.
For a vector $v\in\RR^d$, we denote its entries as  $v_1,\dots,v_d$.
We also denote 
$\smash{\|v\|_p = (\sum_{k\in[d]} |v_k|^p)^{1/p}}$ for all $p>0$, with $\|v\|_0$ defined as the number of non-zero covariates.
For any $\cI\subseteq [M]$, given weights $\{w_m\}_{m=1}^M$ (or sample sizes $\{n_m\}_{m=1}^M$), we denote $W_{\cI}$ as $\sum_{m\in\cI}w_m$ and write $n_{\cI}$ for $\sum_{m\in\cI}n_m$. We also write $[v]_{\cI}$  and $v_{\cI}$ as the sub-vector of $v$ with entries in $\cI$
For a matrix $A\in\RR^{m\times n}$, we denote the $(i,j)$-th covariate of $A$ by  $[A]_{i,j}$ or $A_{i,j}$,
and the $i$-th row (resp., the $j$-the column) by $A_{i,\cdot}$ (resp., $A_{\cdot, j}$).
For two real numbers $a$ and $b$, we write $a\vee b$ and  $a\wedge b$ for $\max\{a,b\}$ and $\min\{a,b\}$, respectively.
For $\sigma^2>0$, we denote by $\subG(\sigma^2)$ the class of $\sigma^2$-sub-Gaussian random variables. For an event $E$, we write $\one(E)$ for the indicator of the event; so $\one(E)(x)=1$ if $x\in E$ and $\one(E)(x)=0$ otherwise.
We use the Bachmann-Landau asymptotic notations $\Omega(\cdot)$, $\Theta(\cdot)$, $O(\cdot)$ to absorb constant factors, and use $\tilde{\Omega}(\cdot)$, $\widetilde{O}(\cdot)$ to also absorb logarithmic factors
in various problem parameters specified in each case. 
Furthermore, we use probabilistic notations such as $O_P(a_{\{n_m\}_{m=1}^M})$ to denote quantities that are bounded by $a_{\{n_m\}_{m=1}^M}$ with overwhelming probabilities as $\min_{m\in[M]}{n_m}\to\infty$.
For a number $x\in\RR$, we use $(x)_{+}$ to denote its non-negative part, \ie, $x\one(x\geq 0)$.

\newpage
\section{Technical Lemmas}

\begin{lemma}[\sc Tail integral formula for expectation]\label{lem:integral}
For any non-negative, continuous random variable $Z$ with $\EE[Z]<\infty$ and any $q\ge 0$, it holds that 
\begin{equation*}
    \EE[Z\mathds{1}(Z\ge q)]=q\PP(Z\ge q)+\int_{q}^{\infty}\PP(Z\ge t)\d t.
\end{equation*}
\end{lemma}
\begin{proof}
    The  result is well known \citep[e.g., Exercise 1.2.3 in][]{vershynin2018high}.
\end{proof}

\begin{lemma}[\sc Maximal Inequalites]
\label{lem:max-ineq}
For $\sigma^2>0$
and for $1 \le m \le M$, 
let $X_m \sim \subG(\sigma^2)$,  not necessarily independent, for all $1 \le m \le M$.
Then, it holds that
\begin{enumerate}
    \item $\mathbb{E}[ \max_{1 \le m \le M} X_m] \le \sigma \sqrt{2 \ln (M)}$; 
    
    \item $\mathbb{E}[ \max_{1 \le m \le M} |X_m|] \le \sigma \sqrt{2 \ln(2M)}$;
    
    \item For any $t\ge 0$, $\mathbb{P}\left( \max_{1 \le m \le M} X_m \ge t \right) \le M \exp( -{t^2}/(2\sigma^2) )$.
\end{enumerate}
\end{lemma}
\begin{proof}
     The  result is well known \cite[e.g., ][]{koltchinskii2002rosenthal}. 
\end{proof}

\begin{lemma}[\sc Bernstein's inequality; \cite{Uspensky1937IntroductionTM}]\label{lem:bern}
Let $Z_{1}, \ldots, Z_{n}$ be $i . i . d .$  random variable  with $|Z_1-\EE[Z_1]|\le b$ and $\mathrm{Var}(Z_1)=\sigma^2>0$,  and let $\bar{Z}=\frac{1}{n}\sum_{i=1}^nZ_i$. 
Then for any $\delta \ge 0$,
\begin{equation}\label{eqn:jgowmfqww}
    \max\{\mathbb{P}(\bar{Z}-\mathbb{E}[\bar{Z}] > \delta),\mathbb{P}(\bar{Z}-\mathbb{E}[\bar{Z}] <- \delta)\} \le\exp \left(-\frac{n \delta^{2}}{2(\sigma^2+b\delta)}\right).
\end{equation}
\end{lemma}

\begin{lemma}[\sc Properties of Orlicz norm \citep{Smithies1962ConvexFA}]\label{lem:orlicz}
    For any $\alpha \in (0,2]$, the following properties hold when $\|Z\|_{\Psi_\alpha}$ exists.
    \begin{enumerate}
        \item Normalization: $\EE[\Psi_\alpha(|Z|/\|Z\|_{\Psi_\alpha})]\le 1$.
        \item Homogeneity: $\|cZ\|_{\Psi_\alpha}=c\|Z\|_{\Psi_\alpha}$ for any $c\in\RR$.
        \item Deviation inequality: $\PP(|Z|\ge t)\le 2 \exp(-(t/\|Z\|_{\Psi_\alpha})^\alpha)$.
    \end{enumerate}
\end{lemma}

\begin{lemma}[\sc Berry-Esseen theorem \citep{shevtsova2010improvement}]\label{lem:berry}
    Given independent random variables $\{Z_t\}_{i=1}^n$ with $\EE[Z_i]=0$, $\EE[Z_i]=\sigma_i^2\ge 0$, and $\EE[|Z_i|^3]=\rho_i<\infty$, let $S_n = {\sum_{i=1}^n Z_i}/{\sqrt{\sum_{i=1}^n \sigma_i^2}}$ be the normalized  sum, and denote $F_n$ the c.d.f.~of $S_n$, and $\Phi$ the c.d.f.~of the standard normal distribution. It holds that 
    \begin{equation*}
        \sup_{z\in \RR}|F_n(z)-\Phi(z)|\le 0.6 \sum_{t=1}^n \rho_t /\left(\sum_{t=1}^n \sigma_t^2\right)^{3/2}.
    \end{equation*}
\end{lemma}

\begin{lemma}[\sc Generalized Hanson-Wright inequality \citep{Gtze2021ConcentrationIF,sambale2020some}]\label{lem:hw-ineq}
    For any $\alpha \in(0,2]$, let $Z_1,\dots,Z_n$ be i.i.d.~zero-mean  random
variables with $\|Z_1\|_{\Psi_\alpha}\le \sigma$  and $A=(a_{i,j})$ be a symmetric matrix.
Then there is an absolute constant $c_{\rm hw}$ such  that for any $t\ge 0$, 
\begin{equation*}
    \PP\left(\left|\sum_{i,j\in [n]}a_{i,j}Z_iZ_j-\sum_{i\in [n]}a_{i,i}\mathrm{Var}(Z_1)\right|\ge t\right)\le 2\exp\left(-\frac{1}{c_{\rm hw}}\min\left\{\frac{t^2}{\sigma^4\|A\|_{\rm F}^2},\left(\frac{t}{\sigma^2\|A\|_{\rm op}^2}\right)^{\alpha /2}\right\}\right)
\end{equation*}
where $\|\cdot\|_{\rm F}$ and $\|\cdot\|_{\rm op}$ indicates the Frobenius norm and the operator norm, respectively.
In particular, when $A=vv^\top$ with $v=(v_1,\dots,v_n)^\top \in\RR^n$, it holds for any $t\ge \sigma^2\|v\|_2^2$ that 
\begin{equation*}
    \PP\left(\left|\left(\sum_{i\in [n]}v_{i}Z_i\right)^2-\|v\|_2^2\mathrm{Var}(Z_1)\right|\ge t\right)\le 2\exp\left(-\frac{1}{c_{\rm hw}}\left(\frac{t}{\sigma^2\|v\|_{2}^2}\right)^{\alpha/2}\right).
\end{equation*}
\end{lemma}

\begin{lemma}[\sc Joint convexity of the KL-divergence \citep{Cover1991ElementsOI}]\label{lem:convexity}
The Kullback-Leibler divergence $D_{\rm KL}(P\,\|\,Q)$ is jointly convex
in its arguments $P$ and $Q$: let $P_1$, $P_2$, $Q_1$, $Q_2$ be distributions on a common set $\cX$, then for any $\lambda\in[0,1]$, it holds that 
\begin{equation*}
    D_{\rm KL}(\lambda P_1 +(1-\lambda)P_2\,\|\,\lambda Q_1+(1-\lambda)Q_2)\le \lambda D_{\rm KL}(P_1\,\|\,Q_1) + (1-\lambda) D_{\rm KL}(P_2\,\|\,Q_2).
\end{equation*}
More generally, if the parameter $\theta$ follows some prior $R$ and if conditioned on the parameter $\theta$, the random variables $X\sim P_\theta$ and $Y\sim Q_{\theta}$, then
\begin{equation*}
    D_{\rm KL}(P_\theta \circ R\,\|\,Q_\theta \circ R)\le\EE_{\theta \sim R}[D_{\rm KL}(P_\theta\,\|\,Q_\theta)],
\end{equation*}
\end{lemma}

\begin{lemma}[{\sc Lemma 8 of }\cite{Han2020SequentialBL}]\label{lem:moment-fund}\label{lem:moment}
Suppose 
that for $d\ge 2 $, $Z\in\RR^d$ is uniformly distributed on the source $(d-1)$-dimensional sphere, then the absolute moments of the first coordinate $[Z]_1$ of $Z$ are, for $k>-1$
\begin{equation*}
    \EE[|[Z]_1|^k]=\frac{\Gamma(\frac{d}{2})\Gamma(\frac{k+1}{2})}{\Gamma(\frac{d+k}{2})\Gamma(\frac{1}{2})}
\end{equation*}
where $\Gamma(\cdot)$ is the gamma function. 
\end{lemma}

\begin{definition}[\sc Maximum expectation over small probability sets]\label{def:small-set-ingeral}
Given a probability space $A=(\Omega, \PP, \cF)$, a random variable $Z$ over $A$, and $\delta \in[0,1]$, 
define $\EE^\delta[Z]$ as the maximum expectation of $Z$ over all measurable sets with probability at most $\delta$:
\begin{equation}\label{is}
    \EE^\delta[Z]:=\sup_{A\in \cF}\left\{\EE[Z(\omega)\one (\omega\in A)]: \,\PP(A)\le \delta \right\}.
\end{equation}
\end{definition}

\begin{lemma}[\sc Sub-Gaussian integral over small probability sets]\label{lem:small-set-ingeral}
For any $Z\sim \subG(\sigma^2)$ with $\sigma^2>0$ and $\EE[Z]=0$, the maximum expectation over small probability sets from \eqref{is} of $|Z|$ and $Z^2$ is bounded as 
\begin{equation}\label{edz}
    \EE^\delta[|Z|]=O\left( \delta \sigma \ln(2/\delta)^{1/2}\right)
\end{equation}
and $\EE^\delta[Z^2]=O\left( \delta \sigma^2 \ln(2/\delta)\right)$.
Moreover, if $|Z|\ge |\widetilde{Z}|$ a.s., then $\EE^\delta[|\widetilde{Z}|]\le \EE^\delta[|Z|]$ for any $\delta \in[0,1]$.

\end{lemma}
\begin{proof}
We first consider a continuous random variable $Z$ over a probability space $A=(\Omega, \PP, \cF)$.
Then, 
from \eqref{is},
$\EE^\delta[|Z|]$ is given by the integral over
\begin{equation*}
    A_\delta=\{\omega \in \Omega:\,|Z(\omega)|\ge q_\delta\}\quad \text{with}\quad \PP(|Z(\omega)|\ge q_\delta)=\delta.
\end{equation*}
Since $Z\sim \subG(\sigma^2)$, by the Chernoff bound, we have 
$\delta =\PP(|Z(\omega)|\ge q_\delta)\le 2\exp\left(-\frac{q_\delta^2}{2\sigma^2}\right)$,
which implies 
\begin{equation}\label{eqn:jiodsfasd}
    q_\delta \le \sqrt{2\ln(2/\delta)}\sigma.
\end{equation}
Plugging \eqref{eqn:jiodsfasd} and the Chernoff bound into Lemma \ref{lem:integral}, we find
\begin{align}
    \EE[|Z|\mathds{1}(|Z|\ge q_\delta)]&=q_\delta \PP(|Z|\ge q_\delta)+\int_{q_\delta}^\infty \PP(|Z|\ge t)\d t\nonumber\\
    &\le \delta q_\delta+\int_{q_\delta}^\infty\min\left\{\delta, 2\exp\left(-\frac{t^2}{2\sigma^2}\right)\right\}\d t\label{eqn:vbjiefa}
\end{align}
where the last inequality follows from the Chernoff bound and from 
$\PP(|Z|\ge t)\le \PP(|Z|\ge q_\delta)= \delta$.
Now,
\begin{align}
    &\int_{q_\delta}^\infty\min\left\{\delta, 2\exp\left(-\frac{t^2}{2\sigma^2}\right)\right\}\d t
    =\delta \left(\sqrt{2\ln(2/\delta)}\sigma-q_\delta\right)+2\int_{\sqrt{2\ln(2/\delta)}\sigma}^\infty \exp\left(-\frac{t^2}{2\sigma^2}\right)\d t\nonumber\\
    &=\delta \left(\sqrt{2\ln(2/\delta)}\sigma-q_\delta\right)+2\sqrt{2\pi}\sigma \cdot\PP\left(\cN(0,\sigma)\ge \sqrt{2\ln(2/\delta)}\sigma\right).\label{eqn:jvboimfwe}
\end{align}
Using \citep[Proposition 2.1.2]{vershynin2018high}, we have
\begin{align}
    &\PP\left(\cN(0,\sigma^2)\ge \sqrt{2\ln(2/\delta)}\sigma\right)=
    \PP\left(\cN(0,1)\ge \sqrt{2\ln(2/\delta)}\right)\nonumber\\
    &\le \frac{1}{\sqrt{2\pi}\sqrt{2\ln(2/\delta)}}\exp\left(-\ln(2/\delta)\right)
    =\frac{\delta}{4\sqrt{\pi\ln(2/\delta)}}.\label{eqn:vbiofoa}
\end{align}
Combining \eqref{eqn:vbiofoa} with \eqref{eqn:jvboimfwe} and \eqref{eqn:vbjiefa}, we find \eqref{edz}.
For random variables $Z$ that are not necessarily continuous, let $Z_\ep := \sqrt{1-\ep}\,Z+\sqrt{\ep}\,Z'$ with independent Gaussian $Z'\sim\cN(0,\sigma^2)$ and $0\le\ep\le 1$. 
Clearly, 
$Z_\ep\sim \subG(\sigma^2)$ and is continuous. 
By the result of the continuous case, we have 
\begin{equation}\label{eqn:hirengewf}
    \delta \sigma \left(\left[2\ln(2/\delta)\right]^{1/2}+\left[2\ln(2/\delta)\right]^{-1/2}\right)
    \ge \EE^\delta[|Z_\ep|]\ge \sqrt{1-\ep}\, \EE^\delta[|Z|]-\sqrt{\ep}\,\EE^\delta[|Z'|].
\end{equation}
Letting $\ep\rightarrow 0$ in the right-hand side of \eqref{eqn:hirengewf}, \eqref{edz} follows for general $Z$.
The result for $\EE^\delta[Z^2]$ follows similarly.
\end{proof}

\begin{lemma}\label{lem:f-g}
Given $\eta\in(0,1]$, $r_k\ge 0$ for all $k\in[d]$, $a\ge 0$,
and for $p\in\{1,2\}$,
consider the functions $f_q: \{x\in\RR^d: 0\le x_k\le 1,\,\forall\, k\in[d] \text{ and }\sum_{k\in[d]}x_k\le s\}\to\RR$,  $f_p(x_1,\dots,x_d):=\sum_{k\in[d]}(x_k\wedge a)^p\mathds{1}\{x_k<\eta \})+a^p\mathds{1}\{x_k\ge \eta \}$.
Then it holds that 
\begin{equation*}
    \max_{x_1,\dots,x_d}f_p(x_1\dots,x_d)\le a^p\left(\left\lceil \frac{s}{ a\wedge \eta }\right\rceil \wedge d\right).
\end{equation*}
\end{lemma}
\begin{proof}
We only prove the result for $f_1$, and the result for function $f_2$ follows similarly. Without loss of generality, we assume $1\ge x_1\ge x_2\ge \cdots\ge x_d\ge0$ and $\lceil {s}/{ (a\wedge \eta) }\rceil <d$ since $f_p(x_1\dots,x_d)\le d a^p$ is clear. In this case, we claim that the maximum can be attained at $x_1=\cdots=x_{\lfloor s/(a\wedge \eta) \rfloor}=a\wedge \eta$, $x_{\lfloor s/(a\wedge \eta)  \rfloor+1}=s-\eta\lfloor s/(a\wedge \eta)  \rfloor $, and $x_k=0$ for all $k>\lfloor s/\eta \rfloor+1$.
Further, the maximum is upper bounded by $a^p\left(\lceil{s}/{(a\wedge \eta )}\rceil \wedge d\right)$. We now use the exchange argument to prove the claim.
\begin{itemize}
    \item[\quad S. 1] If there is some $k$ such that $x_k>a\wedge \eta  \ge x_{k+1}$, then defining $x'$ by letting $(x_k^\prime ,x_{k+1}^\prime )=((a\wedge \eta) , x_k+x_{k+1}-(a\wedge \eta) )$ while for other $j$, $x_j^\prime=x_j$, does not decrease the value of $f_1$. Therefore, the maximum is attained by $x$ such that  for some $j$, $x_1 = \dots = x_j = (a\wedge \eta)  > x_{j+1} \ge  \cdots \ge x_d$.
    \item[S. 2] If  there is some $k$ such that $(a\wedge \eta)  >x_k\ge   x_{k+1}>0$, then defining $x^\prime$ by letting  $(x_k^\prime ,x_{k+1}^\prime )=(\min\{(a\wedge \eta) ,x_k+x_{k+1}\}, \max\{0,x_k+x_{k+1}-(a\wedge \eta) \})$  while for other $j$, $x'_j=x_j$, does not decrease the value of $f_1$. 
    Therefore, combined with Step 1, the maximum is attained by $x$ such that  for some $j$,  $x_1 = \dots = x_j = a\wedge \eta > x_{j+1} \ge  0$ and $x_k=0$ for all $k>j+1$. 
    Thus at most one element lies in $(0,\eta)$.
\end{itemize}
Combining S. 1 and S. 2 above, we complete the proof of the claim, which further leads to the conclusion.
\end{proof}

\newpage
\section{Results on Linear Regression with Gaussian Noise}
\subsection{Lemma \ref{lem:median-Gaussian} and its Proof}\label{app:median-Gaussian}
 Lemma  \ref{lem:median-Gaussian} characterizes the estimation error of the median of Gaussian inputs.  
This is similar to classical results from
robust statistics \citep[see, \eg,][]{lerasle2011robust}, 
but existing results typically assume that the  ``inlier data'' is an i.i.d.~sample. 
In contrast, we only require independence, since we wish to apply it to the non-i.i.d.~variables $\{\widehat{\beta}^{(m)}_{\rm ind}\}_{m=1}^M$.

\begin{lemma}\label{lem:gaussian-quant}
    Given independent Gaussian random variables $\{Z_i\sim \cN(\mu,\sigma_i^2)\}_{i\in \cI}$
    with a shared mean but with possibly different variances, 
    and non-negative weights $\{ w_i\}_{i\in \cI}$ with $W_{\cI}\triangleq \sum_{i\in\cI}w_i$,
 then $1/2+\alpha$- weighted population quantile $\mu_{1/2+\alpha}$ is defined such that
\begin{equation}\label{eqn:weight-quant}
     \sum_{i\in\cI}w_i\Phi\left(\frac{\mu_{1/2+\alpha}-\mu}{\sigma_i}\right)=\left(\frac{1}{2}+\alpha\right)W_{\cI}.
\end{equation}
Then it holds for any $|\alpha|<1/2$ that
\begin{equation*}
    |\mu_{1/2+\alpha}-\mu|\leq C_{\alpha} \alpha \bar \sigma_{\cI} 
\end{equation*}
where $\bar \sigma_{\cI} =\sum_{i\in \cI}w_i\sigma_i/W_{\cI}$, $C_\alpha\triangleq \max_{0<\epsilon<1/2-\alpha}\left\{\phi(\Phi^{-1}(1-\epsilon))(1- \frac{2\alpha}{1-2\epsilon})\right\}^{-1}$, and $\phi$, $\Phi$ are the density and c.d.f.~of the standard normal distribution, respectively.
\end{lemma}
\begin{proof}
 We only prove the case where $\alpha\geq 0$ and the other case follows by symmetry. 
 We denote the normalized weight $w_i/W_{\cI}$ as $\bar w_i$. 
Clearly, $\mu_{1/2+\alpha}\geq \mu$ for $\alpha\geq 0$, so we can divide $\cI$ into to two groups based on weighted probabilities:
        \begin{align}\nonumber
        \cI_{\rm small}:=\{{i\in \cI}:\Phi((\mu_{1/2+\alpha}-\mu)/\sigma_i) \leq 1-\epsilon\},\quad 
        \cI_{\rm large}:=\cI\backslash\cI_{\rm small},
    \end{align}
    where  $0< \epsilon<1/2-\alpha$ is a real number to be chosen later.
    Using the mean-value theorem, for all $0\leq z\leq \Phi^{-1}(1-\epsilon)$, there exists $\xi \in (0, z)$ such that
    \begin{equation}\nonumber
        \Phi(z)= \frac{1}{2}+z\phi(\xi)\geq \frac{1}{2}+z\phi(\Phi^{-1}(1-\epsilon)).
    \end{equation}
    where $\phi$ is the density of the standard normal distribution. 
    We thus have
\begin{equation}\label{eqn:daivzxcvzxvc}
        \frac{1}{2}+\alpha\geq \sum_{i\in\cI_{\rm small}}\bar w_i \left(\frac{1}{2}+\phi(\Phi^{-1}(1-\epsilon))\sigma_i^{-1}(\mu_{1/2+\alpha}-\mu)\right)+\sum_{i\in\cI_{\rm large}}\bar w_i(1-\epsilon),
    \end{equation}
    leading to 
    \begin{align}\label{eqn:nvbidsnfgzds}
        \mu_{1/2+\alpha}\leq &\mu+\alpha /\left(\phi(\Phi^{-1}(1-\epsilon))\sum_{{i\in \cI_{\rm small}}}\bar w_i/\sigma_i\right).
    \end{align}
On the other hand, we have from \eqref{eqn:daivzxcvzxvc} that
\begin{equation}\label{eqn:nvbidsnfgzds2}\nonumber
        \sum_{{i\in \cI_{\rm large}}}w_i\leq \frac{\alpha}{1/2-\epsilon},
    \end{equation}
    which implies
    \begin{equation}\label{eqn:nvbidsnfgzds3}\nonumber
        \sum_{i\in \cI_{\rm small}}w_i\geq 1- \frac{\alpha}{1/2-\epsilon}.
    \end{equation}
Consequently, using H\"older's inequality, we obtain
\begin{align}\label{eqn:nvbidsnfgzds4}
    \sum_{i\in\cI_{\rm small}}\bar w_{i}/\sigma_i\sum_{i\in\cI_{\rm small}}\bar w_{i}\sigma_i\geq \left(\sum_{i\in\cI_{\rm small}}\bar w_i\right)^2\geq \left(1- \frac{\alpha}{1/2-\epsilon}\right)^2.
\end{align}
    Combing \eqref{eqn:nvbidsnfgzds} and \eqref{eqn:nvbidsnfgzds4}, we obtain
    \begin{equation}\label{eqn:nvicxnvcxdd}\nonumber
        \mu_{1/2+\alpha}\leq \mu+ C_\alpha\sum_{i\in\cI_{\rm small}}\bar w_i \sigma_i,
    \end{equation}
    where $ C_\alpha$ is defined in the statement.
\end{proof}

\begin{lemma}\label{lem:median-Gaussian}
Given  independent Gaussian random variables $\{Z_m\sim \cN(\mu_m,\sigma_m^2)\}_{m=1}^M$, any positive weights $\{w_m\}_{m=1}^M$,
and some $\mu\in\RR$, let $\cB\triangleq \{m\in[M]:\mu_m \neq \mu\}$ and $\cG\triangleq [M]\backslash \cB$.
If $|\cB|< M$,
for any $\delta \ge 0$ such that
\begin{equation}\label{eqn:cond}
    \alpha_{\cB,\delta}\triangleq\sum_{m\in\cB}w_m/W_{[M]}+\sqrt{1.01\delta \sum_{m\in\cG}w_m^2}/W_{\cG}<\frac{1}{2},
\end{equation}
it holds with probability at least $1-2e^{-2\delta}$ that
\begin{equation*}
    |\wmed(\{Z_m\}_{m\in[M]};\{w_m\}_{m\in[M]})-\mu|\le  C_{\alpha_{\cB,\delta}}\alpha_{\cB,\delta}\bar \sigma_{\cG},
\end{equation*}
where $\bar \sigma_{\cG} \triangleq \sum_{m\in\cG}w_m\sigma_m/W_{\cG}$ and  $C_\alpha$ is the constant depending only on $\alpha$ defined in Lemma \ref{lem:gaussian-quant}.
\end{lemma}
\begin{proof}
Denote $\cB^c$ as $\cG$ for notational simplicity. 
For all $z\in \RR$,
let $\widehat{F}_{\cG}(z):=\sum_{m\in\cG}w_m\mathds{1}(Z_n\le z)/W_{\cG}$ and $\widehat{F}_{[M]}(z):=\sum_{m\in[M]}w_m\mathds{1}(Z_m\le z)/W_{[M]}$ be the weighted empirical distributions of $\{Z_m\}_{m\in\cG}$ and $\{Z_m\}_{m=1}^M$, respectively. 
Then we have 
\begin{equation}\nonumber
    \EE[\widehat{F}_{\cG}(z)]=\sum_{m\in\cG}w_m\mathds{1}(Z_{m}\le z)/W_{\cG}=\sum_{m\in\cG}w_m\Phi\left( \frac{z-\mu}{\sigma_m}\right)/W_{\cG}.
\end{equation}
By the condition~\eqref{eqn:cond}
on 
$\alpha_{\cB,\delta}$,
there are unique values $z_\text{high}$ and $z_\text{low}$ such that 
\begin{align*}
   \sum_{m\in\cG}w_m\Phi\left( \frac{z_{\rm high}-\mu}{\sigma_m}\right)=& \left(\frac{1}{2}+\alpha_{\cB,\delta}\right)W_{\cG},\\
    \sum_{m\in\cG}w_m\Phi\left( \frac{z_{\rm low}-\mu}{\sigma_m}\right)=& \left(\frac{1}{2}-\alpha_{\cB,\delta}\right)W_{\cG}.
\end{align*}
By Hoeffding’s inequality,  for any given $\delta\ge 0$ and $z\in\RR$, we have with probability at least $1-e^{-2\delta}$ that
\begin{equation}\label{eqn:iogdgsd2-w}\nonumber
    \widehat{F}_{\cG}(z)-W_{\cG}^{-1}\sum_{m\in\cG}w_m\Phi\left( \frac{z-\mu}{\sigma_m}\right)\leq \sqrt{\delta \sum_{m\in\cG}w_m^2}/W_{\cG}.
\end{equation}
It is not hard to verify that for all $z\in\RR$,
\begin{equation}\label{eqn:iogdgsd}
    \left|\widehat{F}_{\cG}(z)-\widehat{F}_{[M]}(z)\right|\le \frac{W_{\cB}}{W_{[M]}}.
\end{equation}
We thus have 
\begin{align*}
    \widehat{F}_{[M]}(z_{\rm high})\geq &\widehat{F}_{\cG}(z_{\rm high})-\frac{W_{\cB}}{W_{[M]}}\\
    \geq &W_{\cG}^{-1}\sum_{m\in\cG}w_m\Phi\left( \frac{z-\mu}{\sigma_m}\right)-\frac{W_{\cB}}{W_{[M]}}-\sqrt{\delta \sum_{m\in\cG}w_m^2}/W_{\cG}>\frac{1}{2}.
\end{align*}
Similarly, we have $\widehat{F}_{[M]}(z_{\rm low})<1/2$. Further using Lemma \ref{lem:gaussian-quant} leads to the conclusion.
\end{proof}

\subsection{Lemma \ref{lem:col-est} and its Proof}\label{app:col-est}
\begin{lemma}\label{lem:col-est}
Under Conditions \ref{asp:noise} and \ref{asp:design-mats}, for any $0<\eta \le \frac{1}{5}$ and $k\in\cI_\eta$, it holds for any $0\le \delta \le W_{[M]}/(21 \max_{m\in[M]}w_m)$ that
\begin{equation*}
    \PP\left(|\widehat{\beta}^\star_k-\beta^\star_k|\ge 1.25C_{0.45}\alpha_{\cB_k,\delta}\bar \sigma_{[M],k}\mid \{\vX^{(m)}\}_{m=1}^M\right)\le 2e^{-2\delta},
\end{equation*}
where $\bar {\sigma}_{[M],k}=\sum_{m\in[M]}w_m\sigma_m\sqrt{[(\vX^{(m)\top}\vX^{(m)})^{-1}]_{k,k}}/W_{[M]}$ and $\alpha_{\cB_k,\delta}$
follows from the definition in \eqref{eqn:cond}.
\end{lemma}
\begin{proof}
Since $W_{\cB_k}\le \eta W_{[M]}$ for any $k\in\cI_\eta$ and $\eta \le \frac{1}{5}$, we have for each $k\in\cI_\eta$ and any $0\le \delta \le 1/(25\sum_{m\in\cG_k}\bar w_m^2)$ that 
\begin{align*}
    \alpha_{\cB_k,\delta}
    =&\frac{W_{\cB_k}}{W_{[M]}}+\sqrt{1.01\delta \sum_{m\in\cG_k}w_m^2}/W_{\cG_k}
    \leq \frac{W_{\cB_k}}{W_{[M]}}+\sqrt{1.01\delta \max_{m\in[M]}w_m}/W_{\cG_k}^{1/2}\\
    \leq &\frac{W_{\cB_k}}{W_{[M]}}+\sqrt{1.01\delta \max_{m\in[M]}w_m}/(0.5 W_{[M]})^{1/2}
    \le \frac{1}{5}+\frac{1}{4}=0.45.
\end{align*}
Therefore, the condition \eqref{eqn:cond} from Lemma
\ref{lem:median-Gaussian} is satisfied with $\alpha=0.45$.  
Thus, by Lemma \ref{lem:median-Gaussian}, we have for any $0\le \delta \le W_{[M]}/(21 \max_{m\in[M]}w_m)$ that
\begin{equation}\label{eqn:vnisdngvzxc1}
    \PP\left(|\widehat{\beta}^\star_k-\beta^\star_k|\ge C_{0.45} \bar \alpha_{\cB_k,\delta}\sigma_{\cG_k} \mid \{\vX^{(m)}\}_{m=1}^M\right)\le 2e^{-2\delta}.
\end{equation}
Furthermore, using $W_{\cG_k}\geq 4W_{[M]}/5$, we have 
\begin{equation}\label{eqn:vnisdngvzxc2}
    \bar {\sigma}_{\cG_k}=\sum_{m\in\cG_k}w_m\sigma_m/W_{\cG_k}\leq 5\sum_{m\in\cG_k}w_m\sigma_m/(4W_{[M]})=1.25 \bar \sigma_{[M],k}.
\end{equation}
Combining \eqref{eqn:vnisdngvzxc1} with \eqref{eqn:vnisdngvzxc2} completes the proof.
\end{proof}

\subsection{Proof of Proposition \ref{prop:center-covariate}}\label{app:center-covariate}

\begin{proof}

For simplicity, we only prove the case $n_1\geq \cdots\geq n_M$ and $ \sigma_1=\cdots=\sigma_M=\sigma$ for some $\sigma>0$, 
and thus $w_m=n_m/\sigma^2$ for all $m\in[M]$. 
The case of heterogeneous variances follows by considering the rescaled sample size $\tilde n_m=n_m\sigma^2/\sigma_m^2$ for each $m\in [M]$.

For each $k\in[d]$, let $\bar {\sigma}_k\triangleq\sum_{m\in[M]}w_m\sqrt{v^{(m)}_k}\sigma/W_{[M]}$ where $v_k^{(m)}= \sqrt{[(\vX^{(m)\top}\vX^{(m)})^{-1}]_{k,k}}$.
Due to Condition \ref{asp:size}, we have 
\begin{align}\nonumber
    W_{[M]}/(21 \max_{m\in[M]}w_m)
    \geq  n_{[M]}/(21 n_1) \geq (21\ln((n_{\rm [M]}/n_M) \wedge (d/s))/c_{\rm s})^2,
\end{align}
where the last inequality follows from Condition \ref{asp:size} and the choice of $\{w_m\}_{m=1}^M$ and $c_{\rm s}$ is defined in Condition \ref{asp:size}.

Taking $\delta \triangleq (21\ln(n_{[M]}/n_M \wedge (d/s))/c_{\rm s})^2$ in Lemma \ref{lem:col-est},   we have 
\begin{equation}\label{eqn:vuisd}
    \PP\left(|\widehat{\beta}^\star_k-\beta^\star_k|\ge 1.25C_{0.45}  \alpha_{\cB_k,\delta}\bar{\sigma}_k\mid \{\vX^{(m)}\}_{m=1}^M\right)\le 2 e^{-2(21\ln(n_{[M]}/n_M \wedge (d/s))/c_{\rm s})^2}=O\left(\frac{n_M}{n_{[M]}}\bigvee \frac{s}{d}\right).
\end{equation}

On the other hand, by using a standard $\epsilon$-net argument (see \eg,  \cite{vershynin2018high}), one can show that the event
\begin{equation}\nonumber
    E\triangleq \{\vX^{(m)\top}\vX^{(m)}\gtrsim \mu n_m I_d/2,\, \forall\, m\in[M]\}
\end{equation}
holds with probability at least $1-O(Mde^{-cn_M})$ where $c$ is a constant only depending on $\mu$ and $L$. Since event $E$ implies $\sqrt{v_k^{(m)}}=O(1/\sqrt{n_m})$ and thus 
\begin{equation}\nonumber
    \bar \sigma_k=O\left(\sum_{m\in[M]}\sqrt{n_m}/\sigma_m/\sum_{m\in[M]}{n_m}/\sigma_m^2\right),
\end{equation}  combining with \eqref{eqn:vuisd}, we complete the proof.

\end{proof}

\subsection{Proof of Theorem \ref{thm:ls-fixed-design}}\label{app:ls-fixed-design}
We provide the proof for $p=1$ and \textsf{Option I} under identical variances, \ie, $\sigma_1=\cdots=\sigma_M=\sigma$; the case $p=2$, \textsf{Option II}, or heterogeneous noise variances follows
similarly.  Let $\cI^{(m)}:=\{k\in[d]:\beta^{(m)}_k=\beta^\star_k\}$  and recall $\cI_\eta$ from \eqref{ieta}. For all $m\in[M]$, we provide a series of bounds for $|\widehat{\beta}^{(m)}_{{\rm MOLAR},k}-\beta^{(m)}_k|$ for each $k\in[d] $ in three cases. We denote $\widehat{\beta}^{(m)}_{\rm MOLAR}$ as $\widehat{\beta}^{(m)}$ below for simplicity.  First noting that $v_k^{(m)}=O(t/\sqrt{n_m})$ with probability at least $1-de^{-cn_m t^2}$ for some constant $c$ (see, \eg, \cite{vershynin2018high}), we have $v_k^{(m)}=O_P(\ln(d)/\sqrt{n_m}),\forall\,m\in[M]$.

\vspace{2mm}
\noindent\textbf{Case 1. } For any $k \in [d]$, we guarantee that
\begin{equation}\label{eqn:case1}
    \EE[|\widehat{\beta}^{(m)}_{k}-\beta^{(m)}_k|\mid \{\vX^{(m)}\}_{m=1}^M]=\widetilde{O}_P(\sigma/\sqrt{n_m}).
\end{equation} 
 By definition, for any $m\in[M]$ and $k\in[d]$, $\widehat{\beta}^{(m)}_{k}$ is either equal to $\smash{\widehat{\beta}^{(m)}_{\mathrm{ind},k}}$ or $\widehat{\beta}^\star_k$, and the latter happens only when $|\widehat{\beta}^\star_k-\widehat{\beta}^{(m)}_{\mathrm{ind},k}|\le  \gamma_m \sqrt{v^{(m)}_k}$. In the latter case, we have 
\begin{equation*}
    |\widehat{\beta}^{(m)}_{k}-\beta^{(m)}_k|=|\widehat{\beta}^\star_k-\beta^{(m)}_k|\le |\widehat{\beta}^{(m)}_{k}-\beta^{(m)}_k|+|\widehat{\beta}^{(m)}_{\mathrm{ind},k}-\widehat{\beta}^\star_k|\le  |\widehat{\beta}^{(m)}_{k}-\beta^{(m)}_k|+\gamma_m \sqrt{v^{(m)}_k}.
\end{equation*}
Therefore, in both cases,
\begin{equation}\label{eqn:nbiwfgew}
    |\widehat{\beta}^{(m)}_{k}-\beta^{(m)}_k|\le |\widehat{\beta}^{(m)}_{\mathrm{ind},k}-\beta^{(m)}_k|+ \gamma_m \sqrt{v^{(m)}_k}.
\end{equation}
By \eqref{lem:ind-est},  $\widehat{\beta}^{(m)}_{\mathrm{ind},k}-\beta^{(m)}_k\mid v_k^{(m)} \sim \cN(0,\sigma^2v^{(m)}_k)$, we have $\EE[|\widehat{\beta}^{(m)}_{\mathrm{ind},k}-\beta^{(m)}_k|\mid v_k^{(m)}] = O(\sigma v_k^{(m)})=\widetilde O_P(\sigma/\sqrt{n_m})$.

\vspace{2mm}
\noindent\textbf{Case 2. } When $k\in \cI^{(m)} \cap \cI_\eta$, 
we can obtain the improved bound
\begin{equation}\label{eqn:case3}
    \EE[|\widehat{\beta}^{(m)}_{k}-\beta^{(m)}_k|\mid  \{\vX^{(m)}\}_{m=1}^M]=\widetilde{O}_P\left(\frac{W_{\cB_k}\sigma }{W_{[M]}\sqrt{n_1}} + \frac{\sigma}{\sqrt{n_{[M]}}}+\frac{s\sigma /d}{\sqrt{n_m}}\right).
\end{equation}
Let $\delta =(21\ln(n_{[M]}/n_M \wedge (d/s))/c_{\rm s})^2$ and $\bar {\sigma}_k=\sum_{m\in[M]}w_m\sqrt{v^{(m)}_k}\sigma/W_{[M]}$ as stated in Section \ref{app:center-covariate}.
Without loss of generality, 
we consider $1.25C_{0.45} \alpha_{\cB_k,\delta} \bar {\sigma}_k\le \sigma \sqrt{v^{(m)}_k}$, otherwise \eqref{eqn:case3} is implied by \eqref{eqn:case1}. 
Define the event $\cE_k=\{|\widehat{\beta}^\star_k-\beta^\star_k|\le 1.25C_{0.45} \alpha_{\cB_k,\delta} \bar {\sigma}_k\}$.
 By Lemma \ref{lem:col-est}, we have $\PP((\cE_k)^c)\le O((n_M/n_{[M]})\wedge (s/d))$. 
 Furthermore, by the condition $1.25C_{0.45}\alpha_{\cB_k,\delta} \bar {\sigma}_k\le \sigma \sqrt{v^{(m)}_k}$, we have that the event $\cE_k$ implies $|\widehat{\beta}^\star_k-\beta^\star_k|\le 1.25C_{0.45}\alpha_{\cB_k,\delta} \bar {\sigma}_k$. On the event $\cE_k$, if  $\widehat{\beta}_{k}^{(m)} \neq  \widehat{\beta}^\star_k$, \ie, $|\widehat{\beta}^\star_k-\widehat{\beta}^{(m)}_{\mathrm{ind},k}|> \gamma_m \sqrt{v^{(m)}_k} $, then,  for $k\in \cI^{(m)}\cap  \cI_\eta $,
\begin{align*}
    |\widehat{\beta}^{(m)}_{\mathrm{ind},k}-\beta^{(m)}_k|=|\widehat{\beta}^{(m)}_{\mathrm{ind},k}-\beta^\star_k|\ge &|\widehat{\beta}^{(m)}_{\mathrm{ind},k}-\widehat{\beta}^\star_k|-|\widehat{\beta}^\star_k-\beta^\star_k|
    >\gamma_m \sigma \sqrt{v^{(m)}_k}-1.25C_{0.45}\alpha_{\cB_k,\delta} \bar {\sigma}_k\\
    \geq &(\gamma_m-\sigma) \sqrt{v^{(m)}_k}.
\end{align*}
Let $\zeta_k^{(m)}  \triangleq (\gamma_m-\sigma) \sqrt{v^{(m)}_k}$
and
\begin{equation*}
    \cF_k^{(m)}\triangleq\left\{|\widehat{\beta}^{(m)}_{\mathrm{ind},k}-\beta^{(m)}_k| \le \zeta_k^{(m)}\sigma\sqrt{v^{(m)}_k}\right\}.
\end{equation*}
The event $\cF_k^{(m)} \cap \cE_k$ implies that $\widehat{\beta}_{k}^{(m)} = \widehat{\beta}^\star_k$ for $k\in \cI_\eta \cap \cI^{(m)}$.
Since $\widehat{\beta}^{(m)}_{\mathrm{ind},k}-\beta^{(m)}_k \mid v_k^{(m)}\sim \subG(v^{(m)}_k\sigma^2)$, we have $\PP((\cF_k^{(m)})^c\mid v_k^{(m)})\le O((n_M/n_{[M]})\wedge (s/d))$.  
Thus,
with probability at least $\PP(\cE_k\cap \cF_k^{(m)})\ge 1- O((n_M/n_{[M]})\wedge (s/d))$, it holds that
\begin{align*}
    |\widehat{\beta}^{(m)}_{k}-\beta^{(m)}_k|=|\widehat{\beta}^\star_k-\beta^{(m)}_k|&\le 1.25C_{0.45}\alpha_{\cB_k,\delta} \bar {\sigma}_k.
\end{align*}
Furthermore, using \eqref{eqn:nbiwfgew} and Lemma \ref{lem:small-set-ingeral} and recalling Definition \ref{def:small-set-ingeral}, we have that for any $k\in \cI_\eta\cap \cI^{(m)}$,
\begin{align}
    \EE[|\widehat{\beta}^{(m)}_{k}-\beta^{(m)}_k|\mid  \{\vX^{(m)}\}_{m=1}^M]&\le \widetilde{O}\left(\alpha_{\cB_k,\delta} \bar {\sigma}_k\right)+\EE^{O((n_m/n_{[M]})\vee(s/d) )}\left[|\widehat{\beta}^{(m)}_{\mathrm{ind},k}-\beta^{(m)}_k|+ \gamma_m \sqrt{v^{(m)}_k}\right]\nonumber\\
    &= \widetilde{O}\left(\alpha_{\cB_k,\delta} \bar {\sigma}_k\right)+\widetilde{O}\left( (n_M/n_{[M]})\vee (s/d) )\sigma \sqrt{v_k^{(m)}}\right).\label{eqn:vningcxew1}
\end{align}
Now,
\begin{equation}\label{eqn:vncixvnxf1}
    \bar {\sigma}_k\triangleq\sum_{m\in[M]}w_m\sqrt{v^{(m)}_k}\sigma/W_{[M]}=\widetilde O_P\left(\sum_{m\in[M]}\sqrt{n_m}\sigma/n_{[M]}\right)\overset{a}{\leq} \widetilde 
 O_P\left(\sigma/\sqrt{n_1}\right)
\end{equation}
and 
\begin{equation}\label{eqn:vncixvnxf2}
    \frac{\sqrt{\sum_{m\in[M]}w_m^2}}{W_M}\bar {\sigma}_k\leq \frac{\sqrt{n_1}}{\sqrt{n_{[M]}}}\widetilde O_P\left(\sum_{m\in[M]}\sqrt{n_m}\sigma/n_{[M]}\right)\overset{b}{\leq} \widetilde O_P({\sigma}/{n_{[M]}}),
\end{equation}
where inequalities $a$ and $b$ are due to Condition \ref{asp:size}.
Thus, \eqref{eqn:vncixvnxf1} and \eqref{eqn:vncixvnxf2} lead to
\begin{equation}\label{eqn:vningcxew2}
    \alpha_{\cB_k,\delta} \bar {\sigma}_k=O_P\left(\frac{W_{\cB_k}\sigma}{W_{[M]}\sqrt{n_1}}+\frac{\sigma}{\sqrt{n_{[M]}}}\right).
\end{equation}
Combining \eqref{eqn:vningcxew1} and \eqref{eqn:vningcxew2}, we reach \eqref{eqn:case3}.

\vspace{2mm}
\noindent
\textbf{Bounding the summed error. }
Combining the cases \eqref{eqn:case1}
and \eqref{eqn:case3},  we obtain 
\begin{align}
    &\EE[\|\widehat{\beta}^{(m)}-\beta^{(m)}\|_1\mid  \{\vX^{(m)}\}_{m=1}^M]\nonumber\\
    =&\sum_{k \in\cI_\eta \cap \cI^{(m)}}\EE[|\widehat{\beta}^{(m)}_{k}-\beta^{(m)}_k|\mid  \{\vX^{(m)}\}_{m=1}^M] + \sum_{k \notin \cI^{(m)}\cap\cI_\eta}\EE[|\widehat{\beta}^{(m)}_{k}-\beta^{(m)}_k|\mid  \{\vX^{(m)}\}_{m=1}^M]\nonumber\\
    \le &\sum_{k \in\cI_\eta \cap \cI^{(m)}}\widetilde O_P\left(\frac{W_{\cB_k}\sigma}{W_{[M]}\sqrt{n_1}}+\frac{\sigma}{\sqrt{n_{[M]}}}+\frac{s\sigma /d}{\sqrt{n_m}}\right)+ \sum_{k \notin \cI^{(m)}\cap\cI_\eta}\widetilde{O}\left( \frac{\sigma}{\sqrt{n_m}}\right)\nonumber\\
    \le &\sum_{k \in\cI_\eta }\widetilde O_P\left(\frac{W_{\cB_k}\sigma}{W_{[M]}\sqrt{n_m}}\right) + \sum_{k \notin \cI_\eta}\widetilde{O}_P\left( \frac{\sigma}{\sqrt{ n_m}}\right)+\widetilde{O}_P\left( \frac{s\sigma}{\sqrt{ n_m}}+\frac{d\sigma}{\sqrt{ n_{[M]}}}\right).\label{eqn:hvisdfgsd}
\end{align}
where the last inequality is due to $|(\cI^{(m)})^c|\le s$ and $n_1\geq n_m$ for any $m\in[M]$.
Using Lemma \ref{lem:f-g} with $a=1$ and $x_k=W_{\cB_k}/W_{[M]}$ for all $k\in[d]$, we have 
\begin{equation}\label{eqn:nvicnvxc}
    \sum_{k\in[d]}\left(\frac{W_{\cB_k}}{W_{[M]}} \one(W_{\cB_k}/W_{[M]} < \eta)+ \one(W_{\cB_k}/W_{[M]} \ge  \eta)\right)\le {\lceil s/\eta \rceil}.
\end{equation}
Plugging  $\eta =1/5=O(1)$ into  \eqref{eqn:nvicnvxc} and  combining  \eqref{eqn:hvisdfgsd}, we have 
\begin{align}\nonumber
    \EE[\|\widehat{\beta}^{(m)}-\beta^{(m)}\|_1\mid  \{\vX^{(m)}\}_{m=1}^M]=\widetilde{O}_P\left( \frac{s\sigma}{\sqrt{ n_m}}+\frac{d\sigma}{\sqrt{ n_{[M]}}}\right).
\end{align}
Using Chebyshev's inequality, we obtain \begin{align}\nonumber
    \|\widehat{\beta}^{(m)}-\beta^{(m)}\|_1=\widetilde{O}_P\left( \frac{s\sigma}{\sqrt{ n_m}}+\frac{d\sigma}{\sqrt{ n_{[M]}}}\right).
\end{align}

Similarly for $p=2$, we can establish 
\begin{align}\nonumber
    \EE[\|\widehat{\beta}^{(m)}-\beta^{(m)}\|_2^2\mid  \{\vX^{(m)}\}_{m=1}^M]=\widetilde{O}_P\left( \frac{s\sigma^2}{ n_m}+\frac{d\sigma^2}{ n_{[M]}}\right).
\end{align}

\subsection{Proof of Theorem \ref{thm:lower1}}\label{app:lower1}
\begin{proof}
As discussed in Section \ref{sec:lb-mul-lr}, we prove Theorem \ref{thm:lower1} by considering two special cases of our sparse heterogeneity model: 
\begin{enumerate}
    \item The \emph{homogeneous} case where $\beta^1=\cdots=\beta^M=\beta^\star \in\RR^d$.
    \item The \emph{$s$-sparse} case where $\beta^\star =0$ and  $\|\beta^{m}\|_0\le s$ for all $m\in[M]$.
\end{enumerate}
We also take $\Sigma^{(m)}=\mu I_d$ for all $m\in[M]$ as the special case to guarantee the lower bound. Therefore, we clearly have 
\begin{equation}\label{eqn:gjojfqowrqw}
    \mathcal{M}\ge\inf_{\widehat{\beta}^\star}
    \sup_{{\beta}^\star\in\RR^d}\EE[\|\widehat{\beta}^\star-{\beta}^\star\|_p^p]\triangleq A
\end{equation}
and 
\begin{equation}\label{eqn:gjoejtomqwetfqw}
    \mathcal{M}
    \ge \inf_{\widehat{\beta}^{(m)}}
    \sup_{\substack{\|\beta^{(m)}\|_0\leq s}}\EE[\|\widehat{\beta}^{(m)}-{\beta}^{(m)}\|_p^p]\triangleq B.
\end{equation}

We will show that $A=\widetilde \Omega(d/( \sum_{m=1}^M n_m/\sigma_m^2)^{p/2})$ and $B=\widetilde \Omega(s\sigma_m^p/(n_m)^{p/2})$. 
Then the conclusion follows from the inequality $A\vee B\ge (A+B)/2=\Omega(A+B)$.

\vspace{2mm}
\noindent \textbf{The case where $p=2$:}
The proof follows the same idea as 
 Example 8.4.5 of \cite{duchi2016introductory} and 
 the argument in \cite{duchi2013distance}.
 These show a lower bound $\Omega(d\sigma^2/\|\vX^\top \vX\|_{\rm op} )$ for linear regression  with a given covariate matrix $\vX$ and i.i.d.~$\cN(0,\sigma^2)$ noises. Here, we sketch the key ideas for the extension to different noise variances.

 Let $\mathcal{V}$ be a packing of $\{-1,1\}^d$ such that $\|v-v^\prime\|\geq d/2$ for distinct elements of $\mathcal{V}$, and $|\mathcal{V}|\geq \exp(d/8)$ as guaranteed by the Gilbert-Varshamov bound \citep[Example 7.5]{duchi2016introductory}. For fixed $\delta >0$, if we set $\beta_v^\star=\delta v$, then we have the 
 packing guarantee for distinct elements $v,v^\prime$ that
\[
\|\beta_v^\star-\beta_{v^\prime}^{\star}\|_1=\delta \sum_{k=1}^d|v_k-v_k^\prime|\geq \delta d/2.
\]
Then we have an upper bound on the Kullback-Leibler divergence of the data distributions associated with $\beta_v^\star$, $\beta_{v^\prime}^\star$, 
a feature vector $x\sim \cN(0, \mu I_d)$, and standard deviation $\sigma$:
\begin{align*}
    D_\mathrm{KL}(\cN(x^\top\beta^\star_{v},\sigma^2)\,\|\,\cN(x^\top \beta^\star_{v^\prime},\sigma^2))=&\frac{1}{2\sigma^2}\EE[\|x^\top(\beta^\star_{v}-\beta^\star_{v^\prime})\|_2^2]=\frac{\mu\|\beta^\star_{v}-\beta^\star_{v^\prime}\|_2^2}{2\sigma^2}.
\end{align*}
Consequently, given the independent observations $\mathbf{Y}_v\triangleq \{\{x_i^{(m)\top}\beta_v^\star+\epsilon_i^{(m)}\}_{i=1}^{n_m}\}_{m=1}^M$ where $\epsilon_i^{(m)}\overset{i.i.d.}{\sim}\cN(0,\sigma_m^2)$ and $\mathbf{Y}_{v^\prime}$, 
we have the Kullback-Leibler divergence for the joint distributions:
\begin{align}
    D_\mathrm{KL}(P(\mathbf{Y}_v, \{\vX^{(m)}\}_{m=1}^M)\,\|\,P(\mathbf{Y}_{v^\prime}, \{\vX^{(m)}\}_{m=1}^M))=\sum_{m=1}^M\frac{\mu n_m}{2\sigma_m^2}\|\beta^\star_{v}-\beta^\star_{v^\prime})\|_2^2
    \leq \sum_{m=1}^M\frac{2\mu n_md\delta^2}{\sigma_m^2},\nonumber
\end{align}
where the last inequality holds because $\|v-v^\prime\|_2^2\leq 4d$.
Now we apply the local Fano method \citep[Proposition 8.43]{duchi2016introductory}, we obtain
\[
\inf_{\widehat{\beta}^\star}
    \sup_{{\beta}^\star\in\RR^d}\EE[\|\widehat{\beta}^\star-{\beta}^\star\|_2^2]\geq \frac{\delta^2 d}{2}\left(1-\frac{I(V;\{\vX^{(m)}\}_{m=1}^M,\mathbf{Y})+\ln(2)}{\ln(|\mathcal{\cV}|)}\right),
\]
where $I(V;\{\vX^{(m)}\}_{m=1}^M,\mathbf{Y})$ is the mutual information between the index variable $V\sim {\rm Unif}(\mathcal{\cV})$ and the responses $ \mathbf{Y}=\{Y^{(m)}\}_{m\in[M]}$.
By choosing $\delta =8^{-1}\left(\sum_{m=1}^M\mu n_m/\sigma_m^2\right)^{-1/2}$, following  \citep[Example 7.5]{duchi2016introductory},  we find that $1-(I(V;\mathbf{Y})+\ln(2))/\ln(|\mathcal{\cV}|)\geq 1/2$. 
This leads to 
\[
A=\inf_{\widehat{\beta}^\star}
    \sup_{{\beta}^\star\in\RR^d}\EE[\|\widehat{\beta}^\star-{\beta}^\star\|_2^2]=\Omega\left(d\left(\sum_{m=1}^M\mu n_{m}/\sigma_m^2\right)^{-1}\right).
\]

By applying the same argument to a single task with $s$-dimensional parameters, one can readily show $B=\Omega(s\sigma_m^p/n_m^{p/2})$.

\vspace{2mm}
\noindent\textbf{The case where $p=1$:} The proof for the $\ell_1$ case follows the same workflow by utilizing $\|\beta_v^\star-\beta_{v^\prime}^{\star}\|_1=\delta \sum_{k=1}^d|v_k-v_k^\prime|\geq \delta d/2$. 
We then obtain $A=\Omega(d/(\sum_{m\in [M]} n_m/\sigma_m^2)^{1/2})$ and thus $B=\Omega(s\sigma_m/ n_m^{1/2})$ accordingly.

\end{proof}

\section{Results on Linear Regression with  General Noise}\label{app:general-noise}

Lemma \ref{lem:median-Gaussian}, the cornerstone of our previous analysis, relies on the Gaussianity of the noise $\{\ep_i^{(m)}\}_{i=1}^{n_m}$ for each $m\in[M]$. 
In this section, we argue that the estimation error of $\widehat{\beta}^\star$ holds up to a smaller-order term, 
even if we relax Gaussianity to bounded skewness and Orlicz norm. 
While this holds for unbalanced datasets, 
for simplicity, we only state the results for balanced datasets, \ie, $n_m=n$ and $\sigma_m=\sigma$ for all $m\in[M]$. 
Recall that for a random variable $Z$,
and for $\alpha\ge 0$,
the $\alpha$-order Orlicz norm of $\ep_i^{(m)}$, $1\le i\le n$, is 
\begin{equation*}
        \|Z\|_{\Psi_\alpha}:=\inf\left\{c>0:\EE\left[\Psi_\alpha\left({|Z|}/{c}\right)\right]\le 1\right\},
   \end{equation*}
where $\Psi_\alpha(u)=e^{u^\alpha} -1$ for any $u\ge 0$.
We make the following technical assumption,  replacing Condition \ref{asp:noise}.
\begin{assumption}[\sc Bounded covariates, skewness, and Orlicz norm]\label{asp:combo}
There are constants $c_{\rm sk},L\ge \mu \ge  0$, $\alpha\in(0,2]$, and $\sigma\ge 0$ such that for each $m\in[M]$,
\begin{enumerate}
    \item the covariates satisfy $\mu n I_d\preceq \vX^{(m)\top}\vX^{(m)}\preceq n L I_d$  and $\|x_i^{(m)}\|_2^2 \le d L$ for all $1\le i\le n$;
    \item the noise variables $\ep_i^{(m)}$, $1\le i\le n$, are 
     i.i.d.~with  zero mean, 
    satisfying  $\EE[|\ep_i^{(m)}|^3]\le c_{\rm sk}(\EE[|\ep_i^{(m)}|^2])^{3/2}$;
    \item the $\alpha$-order Orlicz norm of $\ep_i^{(m)}$, $1\le t\le n$, is bounded as 
    $\|\ep_i^{(m)}\|_{\Psi_\alpha}\le \sigma$.
\end{enumerate}
\end{assumption}

\begin{remark}
    In Condition \ref{asp:combo}, 
    the first condition strengthens Condition \ref{asp:design-mats}, and is used to control the coefficients involved in a normal approximation via the Berry-Esseen theorem \citep{berry1941accuracy}. 
    This condition holds with a high probability for randomly sampled design matrices with i.i.d.~features having a well-conditioned expected covariance matrix   \citep{vershynin2018high},  including the stochastic contexts in bandits that we will use. 
    The second condition, while less standard in the literature, is also needed in the Berry-Esseen theorem. 
    The third condition generalizes the widely used sub-Gaussianity condition \citep{Tripuraneni2021ProvableMO,Han2020SequentialBL,xu2021multitask,Ren2020DynamicBL}. 
    This allows for many choices of random noise, such as
bounded noises (for any $\alpha>0$), sub-Gaussian noises
(for $\alpha=2$), or sub-exponential noises (for $\alpha=1$) like Poisson random variables,
or noise with heavier tails such as Weibull random variables with shape parameter $\alpha \in(0,1]$. 
A smaller $\alpha$ allows for a heavier tail. 
\end{remark}

Under Condition \ref{asp:combo}, we show that the c.d.f.~of each normalized OLS estimate is approximately Gaussian up to a vanishing factor of $\sqrt{d/n}$. 
The key idea is to use Gaussian approximation via the Berry-Esseen theorem (Lemma \ref{lem:berry}).
Based on this, we control the error of the global estimate $\widehat{\beta}^\star$ in Proposition \ref{prop:gau-approx}. 
Compared to Lemma \ref{lem:col-est} obtained under Gaussian noise, the upper bound incurs
an additional vanishing term of order $O(\sqrt{d}/n)$.
The proof is in Appendix \ref{sec:gau-approx}.

\begin{proposition}\label{prop:gau-approx}
Under Condition \ref{asp:combo}, for any $k\in[d]$ and $m\in[M]$,  the c.d.f.~$F_k^{(m)}$ of  $\smash{\widehat{\beta}^{(m)}_{{\rm ind},k}} $ (after standardization) satisfies
    \begin{equation*}
        \sup_{z\in \RR}|F_k^{(m)}(z)-\Phi(z)|\le \frac{0.6c_{\rm sk}L\sqrt{d}}{\mu \sqrt{n}}.
    \end{equation*}
Further,  if  $n\ge 36 d (c_{\rm sk}L/\mu)^2$,
for any $0<\eta \le \frac{1}{10}$ and $k\in\cI_\eta$,  it holds for any $0\le \delta \le M/20$ that
\begin{equation*}
    \PP\left(|\widehat{\beta}^\star_k-\beta^\star_k|\ge 1.25 C_{0.45} \bar{\sigma}_k \alpha^\prime_{\cB_k,\delta} \right)\le 4e^{-2\delta},
\end{equation*}
where $C_{\rm be}$ is an abolute constant, $\bar{\sigma}_k=\sigma_{\rm v} \sum_{m=1}^M\sqrt{v^{(m)}_k}/M$  for each $k\in[d]$ with $\sigma_{\rm v}^2\triangleq {\rm Var}(\ep_i^{(m)})$, and $\alpha^\prime_{\cB_k,\delta}\triangleq {|\cB_k|}/{M}+\sqrt{{1.01\delta}/{(M-|\cB_k|)}}+{0.6 c_{\rm sk}L\sqrt{d}}/(\mu \sqrt{n})$. 
\end{proposition}

The Berry-Esseen theorem helps control the estimation error of the median-based estimator $\widehat{\beta}^\star$ with high probability. 
To  
control the in-expectation estimation error of the final estimates $\{\widehat{\beta}_{\rm ft}^{(m)}\}_{m\in[M]}$, 
we also need the concentration of individual OLS estimates. 
This is guaranteed by the generalized Hanson-Wright inequality (Lemma \ref{lem:hw-ineq}), which characterizes the tail of quadratic forms of random variables with bounded Orlicz norm. 
Combining these two ingredients, we can bound the estimation errors of $\{\widehat{\beta}_{\rm ft}^{(m)}\}_{m\in[M]}$ in Theorem \ref{thm:approx-upper};  
with a proof in Appendix \ref{app:approx-upper}.

\begin{theorem}\label{thm:approx-upper}
Under Condition \ref{asp:heterogneity} and \ref{asp:combo}, 
for any $p\in\{1,2\}$, $m\in[M]$, with $\widehat{\beta}^{(m)}_{\mathrm{MOLAR}}$  from Algorithm \ref{alg:molar}, it holds  that
\begin{equation}\label{eqn:vmocxmvxcv}
    \|\widehat{\beta}^{(m)}_{\mathrm{MOLAR}}-\beta^{(m)}\|_p^p=\widetilde{O}_P\left(\frac{\sigma^p}{n^{p/2}}\left(s+\frac{d}{M^{p/2}}+\frac{d^{1+p/2}}{  n^{p/2}}\right)\right),
\end{equation}
where $\widetilde{O}_P(\cdot)$ absorbs 
a ${\sf Polylog}(M, \mu,L)$ factor with degree and coefficients depending  only on $\alpha$. 
\end{theorem}

Consequently, if $\alpha = \Theta(1)$ and $n=\Omega(d(M\vee (d/s)^{2/p}))$, \eqref{eqn:vmocxmvxcv} matches \eqref{eqn:vjsiodfdvcx} from Theorem \ref{thm:ls-fixed-design}.

\subsection{Preliminaries}
To establish the proofs for Proposition \ref{prop:gau-approx} and Theorem \ref{thm:approx-upper}, we first prove some technical lemmas. Below, we view $\alpha$ as an absolute constant. Let $\sigma_{\rm v}^2$ be the variance of noise $\ep_i^{(m)}$. Under the assumtption that $\|\ep_i^{(m)}\|_{\Psi_\alpha}\le \sigma$, we first establish the following results.
\begin{lemma}[\sc Relation between $\sigma_{\rm v}$ and $\sigma$]\label{lem:variance}
    It holds that $\sigma_{\rm v}^2\le 2\Gamma(1+2/\alpha)\sigma^2=O(\sigma^2)$ where $\Gamma(\cdot)$ is the gamma function.
\end{lemma}
\begin{proof}
    By Lemma \ref{lem:integral}, we have $\sigma_{\rm v}^2=\EE[(\ep_i^{(m)})^2]=\int_{0}^{\infty}\PP((\ep_i^{(m)})^2\ge t) \d t=\int_{0}^{\infty}\PP(|\ep_i^{(m)}|\ge \sqrt{t}) \d t$. Using Lemma \ref{lem:orlicz}, we have $\PP(|\ep_i^{(m)}|\ge \sqrt{t}) \le 2\exp(-(\sqrt{t}/\sigma)^\alpha)$. 
    Therefore, by the change of variables $u= \sqrt{t}/\sigma$, we have 
    \begin{align*}
        \sigma_{\rm v}^2=&\int_{0}^{\infty}\PP(|\ep_i^{(m)}|\ge \sqrt{t}) \d t\le 2\int_0^\infty\exp(-(\sqrt{t}/\sigma)^\alpha)\d t\\
        =& 4\sigma^2\int_0^\infty\exp(-u^\alpha)u\d u=2\Gamma(1+2/\alpha)\sigma^2.
    \end{align*}
\end{proof}

\begin{lemma}[\sc Tail and integral of OLS estimate] \label{lem:tail-gen}
    For each $m\in[M]$, $k\in[d]$, and $u>0$, it holds that 
    \begin{align}\label{eqn:vndivnvdsdcx1}
    \PP\left(|\widehat{\beta}_{\mathrm{ind},k}^{(m)}-\beta^{(m)}_k|\ge (\sigma_{\rm v}+u\sigma)\sqrt{v^{(m)}_k}\right)\le 2\exp\left(-\frac{u^{\alpha}}{c_{\rm hw}}\right).
\end{align}
where $c_{\rm hw}$ is an absolute constant in Lemma \ref{lem:hw-ineq}.
Furthermore, we have for any $p\in\{1,2\}$ and $\delta \in[0,1]$ with $\ln(2/\delta)\ge 1/\alpha$ that
\begin{equation}\label{eqn:vndivnvdsdcx2}
    \EE^\delta[|\widehat{\beta}_{\mathrm{ind},k}^{(m)}-\beta^{(m)}_k|^p ]=O\left(\delta\sigma^p\ln(2/\delta)^{p/\alpha}\sqrt{v^{(m)}_k}\right).
\end{equation}
\end{lemma}
\begin{proof}
    Recall from \eqref{eqn:vjofnbvbcx} that $\widehat{\beta}_{\mathrm{ind},k}^{(m)}-\beta^{(m)}_k=\sum_{i=1}^n \langle w^{(m)}_k, x_i^{(m)}\rangle \ep_i^{(m)}$ with $w_k^{(m)\top}\triangleq (\vX^{(m)\top} \vX^{(m)})^{-1}_{k,\cdot} $  and $\sum_{i=1}^n \langle w^{(m)}_k, x_i^{(m)}\rangle^2=v^{(m)}_k$. By Lemma \ref{lem:hw-ineq}, we have for any $t\ge \sigma^2 v^{(m)}_k$ that
\begin{equation*}
    \PP\left(\left|\left(\widehat{\beta}_{\mathrm{ind},k}^{(m)}-\beta^{(m)}_k \right)^2-\sigma_{\rm v}^2v^{(m)}_k\right|\ge t\right)\le 2\exp\left(-\frac{1}{c_{\rm hw}}\left(\frac{t}{\sigma^2 v^{(m)}_k}\right)^{\alpha /2}\right).
\end{equation*}
Letting $t= u^2\sigma^2 v^{(m)}_k$ and using $\sigma_{\rm v}+u\sigma \ge \sqrt{\sigma_{\rm v}^2+u^2 \sigma^2}$ for any $u \ge 0$, we have 
\begin{align*}
    &\PP\left(|\widehat{\beta}_{\mathrm{ind},k}^{(m)}-\beta^{(m)}_k|\ge (\sigma_{\rm v}+u\sigma)\sqrt{v^{(m)}_k} \right)
    \le \PP\left(|\widehat{\beta}_{\mathrm{ind},k}^{(m)}-\beta^{(m)}_k|\ge \sqrt{\sigma_{\rm v}^2+u^2 \sigma^2}\sqrt{v^{(m)}_k} \right)\\
    \le &\PP\left(\left|\left(\widehat{\beta}_{\mathrm{ind},k}^{(m)}-\beta^{(m)}_k\right)^2-\sigma_{\rm v}^2v^{(m)}_k\right|\ge u^2\sigma^2 v^{(m)}_k \right)
    \le 2\exp\left(-\frac{u^{\alpha}}{c_{\rm hw}}\right).
\end{align*}
We thus obtain \eqref{eqn:vndivnvdsdcx1}. For \eqref{eqn:vndivnvdsdcx2}, 
we only analyze the case $p=1$ and the case $p=2$ follows. 
To this end, we follow the proof of Lemma \ref{lem:small-set-ingeral}. By using the smoothing technique in  Lemma \ref{lem:small-set-ingeral}, it suffices to consider noise to be continuous, 
in which case so is $\widehat{\beta}_{\mathrm{ind},k}^{(m)}-\beta^{(m)}_k$. 
Let $Z\triangleq  \widehat{\beta}_{\mathrm{ind},k}^{(m)}-\beta^{(m)}_k$.
From \eqref{is},
$\EE^\delta[|Z|]$ is given by the integral over the upper $\delta$-level set $A_\delta=\{|Z|\ge q_\delta\}$ with $\PP(|Z(\omega)|\ge q_\delta)=\delta.$

From \eqref{eqn:vndivnvdsdcx1}, we readily find that 
\begin{equation}\label{eqn:vbdsuvbuxzbcvxz}
    q_\delta \le (\sigma_{\rm v}+c_{\rm hw}^{1/\alpha}\ln(2/\delta)^{1/\alpha}\sigma)\sqrt{v^{(m)}_k}.
\end{equation}
Plugging \eqref{eqn:vbdsuvbuxzbcvxz} and \eqref{eqn:vndivnvdsdcx1} into Lemma \ref{lem:integral}, 
and using the change of variables $t=(\sigma_{\rm v}+u c_{\rm hw}^{1/\alpha} \sigma)\sqrt{v^{(m)}_k}$ and $q_\delta= (\sigma_{\rm v}+u_0 c_{\rm hw}^{1/\alpha} \sigma)\sqrt{v^{(m)}_k}$ with $u_0\le\ln(2/\delta)^{1/\alpha}$, we obtain
\begin{align}
    \EE^\delta[|Z| ]&=\EE[|Z|\mathds{1}(|Z|\ge q_\delta) ]=q_\delta \PP(|Z|\ge q_\delta )+\int_{q_\delta}^\infty \PP(|Z|\ge t )\d t\nonumber\\
    &\le \delta q_\delta+c_{\rm hw}^{1/\alpha} \sigma\sqrt{v^{(m)}_k}\int_{u_0}^\infty\min\left\{\delta, 2\exp\left(-{u^{\alpha}}\right)\right\}\d u.\label{eqn:vbjiefa-vcv}
\end{align}
where the last inequality follows from \eqref{eqn:vndivnvdsdcx1} and from 
$\PP(|Z|\ge t )\le \PP(|Z|\ge q_\delta )= \delta$.
Now, we calculate that
\begin{align}
    &\int_{u_0}^\infty\min\left\{\delta, 2\exp\left(-{u^{\alpha}}\right)\right\}\d u
    =\delta \left(\ln(2/\delta)^{1/\alpha}-u_0\right)+2\int_{\ln(2/\delta)^{1/\alpha}}^\infty \exp\left(-{u^{\alpha}}\right)\d u \label{eqn:ZZvnicxnbic}
\end{align}
Note that
\begin{align}
    \int_{\ln(2/\delta)^{1/\alpha}}^\infty e^{-u^{\alpha}}\d u=&\frac{1}{\alpha}\Gamma(1/\alpha, \ln(2/\delta))\le \frac{\delta}{2\alpha}\ln(2/\delta)^{1/\alpha},\label{eqn:vodnvxcvcx}
\end{align}
where the inequality is due to $\Gamma(a, b)\le b^ae^{-b}/(b+1-a)\le b^ae^{-b}$ for any $b\ge a >0$. Combining \eqref{eqn:ZZvnicxnbic} and \eqref{eqn:vodnvxcvcx} with \eqref{eqn:vbjiefa-vcv}, we therefore obtain  
\begin{align}
    \EE^\delta[|Z| ]
    &\le \delta q_\delta+c_{\rm hw}^{1/\alpha} \sigma\sqrt{v^{(m)}_k}\left(\delta \left(\ln(2/\delta)^{1/\alpha}-u_0\right)+\frac{\delta}{2\alpha}\ln(2/\delta)^{1/\alpha}\right)\nonumber\\
    &=O\left(\delta (\sigma_{\rm v}+\sigma\ln(2/\delta)^{1/\alpha})\sqrt{v^{(m)}_k}\right)=O\left(\delta \sigma\ln(2/\delta)^{1/\alpha}\sqrt{v^{(m)}_k}\right).\nonumber
\end{align}
\end{proof}
Without loss of generality, we consider $\sigma_{\rm v}>0$ in the proofs of this section. Otherwise, the setup becomes noiseless, and individual OLS estimates recover the parameters directly.
\subsection{Proof of Proposition \ref{prop:gau-approx}}\label{sec:gau-approx}
\begin{proof}
We have that
$\widehat{\beta}_{\mathrm{ind}}^{(m)}=(\vX^{(m)\top} \vX^{(m)})^{-1}\vX^{(m)\top} Y^{(m)}= {\beta}^{(m)}+(\vX^{(m)\top} \vX^{(m)})^{-1}\vX^{(m)\top} \boldsymbol{\ep}^{(m)}$ where $\boldsymbol{\ep}^{(m)}=(\ep_1^{(m)},\dots,\ep_n^{(m)})^\top$ is the noise vector. 
Using Condition \ref{asp:combo}, we have, for each covariate $k\in[d]$, 
\begin{equation}\label{eqn:vnidsfsd1}
{\rm Var}(\widehat{\beta}_{\mathrm{ind},k}^{(m)})= \sigma_{\rm v}^2 \|(\vX^{(m)\top} \vX^{(m)})^{-1}_{k,\cdot}\vX^{(m)\top} \|_2^2=\sigma_{\rm v}^2v^{(m)}_k.
\end{equation}
This also implies
\begin{equation}\label{eqn:vjofnbvbcx}\widehat{\beta}_{\mathrm{ind},k}^{(m)}=\beta^{(m)}_k+\sum_{t=1}^n \langle w^{(m)}_k, x_i^{(m)}\rangle \ep_i^{(m)}
\end{equation}
where $w_k^{(m)}\triangleq (\vX^{(m)\top} \vX^{(m)})^{-\top}_{k,\cdot} $  and $w_k^{(m)\top}\vX^{(m)\top} \vX^{(m)} w^{(m)}_k=v^{(m)}_k$. Using Condition \ref{asp:combo}, we have $(\mu n)^{-1}\ge v^{(m)}_k= w_k^{(m)\top}\vX^{(m)\top} \vX^{(m)} w^{(m)}_k\ge \mu n \|w^{(m)}_k\|_2^2$
and thus $\|w^{(m)}_k\|_2\le (\mu n)^{-1}$. Thus, by further using Condition \ref{asp:combo}, we obtain
\begin{align}
    &\EE[|\langle w^{(m)}_k, x_i^{(m)}\rangle \ep_i^{(m)}|^3]=|\langle w^{(m)}_k, x_i^{(m)}\rangle|^3\EE[|\ep_i^{(m)}|^3]\nonumber\\
    \le& \|w^{(m)}_k\|_2 \|x_i^{(m)}\|_2 |\langle w^{(m)}_k, x_i^{(m)}\rangle|^2 c_{\rm sk}\sigma_{\rm v}^3\le  \sqrt{d L }c_{\rm sk}\langle w^{(m)}_k, x_i^{(m)}\rangle^2\sigma_{\rm v}^3/(\mu n).\label{eqn:nmvocdnmbvx}
\end{align}
Summing up \eqref{eqn:nmvocdnmbvx} with respect all $t\in[n]$, we find
\begin{equation}\label{eqn:vnidsfsd2}
    \sum_{i=1}^n \EE[|\langle w^{(m)}_k, x_i^{(m)}\rangle \ep_i^{(m)}|^3]\le \sqrt{d L }c_{\rm sk}\sigma_{\rm v}^3/(\mu n)\sum_{i=1}^n\langle w^{(m)}_k, x_i^{(m)}\rangle^2=\sqrt{d L }c_{\rm sk}\sigma_{\rm v}^3/(\mu n) v^{(m)}_k.
\end{equation}
Therefore, plugging the bounds \eqref{eqn:vnidsfsd1} and \eqref{eqn:vnidsfsd2} into Lemma \ref{lem:berry}, we find 
\begin{align*}
    \sup_{z\in \RR}|F_k^{(m)}(z)-\Phi(z)|\le &C_{\rm be}\frac{\sum_{t=1}^n \EE[|\langle w^{(m)}_k, x_t^{(m)}\rangle \ep_i^{(m)}|^3]}{{\rm Var}(\widehat{\beta}_{\mathrm{ind},k}^{(m)})^{3/2}}
    \le  0.6\frac{\sqrt{d L }c_{\rm sk}v^{(m)}_k\sigma_{\rm v}^3/(\mu n)}{\sigma_{\rm v}^2v^{(m)}_k \sigma_{\rm v}/\sqrt{Ln}}=\frac{0.6 c_{\rm sk}L\sqrt{d}}{\mu \sqrt{n}}.
\end{align*}

Next, we use this to control 
the estimation error of $\widehat{\beta}^\star$. 
The analysis is similar to the one of Lemma \ref{lem:median-Gaussian}.
For each $k\in[d]$ and $z\in\RR$, 
let $\widehat{F}_{\cG_k}(z):=\frac{1}{|\cG_k|}\sum_{m\in\cG_k}\mathds{1}(\widehat{\beta}^{(m)}_{{\rm ind},k}\le z)$ and $\widehat{F}_{[M]}(z):=\frac{1}{M}\sum_{m\in[M]}\mathds{1}(\widehat{\beta}^{(m)}_{{\rm ind},k}\le z)$ be the empirical distribution of $\{\widehat{\beta}^{(m)}_{{\rm ind},k}\}_{m\in\cG_k}$ with $\cG_k\triangleq [M]\backslash \cB_k$ and $\{\widehat{\beta}_{{\rm ind},k}^{(m)}\}_{m\in[M]}$, respectively. 
Since, for any $m\in [M]\backslash\cG_k$, $\widehat{\beta}_{{\rm ind},k}^{(m)}$ has mean $\beta^\star_k$ and variance $\sigma_{\rm v}^2v^{(m)}_k$,
we have 
\begin{align*}
    \EE[\widehat{F}_{\cG_k}(z)]=&\frac{1}{|\cG_k|}\sum_{m\in \cG_k}\PP(\widehat{\beta}_{{\rm ind},k}^{(m)} \le z)
    =\frac{1}{|\cG_k|}\sum_{m\in[M]\backslash \cB_k}F_k^{(m)}\left( \frac{z-\beta_k^\star}{\sigma_{\rm v}\sqrt{v^{(m)}_k}}\right).
\end{align*}
Let $z_1$  be the value such that 
\begin{equation*}
    \frac{1}{|\cG_k|}\sum_{m\in\cG_k}\Phi\left( \frac{z_1-\beta_k^\star}{\sigma_{\rm v}\sqrt{v^{(m)}_k}}\right)= \frac{1}{2}+\alpha^\prime_{\cB_k,\delta},
\end{equation*}
where $\alpha^\prime_{\cB_k,\delta}\triangleq {|\cB_k|}/{M}+\sqrt{{1.01\delta}/{(M-|\cB_k|)}}+{0.6 c_{\rm sk}L\sqrt{d}}/(\mu \sqrt{n})$.
By Hoeffding’s inequality,  for any $\delta\ge 0$ and $z\in\RR$, we have with probability at least $1-2e^{-2\delta}$ that
\begin{equation*}
    \left|\widehat{F}_{\cG_k}(z)-\frac{1}{|\cG_k|}\sum_{m\in \cG_k}F_k^{(m)}\left( \frac{z-\beta^\star_k}{\sigma\sqrt{v^{(m)}_k}}\right)\right|\le \sqrt{\frac{\delta}{M-|\cB_k|}}.
\end{equation*}
Similar to \eqref{eqn:iogdgsd}, it is clear that  $|\widehat{F}_{\cG_k}(z)-\widehat{F}_{[M]}(z)|\le {|\cB_k|}/{M}$ for all $z\in\RR$.
Combing the above properties and using a union bound, we have with probability at least $1-4e^{-2\delta}$ that
\begin{align*}
    \widehat{F}_{[M]}(z_1)\ge &\widehat{F}_{[M]}(z_1)- \widehat{F}_{\cG_k}(z_1)+\widehat{F}_{\cG_k}(z_1)-\EE[\widehat{F}_{\cG_k}(z_1)]\\
    & +\frac{1}{|\cG_k|}\sum_{m\in \cG_k}\left(F_k^{(m)}\left(\frac{z_1-\beta_k^\star}{\sigma_{\rm v}\sqrt{v^{(m)}_k}}\right)-\Phi\left( \frac{z_1-\beta_k^\star}{\sigma_{\rm v}\sqrt{v^{(m)}_k}}\right)\right)
     +\frac{1}{|\cG_k|}\sum_{m\in\cG_k}\Phi\left( \frac{z_1-\beta_k^\star}{\sigma_{\rm v}\sqrt{v^{(m)}_k}}\right)\\
    \ge  & - \frac{|\cB_k|}{M}-\sqrt{\frac{\delta}{M-|\cB_k|}}-\frac{0.6c_{\rm sk}L\sqrt{d}}{\mu \sqrt{n}}+ \frac{1}{2}+\alpha_{\cB_k,\delta}^\prime>\frac{1}{2}.
\end{align*}
This implies $ \widehat{\beta}^\star_k=\med(\{\widehat{\beta}^{(m)}_{{\rm ind},k}\}_{m\in [M]})\le z_1$.
By a similar argument to the proof of Lemmas \ref{lem:median-Gaussian} and \ref{lem:col-est}, one can verify $\alpha^\prime_{\cB_k,\delta}<0.45$ and thus upper bound $z_1$ as $\widehat{\beta}^\star_k\le z_1\le \beta_k^\star+\tilde{\sigma}_kG_{[M],\cB_k,\delta,n}^\prime C_{\ep}$.
Similarly, it also holds that $\widehat{\beta}^\star_k\ge \mu-1.25 C_{0.45} \bar{\sigma}_k \alpha^\prime_{\cB_k,\delta}$, finishing the proof.

\end{proof}

\subsection{Proof of Theorem \ref{thm:approx-upper}}\label{app:approx-upper}

\begin{proof}
We provide the proof  for $p=1$; the case $p=2$ follows
similarly.  
Let $\cI^{(m)}=\{k\in[d]:\beta^{(m)}_k=\beta^\star_k\}$ for each $m\in[M]$. We denote $\widehat{\beta}^{(m)}_{\rm MOLAR}$ as $\widehat{\beta}^{(m)}$ below for simplicity.
For each $m\in[M]$, we bound $\EE[[\widehat{\beta}^{(m)}_{k}-\beta^{(m)}_k|]$ for $k\in[d] $ in two cases. 

\vspace{2mm}
\noindent\textbf{Case 1. } 
For any $k \in[d]$ (in particular, for $k\notin \cI_\eta$ or $k\notin \cI^{(m)}$), 
since $\gamma_m=\widetilde{O}(\sigma)$,
following the argument for Case 1 of Theorem \ref{thm:ls-fixed-design} and using Lemma \ref{lem:variance},
we readily obtain 
\begin{equation}\label{eqn:case1-gn}
    \EE[|\widehat{\beta}^{(m)}_{k}-\beta^{(m)}_k|]={O}\left(\sigma_{\rm v}\sqrt{v^{(m)}_k}\right)+\widetilde{O}\left(\sigma\sqrt{v^{(m)}_k}\right)=\widetilde{O}(\sigma/\sqrt{ n}),
\end{equation} 
where the last inequality is due to Lemma \ref{lem:variance}.
If  $n<  100 d (C_{\rm be}c_{\rm sk}L/\mu)^2$ or $M\le 20 \ln(M)\vee(\alpha^{-1} +
\ln(3)) $, by summing up \eqref{eqn:case1-gn} with respect to all $k\in[d]$, we obtain 
\begin{equation*}
    \EE[\|\widehat{\beta}^{(m)}_{k}-\beta^{(m)}\|_1]=\widetilde{O}(d\sigma/\sqrt{n}).
\end{equation*}
This yields the conclusion in \eqref{eqn:vmocxmvxcv}. Therefore, we next assume $n\ge  100 d (C_{\rm be}c_{\rm sk}L/\mu)^2$ and $M\ge 20 \ln(M)\vee \alpha^{-1}$.

\vspace{2mm}
\noindent\textbf{Case 2. }
When $k\in \cI^{(m)} \cap \cI_\eta$, we will show
\begin{equation}\label{eqn:case2-gn}
    \EE[|\widehat{\beta}^{(m)}_{k}-\beta^{(m)}_k|]=\widetilde{O}\left(\frac{\sigma}{\sqrt{\mu n}}\left(\frac{|\cB_k|}{M} + \frac{1}{\sqrt{M}}+\frac{\sqrt{d} }{\sqrt{n}}\right)\right).
\end{equation}
Let $\delta = \ln(M)\vee(\alpha^{-1} +
\ln(3))=\widetilde{O}(1)$. If $1.25 C_{0.45}\tilde{\sigma}_k\alpha^\prime_{\cB_k,\delta}\ge \sigma_{\rm v}\sqrt{v^{(m)}_k}$, then from \eqref{eqn:case1-gn}, we directly conclude \eqref{eqn:case2-gn}.
Otherwise suppose $1.25 C_{0.45}\tilde{\sigma}_k\alpha^\prime_{\cB_k,\delta}\le \sigma_{\rm v}\sqrt{v^{(m)}_k}$.
Define the event $\cE_k=\{|\widehat{\beta}^\star_k-\beta^\star_k|\le 1.25 C_{0.45}\tilde{\sigma}_k\alpha^\prime_{\cB_k,\delta}\}$.
Since $M\ge 20 \delta $, using Proposition \ref{prop:gau-approx}, we have $\PP((\cE_k)^c)\le 4e^{-\delta} $. 
Furthermore, by the condition $G^\prime_{[M],\cB_k,\delta,n}\tilde{\sigma}_k\le \sigma_{\rm v}\sqrt{v^{(m)}_k}$, we have that the event $\cE_k$ implies $|\widehat{\beta}^\star_k-\beta^\star_k|\le \sigma_{\rm v} \sqrt{v^{(m)}_k}$. 
On the event $\cE_k$, if $|\widehat{\beta}^\star_k-\widehat{\beta}^{(m)}_{\mathrm{ind},k}|> \gamma_m \sqrt{v^{(m)}_k} $, then,  for $k\in \cI_\eta \cap \cI^{(m)}$
\begin{align*}
    |\widehat{\beta}^{(m)}_{\mathrm{ind},k}-\beta^{(m)}_k|=|\widehat{\beta}^{(m)}_{\mathrm{ind},k}-\beta^\star_k|\ge &|\widehat{\beta}^{(m)}_{\mathrm{ind},k}-\widehat{\beta}^\star_k|-|\widehat{\beta}^\star_k-\beta^\star_k|
    > (\gamma_m -\sigma_{\rm v}) \sqrt{v^{(m)}_k}.
\end{align*}
Let 
\begin{equation*}
    \cF_k^{(m)}\triangleq\left\{|\widehat{\beta}^{(m)}_{\mathrm{ind},k}-\beta^{(m)}_k| \le (\gamma_m -\sigma_{\rm v}) \sqrt{v^{(m)}_k}\right\}.
\end{equation*}
Since $\gamma_m-\sigma_{\rm v}\ge \sigma_{\rm v}+c_{\rm hw}^{1/\alpha}\delta^{1/\alpha} \sigma$, by \eqref{eqn:vndivnvdsdcx1},  we have 
\begin{align*}
    \PP(\cF_k^{(m)})\le &\PP\left(|\widehat{\beta}_{\mathrm{ind},k}^{(m)}-\beta^{(m)}_k|\ge (\sigma_{\rm v}+c_{\rm hw}^{1/\alpha}\delta^{1/\alpha} \sigma)\sqrt{v^{(m)}_k}\right)\le 2e^{-\delta}.
\end{align*}
Since event $\cF_k^{(m)} \cap \cE_k$ implies that $[\widehat{\beta}_{\rm ft}^{(m)}]_k = \widehat{\beta}^\star_k$ for $k\in \cI_\eta \cap \cI^{(m)}$. In other words,
with probability at least $\PP(\cE_k \cap \cF_k^{(m)})\ge 1- 6e^{-\delta}$, it holds that
\begin{align*}
    |\widehat{\beta}^{(m)}_{k}-\beta^{(m)}_k|=|\widehat{\beta}^\star_k-\beta^{(m)}_k|&\le 1.25 C_{0.45}\tilde{\sigma}_k\alpha^\prime_{\cB_k,\delta}.
\end{align*}
Furthermore,  using \eqref{eqn:nbiwfgew} and Lemma \ref{lem:tail-gen} (with $\ln(2/(6e^{-\delta}))\ge 1/\alpha$) and Lemma \ref{lem:variance}, we have that for any $k\in \cI_\eta\cap \cI^{(m)}$,
\begin{align}
    \EE[|\widehat{\beta}^{(m)}_{k}-\beta^{(m)}_k|]&\le \widetilde{O}\left(\tilde{\sigma}_kG^\prime_{[M],\cB_k,\delta,n}\right)+\EE^{6e^{-\delta }}\left[|\widehat{\beta}^{(m)}_{\mathrm{ind},k}-\beta^{(m)}_k|+ \gamma_m \sqrt{v^{(m)}_k}\right]\nonumber\\
    &= \widetilde{O}\left(\frac{\sigma_{\rm v}}{\sqrt{\mu n}}\left(\frac{|\cB_k|}{M} + \frac{1}{\sqrt{M}}+\frac{\sqrt{d} }{\sqrt{n}}\right)\right)+\widetilde{O}\left( \frac{\delta}{\sqrt{n}}\left(\sigma_{\rm v}+\sigma\right)\right)\nonumber\\
    &=\widetilde{O}\left(\frac{\sigma}{\sqrt{\mu n}}\left(\frac{|\cB_k|}{M} + \frac{1}{\sqrt{M}}+\frac{\sqrt{d} }{\sqrt{n}}\right)\right).\nonumber
\end{align}

\vspace{2mm}
\noindent
\textbf{Bounding the summed error. }
Combining the cases \eqref{eqn:case1-gn}, \eqref{eqn:case2-gn}, and using $|(\cI^{(m)})^c|\le s$,  we obtain 
\begin{align}
     &\EE[\|\widehat{\beta}^{(m)}_{\mathrm{MOLAR}}-\beta^{(m)}\|_1]
    \le\frac{ \sigma}{\sqrt{\mu n}}\widetilde{O}\left(\sum_{k \in\cI_\eta \cap \cI^{(m)}}\left(\frac{|\cB_k|}{M}+\frac{1}{\sqrt{M}}+\frac{\sqrt{d} }{\sqrt{n}}\right)+ |(\cI_\eta \cap \cI^{(m)})^c|\right)\nonumber\\
    &\le \frac{ \sigma}{\sqrt{\mu n}}\widetilde{O}\left(s+\frac{d}{\sqrt{M}}+\frac{\sqrt{d} }{\sqrt{n}}+\sum_{k \in[d]}\left(\frac{|\cB_k|}{M}\one(|\cB_k|/M<\eta)+ \one(|\cB_k|/M\ge \eta )\right)\right).\label{eqn:hvisdfgsd-gn}
\end{align}
Using Lemma \ref{lem:f-g} with $a=1$ and $x_k=|\cB_k|/M$ for all $k\in[d]$, we have 
\begin{equation}\label{eqn:vnixcvcxcvx}
    \sum_{k \in[d]]}\left(\frac{|\cB_k|}{M}\one(|\cB_k|/M<\eta)+ \one(|\cB_k|/M\ge \eta )\right)\le {\lceil s/\eta \rceil}.
\end{equation}
Letting $\eta =1/10=\Theta(1)$ and plugging \eqref{eqn:vnixcvcxcvx} into \eqref{eqn:hvisdfgsd-gn}, we find the conlusion.

\end{proof}

\newpage
\section{Extensions to the High-Dimensional Case}
\label{app:high-dim}
Our results in Section \ref{sec:ana-gau} rely on lower bounds on the singular values of the design matrices, \ie, on $n_m \gg d$ for all $m\in[M]$. 
When $d\gg n_m$, the design matrices are not invertible, and thus the covariate-wise shrinkage in Algorithm~\ref{alg:molar} becomes pathological. 
In this case, given a global estimate $\widehat \beta^\star$, one can employ a LASSO-based debiasing step instead of covariate-wise shrinkage as in \cite{xu2021multitask}: for all $m\in[M]$,
\begin{equation}\label{eqn:Lvncxiv}
\widehat{\beta}^{(m)}=\argmin_{\beta \in \RR^d}\frac{1}{2n_m}\|\vX^{(m)}  \beta-Y^{(m)}\|_2^2+\lambda_m \| \beta - \widehat \beta^\star\|_1. \tag{LASSO-based debiasing}
\end{equation}
Given a global estimate $\widehat \beta^\star$ that is accurate on a support set $\cG\subseteq [d]$ (\ie, $\|\widehat \beta^\star_\cG -\beta^\star_\cG \|_1$ is small), the performance of the LASSO-based debiasing can be analyzed as follows:
\begin{proposition}\label{prop:lasso-debias}
    Under Conditions \ref{asp:heterogneity}, \ref{asp:noise}, and \ref{asp:design-mats}, taking $\lambda_m = c\sigma\sqrt{\ln(dn_m)/n_m}$,
    if $n_m \geq c\ln(d)(s+|\cG^c|)$ for a 
    sufficiently large  $c$  only depending on $\Sigma^{(m)}$, where $\cG\subseteq [d]$ is a support set, it holds that 
    \begin{equation}
        \|\widehat \beta^{(m)}-\beta^{(m)}\|_1=\widetilde{O}_P\left(\frac{\sigma (s+|\cG^c|)}{\sqrt{n_m}}+\|\widehat \beta^\star_\cG-\beta^\star_\cG\|_1\right).
    \end{equation}
    where for any vector $v$, $v_\cG$ is the sub-vector of $v$ with entries in $\cG$.
\end{proposition}
\begin{proof}
Define the event $E$ be the intersection of $ \{n_m^{-1}\|\vX^{(m)\top}(\vX^{(m)} \beta^{(m)}-Y^{(m)})\|_\infty\leq \lambda_m/2\}$ and
\begin{equation}\label{eqn:ifncsixc}
   \left\{ \|\vX^{(m)} v\|_2^2\geq n_m c_1\|v\|_2\left(\|v\|_2-c_1\sqrt{{\ln(d)}/{n_m}}\|v\|_1\right),\;\forall\, v\in\RR^d\right\},
\end{equation}
where $c_1$ is a sufficiently large constant that only depends $\mu$ and $L$ in Condition \ref{asp:design-mats}. 
Here \eqref{eqn:ifncsixc} is referred to as the restricted strong convexity and holds with probability at least $1- \exp(-c_1\mu n_m)$ \citep{negahban2012unified}. By the definition of $\lambda_m$, $ \{n_m^{-1}\|\vX^{(m)\top}(\vX^{(m)} \beta^{(m)}-Y^{(m)})\|_\infty\leq \lambda_m/2\}$ holds with probability $1-1/(n_md)$. Thus, $\PP(E)\to 1$ as $n_m\to \infty$.

Using 
    \begin{align}\nonumber
        \frac{1}{2n_m}\|\vX^{(m)}  \widehat\beta^{(m)}-Y^{(m)}\|_2^2\leq \frac{1}{2n_m}\|\vX^{(m)}  \beta^{(m)}-Y^{(m)}\|_2^2+\lambda_m \| \beta^{(m)} - \widehat \beta^\star\|_1-\lambda_m \|\widehat\beta^{(m)} - \widehat \beta^\star\|_1,
    \end{align}
    we have 
    \begin{align}\nonumber
        &\frac{1}{2n_m}\|\vX^{(m)}(\widehat \beta^{(m)}-\beta^{(m)})\|_2^2
        = \frac{1}{2n_m}\|\vX^{(m)}\widehat \beta^{(m)}-Y^{(m)}\|_2^2+\frac{1}{2n_m}\|\vX^{(m)}\beta^{(m)}-Y^{(m)}\|_2^2\nonumber\\
        &\quad +\frac{1}{n_m}\langle \vX^{(m)} \widehat\beta^{(m)}-Y^{(m)},\vX^{(m)} \beta^{(m)}-Y^{(m)}\rangle\nonumber\\
        \leq &\frac{1}{n_m}\langle \widehat\beta^{(m)}-\beta^{(m)},\vX^{(m)\top}(\vX^{(m)} \beta^{(m)}-Y^{(m)})\rangle+\lambda_m \| \beta^{(m)} - \widehat \beta^\star\|_1-\lambda_m \|\widehat\beta^{(m)} - \widehat \beta^\star\|_1.\nonumber
    \end{align}
    Thus, letting $\cJ_m\triangleq \cG^c\cup \supp(\beta^{(m)}-\beta^\star)$, we have 
    $|\cJ_m|\leq s+|\cG^c|$ by Condition \ref{asp:heterogneity}.
    Since $\beta^{(m)}_{\cJ_m^c}=\beta^\star_{\cJ_m^c}$,
    conditioned on the event $E$, we have 
    \begin{align}
        &0\leq \frac{1}{2n_m}\|\vX^{(m)}(\widehat \beta^{(m)}-\beta^{(m)})\|_2^2\leq \lambda_m \| \beta^{(m)} - \widehat \beta^\star\|_1-\lambda_m \|\widehat\beta^{(m)} - \widehat \beta^\star\|_1+\frac{\lambda_m}{2}\|\widehat \beta^{(m)}-\beta^{(m)}\|_1\nonumber\\
        = &\lambda_m\|\beta^{(m)}_{\cJ_m}- \widehat \beta^\star_{\cJ_m}\|_1+\lambda_m\|\beta^\star_{\cJ_m^c}- \widehat \beta^\star_{\cJ_m^c}\|_1
         -\lambda_m\|\widehat\beta^{(m)}_{\cJ_m}-\widehat \beta^\star_{\cJ_m}\|_1-\lambda_m\|\widehat\beta^{(m)}_{\cJ_m^c}-\widehat \beta^\star_{\cJ_m^c}\|_1\nonumber\\
        &\; +\frac{\lambda_m}{2}\|\widehat \beta^{(m)}_{\cJ_m}-\beta^{(m)}_{\cJ_m}\|_1+\frac{\lambda_m}{2}\|\widehat \beta^{(m)}_{\cJ_m^c}-\beta^{(m)}_{\cJ_m^c}\|_1\nonumber\\
        \leq&\frac{3\lambda_m}{2}\| \widehat \beta^{(m)}_{\cJ_m}-\beta^{(m)}_{\cJ_m}\|_1-\frac{\lambda_m}{2}\| \widehat \beta^{(m)}_{\cJ_m^c}-\beta^{(m)}_{\cJ_m^c}\|_1
        +2\lambda_m \| \widehat \beta^\star_{\cJ_m^c}-\beta^\star_{\cJ_m^c}\|_1.\label{eqn:vnicvcxxnv}
    \end{align}
    Noting $\cJ_m^c\subseteq \cG$, we have 
    \begin{align}
        \| \widehat \beta^{(m)}_{\cJ_m^c}- \beta^{(m)}_{\cJ_m^c}\|_1\leq & 3\| \widehat \beta^{(m)}_{\cJ_m}- \beta^{(m)}_{\cJ_m}\|_1+4\| \widehat \beta^\star_{\cG}-\beta^\star_{
        \cG}\|_1\label{eqn:gndifnvc}.
    \end{align}
    Consequently, it holds that
    \begin{align}
        \|\widehat \beta^{(m)}-\beta^{(m)}\|_1\leq &4\| \widehat \beta^{(m)}_{\cJ_m}- \beta^{(m)}_{\cJ_m}\|_1+4\| \widehat \beta^\star_{\cG}-\beta^\star_{
        \cG}\|_1\nonumber\\
        \leq & 4|\cJ_m|^{1/2}\|\widehat \beta^{(m)}-\beta^{(m)}\|_2+4\| \widehat \beta^\star_{\cG}-\beta^\star_{
        \cG}\|_1,\label{eqn:viusdvnwef}
    \end{align}
        where we use Young's inequality in the last equation. 
    Denote $\| \widehat \beta^\star_{\cG}-\beta^\star_{
        \cG}\|_1$ as $\mathrm{err}^\star$.
    Now using the restricted strong convexity and \eqref{eqn:gndifnvc} in \eqref{eqn:vnicvcxxnv}, we obtain
    \begin{align}
        &2\lambda_m |\cJ_m|^{1/2}\|\widehat \beta^{(m)}-\beta^{(m)}\|_2+2\lambda_m \mathrm{err}^\star\nonumber\\
        \geq &\frac{c_1}{2}\|\widehat \beta^{(m)}-\beta^{(m)}\|_2\left(\|\widehat \beta^{(m)}-\beta^{(m)}\|_2- 4c_1\sqrt{\frac{\ln(d)|\cJ_m|}{n_m}}\|\widehat \beta^{(m)}-\beta^{(m)}\|_2-4c_1\sqrt{\frac{\ln(d)}{n_m}}\mathrm{err}^\star\right)\nonumber\\
        &\quad + \frac{\lambda_m}{2}\|\widehat \beta^{(m)} -\beta^{(m)}\|_1\nonumber
    \end{align}
    Suppose $n_m\geq c\ln(d) (s+|\cG^c|)$ with $c $ sufficiently large such that $4c_1\sqrt{{\ln(d)|\cJ_m|}/{n_m}}\leq 1/3$. If $\mathrm{err}^\star \leq |\cJ_m|^{1/2}\|\widehat \beta^{(m)} -\beta^{(m)}\|_2$, we obtain 
    \begin{align}\nonumber
         2\lambda_m |\cJ_m|^{1/2}\|\widehat \beta^{(m)}-\beta^{(m)}\|_2+2\lambda_m \mathrm{err}^\star\geq \frac{c_1}{6}\|\widehat \beta^{(m)}-\beta^{(m)}\|_2^2+ \frac{\lambda_m}{2}\|\widehat \beta^{(m)} -\beta^{(m)}\|_1.
    \end{align}
    Using Young's inequality $2ab \leq a^2+b^2$, we have 
   $\|\widehat \beta^{(m)} -\beta^{(m)}\|_1=O\left(\lambda_m |\cJ_m|+\mathrm{err}^\star\right)$.
    If $\mathrm{err}^\star > |\cJ_m|^{1/2}\|\widehat \beta^{(m)} -\beta^{(m)}\|_2$, using \eqref{eqn:viusdvnwef}, we directly obtain $ \|\widehat \beta^{(m)} -\beta^{(m)}\|_1=O\left(\mathrm{err}^\star\right)$.
\end{proof}

Notably, Proposition \ref{prop:lasso-debias} does not require $n_m >d$ and directly implies
\begin{corollary}\label{cor:l1-optimal}
    Suppose 
    \begin{equation}\label{eqn:pivot-cond}
        \|\widehat \beta^\star_\cG-\beta^\star_\cG\|_1=\widetilde{O}_P(d\sigma/\sqrt{n_{[M]}}+s\sigma/\sqrt{n_m})\quad \text{for some}\quad |\cG^c|=O(s).
    \end{equation} If $n_m\geq cs\ln(d)$ with $c$ sufficiently large, LASSO-based debiasing gives the minimax optimal estimation error:
    \begin{equation}\nonumber
        \|\widehat \beta^{(m)}-\beta^{(m)}\|_1=\widetilde O_P\left(\frac{d\sigma}{\sqrt{n_{[M]}}}+\frac{s\sigma}{\sqrt{n_m}}\right).
    \end{equation}
\end{corollary}

When $n_m>d$ for all $m\in[M]$, the weighted median-based global estimate given by Algorithm \ref{alg:molar} satisfies \eqref{eqn:pivot-cond}. However, the global estimate is not applicable if $n_m <d$ so that the data matrix $\vX^{(m)\top}\vX^{(m)}$ is not full-rank. Obtaining a good global estimate $\widehat \beta^\star$ is challenging when $n_m<d$ as the parameters $\{\beta^{(m)}\}_{m=1}^M$ can be dense. 
Fortunately, 
we can make progress
by additionally assuming the parameters $\{\beta^{(m)}\}_{m=1}^M$ are sparse (Condition \ref{asp:sparse-para}) and the heterogeneity is $\ell_2$-bounded. 

\begin{assumption}[\sc Sparse global parameter]\label{asp:sparse-para}
We have
$\|\beta^\star\|_0\le k$ for some $k \in [s,d]$ and $\max_{m\in[M]}\|\beta^{(m)}-\beta^\star\|_2\leq b$ for some constant $b\geq 0$.
\end{assumption}

To this end, we borrow the transGLM method~\citep{li2023estimation}, which leverages multiple datasets with sparse heterogeneity to learn a single generalized linear model. In the linear case, transGLM can be simplified to Algorithm \ref{alg:transGLM}.
\begin{algorithm}[t]
	\caption{Transfer learning via constrained $\ell_1$-minimization}
	\label{alg:transGLM}
	\begin{algorithmic}
		\STATE \noindent {\bfseries Input:} $\{(\vX^{(m)},Y^{(m)})\}_{m=1}^M$, regularization $\{\lambda_m=c\sigma\sqrt{\ln(n_md)/n_m}\}_{m=1}^M$ and $\lambda_{\rm all}=c\sigma\sqrt{\ln(p)/n_{[M]}}$
		\vspace{1mm}
		\STATE \noindent {\bfseries Step 1:} Compute an initial estimate $\widehat{\beta}_{\rm init}^\star=\argmin_{\beta \in \RR^d}\frac{1}{2n_1}\|\vX^{(1)}  \beta-Y^{(1)}\|_2^2+\lambda_1 \| \beta\|_1$
        \vspace{1mm}
        \STATE  \noindent {\bfseries Step 2:} Set
        \vspace{-8mm}\begin{align}
           \qquad \qquad\widehat{\beta}^\star,\widehat \delta^{(2)},\cdots,\widehat \delta^{(m)}=&\underset{\beta, \,\|\delta^{(m)}\|_2\leq b}{\operatorname{argmin}} \lambda_{\rm all}\|\beta\|_1+\sum_{m=2}^M\lambda_m\|\delta^{(m)}\|_1\nonumber\\
            &{\rm s.t.} \begin{cases}
                \|X^{(m)\top} (X^{(m)} (\beta+\delta^{(m)})-Y^{(m)})\|_\infty\leq \lambda_m,\,\forall\,2\leq m\leq M\nonumber\\
                \|\sum_{m=1}^M X^{(m)\top} (X^{(m)} (\beta+\delta^{(m)})-Y^{(m)})\|_\infty\leq \lambda_{\rm all}\\
                \|\beta -\widehat \beta_{\rm init}^\star\|_1\leq \lambda_1^{-1}\nonumber
            \end{cases}
        \end{align}
		\STATE \noindent {\bfseries Output:} $\widehat{\beta}^\star$
	\end{algorithmic}
\end{algorithm}
The estimation error of $\widehat \beta^\star$ for $\beta^\star$ directly follows from  \citet[Theorem 3.1]{li2023estimation}. We paraphrase their result in our setup with  notations defined in this paper as follows:
\begin{theorem}[\cite{li2023estimation}, Theorem 3.1]\label{thm:transglm}
    Suppose 
    \begin{equation}\nonumber
        n_{[M]}\geq c_1\max\{k^2\ln(d)^2, Mn_1\}, \quad n_1\geq c_1ks\ln(d)^2,
    \end{equation}
    where $c_1$ is some large enough quantity not depending on $d$, $k$, $s$, and $\{n_m\}_{m=1}^M$.
    Under Conditions \ref{asp:heterogneity}, \ref{asp:noise}, and \ref{asp:design-mats}, it holds that 
    \begin{equation}\nonumber
        \|\widehat \beta^\star_{\cG}-\beta^\star_{\cG}\|_1=\widetilde O_P\left(\frac{k\sigma}{\sqrt{n_{[M]}}}+\frac{\sqrt{sk}\sigma}{\sqrt{n_{1}}}\right),
    \end{equation}
    where $\cG = [d]\backslash \supp(\beta_1-\beta)$ with $|\cG^c|\leq s$.
\end{theorem}

By Combining Theorem~\ref{thm:transglm} with Proposition~\ref{prop:lasso-debias}, we have 
\begin{corollary}[Transfer Learning + Lasso Debiasing]\label{thm:high-dim}
    Suppose $n_m\geq c_1ks \ln(d)^2$ for all $1\leq m\leq M$ with $n_1/\min_m n_m=O(1)$, and $n_{[M]}\geq c_1 k^2\ln(d)^2$ 
    where $c_1$ is some large enough value not depending on $d$, $k$, $s$, and $\{n_m\}_{m=1}^M$.
    Under Conditions \ref{asp:heterogneity}, \ref{asp:noise}, and \ref{asp:design-mats}, the task-wise estimates obtained through applying the \ref{eqn:Lvncxiv} to the global estimate output by Algorithm \ref{alg:transGLM} achieves for $1\leq m\leq M$,
        \begin{equation}\nonumber
        \|\widehat \beta^{(m)}-\beta^{(m)}\|_1=\widetilde{O}_P\left(\frac{\sqrt{sk}\sigma }{\sqrt{n_m}}+\frac{k\sigma}{\sqrt{n_{[M]}}}\right).
    \end{equation}
\end{corollary}
Corollary~\ref{thm:high-dim} 
bounds the estimation error of the multitask approach with $n_m\geq c_1ks \ln(d)^2$ and  $n_{[M]}\geq c_1 k^2\ln(d)^2$, which holds in the high-dimensional case so long as $ks \ll d$ and $M$ is large. 
Furthermore, the multitask approach outperforms the individual LASSO, the single-task minimax optimal method, whose estimation error is $k\sigma /\sqrt{n_m}$. 



\newpage
\section{Inference for Task-wise Parameters}\label{app:conf-interval}
In this section, we consider statistical inference for $\{\beta^{(m)}\}_{m=1}^M$ in the multi-task learning setting. Notably, the MOLAR estimates obtained in Algorithm~\ref{alg:molar} and the high-dimensional estimates given in Appendix ~\ref{app:high-dim} are biased. 
It is tempting to consider debiasing them to facilitate inference as in~\citep{zhang2014confidence,van2014asymptotically,javanmard2014confidence}.
However,  debiasing may increase the estimation error of our multi-task estimate, 
and thus likely cannot give confidence intervals with lengths shorter than those of the individual OLS estimates.

To show the potential of narrower confidence intervals in our setup, for simplicity, we consider $n_1=\cdots=n_M=:n$ and $\sigma_1=\cdots=\sigma_M=:\sigma$, and assume $\sigma$  is known throughout the section. The case where $\{\sigma_m\}_{m=1}^M$ or $\{n_m\}_{m=1}^M$ are non-identical can be handled similarly by adjusiting the weights $\{w_m\}_{m=1}^M$ in the collaboration step.
We additionally make the following assumption to constrain the distribution of heterogeneity. Again, here the number $1/5$ is taken for simplicity and can be, in principle, replaced with a constant number in $[0,1/2)$.
\begin{assumption}[\sc  Bounded entry-wise disagreement]
    We assume $|\cB_k|/M\leq 1/5$ for all $k\in[M]$.
\end{assumption}

Given a fixed and  properly small tolerance $\alpha$, for each $k\in[d]$ and $m\in[M]$, we let $\widetilde I_{k}^{(m)}=\left[\widehat \beta_{{\rm ind}, k}^{(m)}-\sigma z_{1-\alpha/2}\sqrt{v_k^{(m)}},\widehat \beta_{{\rm ind},k}^{(m)}+\sigma z_{1-\alpha/2}\sqrt{v_k^{(m)}}\right]$ be the  interval with coverage $1-\alpha/2$ for $\beta_{{\rm ind},k}^{(m)}$ centered at $\widehat \beta_{{\rm ind},k}^{(m)}$ where $v_k^{(m)}=\sqrt{[\vX^{(m)\top}\vX^{(m)})^{-1}]_{k,k}}$.We let $I_{k}^{\star}$ be an interval with coverage $1-\alpha/2$ for $\beta^\star_k$ centered at $\widehat \beta_{k}^{\star}$. Based on Lemma~\ref{lem:median-Gaussian} and \eqref{eqn:vnisdngvzxc2},   for $\alpha_{\cB_k,\delta}<0.45$, we can set  
\begin{equation}\nonumber
     I_{k}^{\star}=[\widehat \beta_{k}^{\star}-1.25C_{0.45}\alpha_{\cB,\delta}\bar {\sigma}_{k},\widehat \beta_{k}^{\star}+1.25C_{0.45}\alpha_{\cB,\delta}\bar {\sigma}_{k}]
\end{equation}
with $\delta = \ln((8/\alpha)\vee (2n))/2$ (\ie, $2e^{-2\delta}=(\alpha/4)\wedge n^{-1}$), $\alpha_{\cB_k,\delta}\triangleq{|\cB_k|}/{M}+\sqrt{1.01\delta /(M-|\cB_k|)}$, and
$\bar {\sigma}_{k}=\sigma\sum_{m\in[M]}\sqrt{v_k^{(m)}}/M$. 
Since $v_k^{(m)}=\widetilde O_P(1/n)$, we have ${\rm length}(\widetilde I_k^{(m)})=\widetilde O_P(1/\sqrt{n})$ and ${\rm length}(I_k^{\star})=\widetilde O_P(\alpha_{\cB_k, \delta}/\sqrt{n})$ for all $k\in[d]$ and $m\in[M]$. Noting $\sum_{k=1}^d|\cB_k|/M=s/d$, we further have $\sum_{k=1}^d{\rm length}(I_k^{\star})=O_P(s/\sqrt{n}+d/\sqrt{Mn})$.

Since $\{I_k^{\star}\}_{k=1}^d$ are narrower than $\{\widetilde I_k^{(m)}\}_{k=1}^d$ on average, 
we use the following strategy to construct the ultimate confidence interval $I_k^{(m)}$ for $\beta_k^{(m)}$ with at least $1-\alpha$ entry-wise coverage.
We first compare $\widehat\beta_{{\rm ind},k}^{(m)}$ and $\widehat \beta_{k}^{\star}$. If they are close enough such that $|\widehat\beta_{{\rm ind},k}^{(m)}-\widehat\beta_{k}^{\star}|<\widetilde \gamma_m\sqrt{v_k^{(m)}}$ for some pre-specified $\widetilde \gamma_m$, it is likely that $\beta_{k}^{(m)}=\beta_k^\star$ and we thus adopt the confidence interval $I_k^\star$ as the final interval $I_k^{(m)}$; otherwise we adopt $\widetilde I_k^{(m)}$ as $I_k^{(m)}$. 
Formally, we aim to attain $\PP(\beta_k^{(m)}\in I_k^{(m)})\geq 1-\alpha$ and the total
length of entry-wise confidence intervals is $\sum_{k=1}^d {\rm length}(I_k^{(m)})$.
We show the following guarantee for the confidence intervals $\{I_k^{(m)}\}_{k=1}^d$. Without loss of generality, we assume $M\geq \tilde c \ln(n)$ for a sufficiently large constant $\tilde c$.
\begin{proposition}\label{prop:conf-intv}
    For any $\alpha\in(2e^{-2c M},1]$ with some constant $c$ sufficiently small, for any $m\in[M]$, if we set $\widetilde \gamma_m=\sqrt{2}\ln((2n)\vee (8/\alpha))\sigma$, it  holds for all $k\in[d]$ that 
    \begin{equation}\nonumber
        \PP\left(\beta_k^{(m)}\in I_k^{(m)}\right)\geq 1-\alpha
    \end{equation}
    provided $|\beta_k^{(m)}-\beta_k^\star|>3\max\{\ln(n\vee (4/\alpha))\sqrt{v_k^{(m)}}\sigma,1.25C_{\alpha_{\cB_k,\delta}}\alpha_{\cB,\delta}\bar {\sigma}_{k}\}=\widetilde \Omega_P(1/\sqrt{n})$, where $\delta = \ln((2n)\vee (8/\alpha))$ for $\beta_k^{(m)}\neq \beta_k^\star$.
    Furthermore, the total length satisfies
    \begin{equation}\nonumber
        \sum_{k=1}^d {\rm length}(I_k^{(m)})=\widetilde O_P(s/\sqrt{n}+d/\sqrt{Mn}).
    \end{equation}
\end{proposition}
\begin{proof}
Clearly, $v_k^{(m)}=\widetilde O_P(1/n)$ so long as $n\gg d$. We prove the result by considering two cases. 

    \noindent Case 1: $\beta_k^{(m)}=\beta_k^\star$. In this case, given the choice of $\widetilde \gamma_m$ and Lemma~\ref{lem:median-Gaussian}, we can easily verify $\PP(|\widehat \beta_{{\rm ind},k}^{(m)}-\widehat \beta^\star_k|\geq \widetilde \gamma_m\sqrt{v_k^{(m)}})=(\alpha/4)\wedge n^{-1}\leq \alpha/2$. This, combined with $\PP(\beta_k^\star \notin I_k^\star)\leq \alpha/2$ and a union bound, leads to 
    \begin{equation}\nonumber
        \PP(\beta_k^{(m)}\in I_k^\star)\geq \PP\left(\beta_k^{\star}\in I_k^\star \text{ and }|\widehat \beta_{{\rm ind},k}^{(m)}-\widehat \beta^\star_k|< \widetilde \gamma_m\sqrt{v_k^{(m)}}\right)\geq 1-\alpha. 
    \end{equation}
    Furthermore, with probability  $\PP\left(|\widehat \beta_{{\rm ind},k}^{(m)}-\widehat \beta^\star_k|< \widetilde \gamma_m\sqrt{v_k^{(m)}}\right)=1-o(1)$ we have that
    \begin{equation}\nonumber
        {\rm length}(I_k^{(m)})={\rm length}(I_k^{\star}) =O(\alpha_{\cB_k, \delta}/\sqrt{n})=\widetilde O\left(\left(|\cB_k|/M+1/\sqrt{M}\right)\sqrt{v_k^{(m)}}\right).
    \end{equation}

    \noindent Case 2: $\beta_k^{(m)}\neq\beta_k^\star$. Denote $ |\beta_k^{(m)}-\beta_k^\star|$ as $\epsilon$. In this case, we shall prove that $|\widehat \beta_{{\rm ind},k}^{(m)}-\widehat \beta^\star_k|\geq \widetilde \gamma_m\sqrt{v_k^{(m)}}$ with a high probability. We first see that 
    \begin{align}
        |\widehat \beta_{{\rm ind},k}^{(m)}-\widehat \beta^\star_k|=&|\widehat \beta_{{\rm ind},k}^{(m)}-\beta_k^{(m)}+\beta_k^{(m)}-\beta_k^\star+\beta_k^\star-\widehat \beta^\star_k|\nonumber\\
        \geq &\epsilon - |\widehat \beta_{{\rm ind},k}^{(m)}-\beta_k^{(m)}|-|\beta_k^\star-\widehat \beta^\star_k|.
    \end{align}
    Therefore, $|\widehat \beta_{{\rm ind},k}^{(m)}-\widehat \beta^\star_k|<  \widetilde \gamma_m\sqrt{v_k^{(m)}}$ suggests that one of $|\widehat \beta_{{\rm ind},k}^{(m)}-\beta_k^{(m)}|\geq (\epsilon- \widetilde \gamma_m\sqrt{v_k^{(m)}})/2$ or $|\beta_k^\star-\widehat \beta^\star_k|\geq (\epsilon- \widetilde \gamma_m\sqrt{v_k^{(m)}})/2$ must hold. However, by the condition on $\epsilon$ and $\widetilde \gamma_m$,  using the sub-Gaussianity of $\widehat \beta_{{\rm ind},k}^{(m)}$ and Lemma~\ref{lem:median-Gaussian}, we have 
    \begin{equation}\label{eqn:vbxcibnacz1}
        \PP\left(|\widehat \beta_{{\rm ind},k}^{(m)}-\beta_k^{(m)}|\geq (\epsilon- \widetilde \gamma_m\sqrt{v_k^{(m)}})/2\right)\leq \PP\left(|\widehat \beta_{{\rm ind},k}^{(m)}-\beta_k^{(m)}|\geq \widetilde \gamma_m\sqrt{v_k^{(m)}}\right)\leq \frac{\alpha}{4}\bigwedge \frac{1}{n}\leq \frac{\alpha}{4}
    \end{equation}
    and 
    \begin{equation}\label{eqn:vbxcibnacz2}
        \PP\left(|\beta_k^\star-\widehat \beta^\star_k|\geq (\epsilon- \widetilde \gamma_m\sqrt{v_k^{(m)}})/2\right)\leq \PP\left(|\beta_k^\star-\widehat \beta^\star_k|\geq  1.25C_{\alpha_{\cB_k,\delta}}\alpha_{\cB,\delta}\bar {\sigma}_{k}\right)\leq \frac{\alpha}{4}\bigwedge \frac{1}{n}\leq \frac{\alpha}{4}.
    \end{equation}
    where $\delta =\ln((8/\alpha)\vee (2n))/2$ is much smaller than $M$ so that $\alpha_{\cB_k,\delta}<0.45$.
    Combining \eqref{eqn:vbxcibnacz1}, \eqref{eqn:vbxcibnacz2} with $\PP(\beta_k^{(m)} \notin I_k^{(m)})\leq \alpha/2$ and a union bound, we find
    \begin{equation}\nonumber
         \PP\left(\beta_k^{(m)}\in I_k^\star)\geq\PP(\beta_k^{(m)}\in I_k^{(m)}\text{ and }|\widehat \beta_{{\rm ind},k}^{(m)}-\widehat \beta^\star_k|\geq  \widetilde \gamma_m\sqrt{v_k^{(m)}}\right)\geq 1-\alpha. 
    \end{equation}

    Denoting $\{k\in[d]:\beta_k^{(m)}=\beta_k^\star\}$ as $\cI^{(m)}$ with $|\cI^{(m)}|=s$, combining the two cases, we have 
    \begin{align}\nonumber
        \sum_{k=1}^d {\rm length}(I_k^{(m)})=\widetilde O_P\left(|\cI^{(m)}|/\sqrt{n}+\sum_{k\notin  \cI^{(m)}}\alpha_{\cB_k,\delta}/\sqrt{n}\right).
    \end{align}
    By noting that 
    \begin{equation}\nonumber
        \sum_{k\notin  \cI^{(m)}}\alpha_{\cB_k,\delta}\leq \sum_{k\in[d]}\alpha_{\cB_k,\delta}=\widetilde O(s+d/\sqrt{M}),
    \end{equation}
    we complete the proof.
\end{proof}

Proposition~\ref{prop:conf-intv} shows that our confidence intervals for the entries of the task-wise parameters $\{\beta^{(m)}\}_{m=1}^M$ have total length $\widetilde O_P(s/\sqrt{n}+d/\sqrt{Mn})$. When $s\ll d$ and $M\gg 1$, 
the length is shorter than $\widetilde O_P(d/\sqrt{n})$, the length of the standard intervals based on the individual OLS estimates. 

However, we remark that this proposition requires that the unequal entries between $\beta^{(m)}$ and  $\beta^\star$ be separated by  $\Omega(1/\sqrt{n})$. 
This condition turns out to be necessary to attain confidence intervals with a total length shorter than $\widetilde O_P(d/\sqrt{n})$. 
To see this, we can argue as follows.
Even if the global parameter $\beta^\star$ was \emph{exactly known}, 
corresponding to the case where $M=\infty$, our setup would reduce to constructing a confidence interval for $\beta^{(m)}-\beta^\star$, which is $s$-sparse. This ideal setting becomes an example of inference for single-task sparse linear regression studied by e.g., \cite{cai2017confidence}.
That work shows that the minimax optimal length of  confidence intervals for individual entries is  $\Omega_P(1/\sqrt{n})$ 
when the non-zero entries of the sparse parameter are not constrained to be away from zero (\ie, each entry is either zero or of magnitude $O(1/\sqrt{n})$). 
One can follow the same idea to show that the total length of confidence intervals with entry-wise coverage is $\Omega_P(d/\sqrt{n})$, irrespective of the parameter's sparsity.
The challenge in constructing shorter confidence intervals mainly lies in \emph{identifying the support set of the sparse parameter} when the non-zero entries of the parameter are $O(1/\sqrt{n})$.

\newpage
\section{Results on GLMs}\label{app:glm}
Given the predictors $x\in\RR^d$, if the response $y$ follows the generalized linear models (GLMs),
then its conditional distribution takes the form, for all $x\in\RR^d$,
\begin{equation}\label{eqn:glm}
    y\mid x \sim \PP(y\mid x) = \rho(y)\exp\left(y\langle x, \beta\rangle-\psi(\langle x, \beta\rangle)\right)
\end{equation}
where $\beta\in\RR^d$ is the unknown parameter, 
and $\rho$ and $\psi$  are some known univariate functions. Two important properties of GLMs are $\EE[y\mid x]=\psi^\prime(\langle x, \beta\rangle)$ and $\mathrm{Var}(y\mid x)=\psi^{\prime\prime}(\langle x, \beta\rangle)$ \citep{mccullagh2019generalized}.
In particular, 
for linear models with Gaussian noise, we have a continuous response $y$ and  $\psi(u)=u^2/(2\sigma^2)$ for all $u\in\RR$. 

In the scenario of multitask GLMs, the individual estimate $\widehat \beta^{(m)}_{\rm ind}$ can be 
taken as the minimizer of the negative log-likelihood function
\begin{equation*}
    \widehat \beta^{(m)}_{\rm ind}:=\underset{\beta \in\RR^d}{\operatorname{argmin}}\frac{1}{n_m}\sum_{i=1}^{n_m}\left(-y_i^{(m)}\langle x_i^{(m)},\beta\rangle + \psi(\langle x_i^{(m)}, \beta\rangle)\right).
\end{equation*}
Due to the nonlinearity of GLMs, to apply the MOLAR method, we also need to replace the inverse data matrix $(\vX^{(m)\top} \vX^{(m)})^{-1}$ with $ (\vX^{(m)\top} \widehat D^{(m)}\vX^{(m)})^{-1}$ where $\widehat D^{(m)}\in\RR^{n_m\times n_m}$ is the diagonal matrix with elements $\{\psi^{\prime\prime}(\langle x_i^{(m)}, \widehat \beta^{(m)}_{\rm ind}\rangle)\}_{i=1}^{n_m}$. The two adjustments form Algorithm~\ref{alg:molar-glm}.
\begin{algorithm}[h]
	\caption{{MOLAR-GLM}: Weighted-Median-based Multitask GLM}
	\label{alg:molar-glm}
	\begin{algorithmic}
		\STATE \noindent {\bfseries Input:} $\{(\vX^{(m)},Y^{(m)})\}_{m=1}^M$, thresholds $\{\gamma_m\}_{m=1}^M$, weights $\{w_m\}_{m=1}^M$
		\FOR{$m\in[M]$}
		\vspace{1mm}
		\STATE Let $\widehat{\beta}_{\mathrm{ind}}^{(m)}$ be the individual MLE for $(\vX^{(m)},Y^{(m)})$ 
		\vspace{1mm}
		\ENDFOR
	    \vspace{1mm}
		\STATE Let $\widehat{\beta}^\star = \wmed(\{\widehat{\beta}_{\mathrm{ind}}^{(m)}\}_{m=1}^M; \{w_m\}_{m=1}^M)$ be the covariate-wise weighted median
  \vspace{1mm}
		\FOR{$m\in[M]$ and $k\in[d]$}
		    \vspace{2mm}
                \STATE \textsf{/* Option I: hard thresholding */}
                \vspace{1mm}
		        \STATE $\widehat{\beta}_{{\rm MOLAR}, k}^{(m)}=\widehat{\beta}^\star_k$ \textbf{ if } $|\widehat{\beta}^\star_k-\widehat{\beta}^{(m)}_{\mathrm{ind},k}|\le \gamma_m\sqrt{[(\vX^{(m)\top}  \widehat  D^{(m)}\vX^{(m)})^{-1}]_{k,k}} $ \textbf{ else } $\smash{\widehat{\beta}^{(m)}_{\mathrm{ind},k}}$  
		        \vspace{2mm}
          \STATE \textsf{/* Option II: soft thresholding */}
                \vspace{1mm}
		        \STATE $\widehat{\beta}_{{\rm MOLAR}, k}^{(m)}=\widehat{\beta}^\star_{k}+\mathsf{SoftThresholding}\left(\widehat{\beta}^{(m)}_{\mathrm{ind},k}-\widehat{\beta}^\star_k;\gamma_m\sqrt{[(\vX^{(m)\top} \widehat D^{(m)}\vX^{(m)})^{-1})^{-1}]_{k,k}} \right)$
		        \vspace{2mm}
		\ENDFOR
		\STATE \noindent {\bfseries Output:} $\{\widehat{\beta}_{\rm MOLAR}^{(m)}\}_{m=1}^M$
	\end{algorithmic}
\end{algorithm}

We analyze  \oursglm 
for sparsely heterogeneous parameters $\{\beta^{(m)}\}_{m=1}^M$ satisfying Condition~\ref{asp:heterogneity} in the asymptotic sense where sample sizes are sufficiently large. For simplicity, we only consider $n_1=\cdots=n_M=:n$. Our analysis is built on the asymptotic normality of individual GLM estimates \citep{van2014asymptotically,xia2023debiased}.
Specifically, suppose the following holds.
\begin{assumption}[\sc Conditions for GLMs]\label{asp:glm}
 For each $m\in[M]$, the following conditions hold for the GLM model~\eqref{eqn:glm}.
\begin{enumerate}
    \item $\beta^{(m)}$ is the unique maximizer to $\EE[y_i^{(m)}\langle x_i^{(m)}, \beta\rangle-\psi(\langle x_i^{(m)}, \beta\rangle]$ and $P(y_i^{(m)}\mid x_i^{(m)})$ is quadratic mean differentiable at $\beta^{(m)}$. 
    \item $\frac{1}{n}\sum_{i=1}^ny_i^{(m)}\langle x_i^{(m)}, \beta\rangle-\psi(\langle x_i^{(m)}\rangle, \beta\rangle$ converges uniformly to $\EE[y_i^{(m)}\langle x_i^{(m)}, \beta\rangle-\psi(\langle x_i^{(m)}, \beta\rangle]$ as $n\to \infty$.
    \item $\psi$ is twice continuously differentiable and $\psi$ is uniformly bounded.
    \item The population Fisher matrix $\Sigma^{(m)}\triangleq \EE[x_i^{(m)}x_i^{(m)\top}\psi^{\prime\prime}(\langle x_i^{(m)},\beta^{(m)}\rangle)]$ is positive definite and its eigenvalues are bounded and bounded away from 0, \ie, $$cI_d\preceq \EE[x_i^{(m)}x_i^{(m)\top}\psi^{\prime\prime}(\langle x_i^{(m)},\beta^{(m)}\rangle)]\preceq CI_d$$ for some $\Omega(1)=c\leq C=O(1)$.
\end{enumerate}
\end{assumption}

Condition~\ref{asp:glm} is assumed for technical simplicity, which directly facilitates the 
analysis of maximum likelihood estimates in \citep{van2014asymptotically} and can be possibly relaxed.
Specifically, we can obtain that
\begin{proposition}[\sc Asymptotic normality]\label{prop:asymp}
    Under Condition~\ref{asp:glm}, it holds that for each $m\in[M]$ that
    \begin{equation}\nonumber
        \sqrt{n}(\widehat{\beta}_{{\rm ind}}^{(m)}-\beta^{(m)})\overset{d}{\to}\mathcal{N}\left(0,\, (\Sigma^{(m)})^{-1}\right)\text{ and } 
        V^{(m)}/n\overset{p}{\to}\Sigma^{(m)},
    \end{equation}
    where $ V^{(m)}\triangleq \vX^{(m)\top} \widehat D^{(m)}\vX^{(m)}$.
\end{proposition}
\begin{proof}
    Under Condition~\ref{asp:glm}, it is  easily have $V^{(m)}/n\overset{p}{\to}\Sigma^{(m)}$. Since $\frac{1}{n}\sum_{i=1}^ny_i^{(m)}\langle x_i^{(m)}, \beta\rangle-\psi(\langle x_i^{(m)}, \beta\rangle$ converges uniformly to $\EE[y_i^{(m)}\langle x_i^{(m)}, \beta\rangle-\psi(\langle x_i^{(m)}, \beta\rangle]$ and
    $\EE[y_i^{(m)}\langle x_i^{(m)}, \beta\rangle- $ $\psi(\langle x_i^{(m)}, \beta\rangle]$ has a unique maximum $\beta^{(m)}$ that is well-separated due to the quadratic mean differentiablility,  $\widehat{\beta}_{{\rm ind}}^{(m)}\overset{p}{\to}\beta^{(m)}$ by \citet[Theorem 5.7]{van2014asymptotically}. Then \citet[Theorem 5.39]{van2014asymptotically} guarantees $\sqrt{n}(\widehat{\beta}_{{\rm ind}}^{(m)}-\beta^{(m)})\overset{d}{\to}\mathcal{N}(0,\,(\Sigma^{(m)})^{-1})$
\end{proof}

Given the asymptotic normality of individual estimates, one can establish the following asymptotic bounds for the tail probability of the global estimate $\widehat \beta^\star$.
\begin{lemma}\label{lem:col-est-glm}
Under Condition~\ref{asp:glm}, for any $0<\eta \le \frac{1}{5}$ and $k\in\cI_\eta$, it holds for any $0\le \delta \le M/21 $ and $n$ sufficiently large that
\begin{equation*}
    \PP\left(|\widehat{\beta}^\star_k-\beta^\star_k|\ge 2C_{0.45}\alpha_{\cB_k,\delta}\bar v_{[M],k}\right)\le 2e^{-2\delta},
\end{equation*}
where $\bar {v}_{[M],k}=\sum_{m\in[M]}[(\Sigma^{(m)})^{-1/2}]_{k,k}/(\sqrt{n}M)$ and $\alpha_{\cB_k,\delta}$
follows from the definition in \eqref{eqn:cond}.
\end{lemma}
\begin{proof}
By Proposition~\ref{prop:asymp}, we know $\sqrt{n}(\widehat{\beta}_{{\rm ind},k}^{(m)}-\beta_k^{(m)})\overset{d}{\to}\mathcal{N}\left(0,\, [(\Sigma^{(m)})^{-1}]_{k,k}\right)$ for any $k\in[d]$ and $m\in[M]$. 
Note that the weighted median of $\{\widehat{\beta}_{{\rm ind},k}^{(m)}\}_{m=1}^M$ is  upper and lower bounded by the  $(1/2+|\cB_k|/M)$-weighted-quantile and $(1/2-|\cB_k|/M)$-weighted-quantile of $\{\widehat{\beta}_{{\rm ind},k}^{(m)}\}_{m\in \cG_k}$ respectively, where $\cG_k\triangleq \{m\in[M]:\beta_k^{(m)}=\beta_k^\star\}$. The $(1/2+|\cB_k|/M)$-weighted-quantile of $\{\widehat{\beta}_{{\rm ind},k}^{(m)}\}_{m\in \cG_k}$ is equivalent to the $(1/2+|\cB_k|/M)$-weighted-quantile of $\{\sqrt{n}(\widehat{\beta}_{{\rm ind},k}^{(m)}-\beta_k^{(m)})\}_{m\in \cG_k}$ after scaling and translation. We can also leverage Lemma~\ref{eqn:weight-quant} to cover the case of asymptotic normality: for any $\alpha\in[0,1/2)$, let $\mu_{1/2+\alpha}$ be value such that 
\begin{equation}\nonumber
    \sum_{i\in\cG_k}w_i\Phi\left(\frac{\mu_{1/2+\alpha}}{[(\Sigma^{(m)})^{-1}]_{k,k}}\right)=\left(\frac{1}{2}+\alpha\right)W_{\cG_k},
\end{equation}
where $\Phi$ is the c.d.f.~of the standard normal distribution. By the asymptotic normality of $\sqrt{n}(\widehat{\beta}_{{\rm ind},k}^{(m)}-\beta_k^{(m)})$, we know  
\begin{equation*}
    \sum_{i\in\cG_k}w_i\PP(\sqrt{n}(\widehat{\beta}_{{\rm ind},k}^{(m)}-\beta_k^{(m)}) \leq  \mu_{1/2+1.01\alpha})\to\sum_{i\in\cG_k}w_i\Phi\left(\frac{\mu_{1/2+1.01\alpha}}{[(\Sigma^{(m)})^{-1}]_{k,k}}\right)=\left(\frac{1}{2}+1.01\alpha\right)W_{\cG_k}.
\end{equation*}
Therefore, for $n$ sufficiently large, we have 
\begin{equation}\nonumber
     \sum_{i\in\cG_k}w_i\PP(\sqrt{n}(\widehat{\beta}_{{\rm ind},k}^{(m)}-\beta_k^{(m)}) \leq  \mu_{1/2+1.01\alpha})\geq  \left(\frac{1}{2}+\alpha\right)W_{\cG_k},
\end{equation}
which implies that the weighted population $(1/2+\alpha)$-quantile of $\{\sqrt{n}(\widehat{\beta}_{{\rm ind},k}^{(m)}-\beta_k^{(m)})\}_{m\in \cG_k}$ is upper bounded by $\mu_{1/2+1.01\alpha}$, which is in turn upper bounded by the one in Lemma~\ref{lem:gaussian-quant}. 
The lower bound can be argued similarly. In summary, we can give a bound for $\beta^\star_k$ similarly to Lemma \ref{lem:col-est} for sufficiently large $n$ by relaxing a small number close to $C_\alpha$ to as $C_{1.01\alpha}$, which gives us the result.
\end{proof}

Furthermore, one can easily follow the proof of Theorem 1 to bound the estimation errors of the MOLAR estimates in GLMs for sufficiently large $n$. 
We omit the proof here.

\begin{theorem}[\sc Error bound for task-wise parameters]
Under Conditions~\ref{asp:heterogneity} and \ref{asp:glm},
taking
$\gamma_m =c_1\sqrt{\ln(M\wedge d)}$  for all $m\in[M]$ with $c_1$ sufficiently large,
with $\widehat{\beta}^{(m)}_{\mathrm{MOLAR}}$  from Algorithm \ref{alg:molar-glm} using either Option I or II,
it holds for any $p\in\{1,2\}$, $m\in[M]$ and $n$ sufficiently large that 
\begin{equation*}
    n^{p/2}\|\widehat{\beta}^{(m)}_{\mathrm{MOLAR}}-\beta^{(m)}\|_p^p =\widetilde{O}_P\left(s+\frac{d}{
    M^{p/2}}\right).
\end{equation*}
\end{theorem}

An analytical comparison of MOLAR-GLM with related methods \citep{tian2022transfer,li2023estimation} is in Table~\ref{tab:glm-comparsion}. 
\cite{tian2022transfer,li2023estimation} originally considered sparse parameters, and we 
adjust their results to dense parameters and state them with notations defined in our manuscript for a clear comparison. 
We observe that while our method may require larger task-wise sample sizes, it theoretically has faster rates of convergence than
existing methods for GLMs (with dense parameters).

\begin{table}[H]
\centering 
\caption{\small Bounds on the estimation error $\sqrt{n}\|\widehat{\beta}^{(m)}-\beta^{(m)}\|_1$ for GLMs. Suppose $n_1=\cdots=n_M=:n$. Constants and logarithmic factors are omitted for clarity. The results  \cite{tian2022transfer,li2023estimation} are adjusted to dense parameters (\ie, $\|\beta^{(m)}\|_0=\Omega(d)$ for all $m\in[M]$).
}
\label{tab:glm-comparsion}
\footnotesize
\begin{tabular}{lccc}
\toprule
    Method & \hspace{-10mm}Heterogeneity Condition & \hspace{-10mm}Rate & Sample Size \\
    \midrule
    \cite{tian2022transfer} & $\hspace{-10mm}\max_{m\in[M]}\|\beta^{(m)}-\beta^{\star}\|_1= h$ & $\hspace{-10mm}(n^{1/4}h^{1/2})\wedge (\sqrt{n} h)+{ d}/{\sqrt{M}}$  & $Mn \gg d$, $n\gg h^2$\\
    \cite{li2023estimation} & $\hspace{-3mm}\begin{cases}\max_{m\in[M]}\|\beta^{(m)}-\beta^{\star}\|_0= s\\ \max_{m\in[M]}\|\beta^{(m)}-\beta^{\star}\|_2= O(1)
    \end{cases}$ & $\hspace{-10mm}\sqrt{sd}+d/\sqrt{M}$ & $Mn\gg d^2$, $n\gg s d$\\
    MOLAR-GLM & $\hspace{-10mm}\max_{m\in[M]}\|\beta^{(m)}-\beta^{\star}\|_0= s$ & $\hspace{-10mm}s+{ d}/{\sqrt{M}}$ & $n$ sufficiently large\\
    Lower Bound & $\hspace{-10mm}\max_{m\in[M]}\|\beta^{(m)}-\beta^{\star}\|_0= s$ & $\hspace{-10mm}s+{ d}/{\sqrt{M}}$ & -----\\
\bottomrule
\end{tabular}
\end{table}

Acknowledging the importance of extending the application of the MOLAR method to encompass a broader setup, particularly in the context of generalized linear models for $n$ being small,
we think that a detailed investigation and optimality in GLMs particularly for $n$ being small may require entirely different algorithmic designs and should indeed serve as an independent and valuable avenue for future research.

\newpage

\section{Results on Contextual Bandits under {\sf Model-C}}

\subsection{Lemma \ref{lem:eig-lb} and its Proof}\label{app:eig-lb}

To analyze individual regret, we  need that the empirical
covariance matrices, at the end of each batch, are well-conditioned  with high probability,
even when the arms are adaptively chosen. 
We  show Lemma \ref{lem:eig-lb}, which guarantees that for any $m\in[M]$, 
with high probability  the singular values of the contexts $\vX_{q}^{(m)}$ 
are
lower bounded. 
Lemma \ref{lem:eig-lb} is similar to \citep[Lemma 4]{Han2020SequentialBL} and \citep[Lemma 5]{Ren2020DynamicBL} in the single-bandit regime.
However, 
\cite{Han2020SequentialBL} assume Gaussian contexts, which is stronger than our Condition \ref{asp:covariate-subg} and \ref{asp:diverse}, 
while 
\cite{Ren2020DynamicBL} consider an $s$-sparse parameter and 
sub-Gaussian contexts.
The proof of Lemma \ref{lem:eig-lb} 
relies on using an $\ep$-net argument.
\begin{lemma}\label{lem:eig-lb}
    Under Conditions \ref{asp:covariate-subg} and \ref{asp:diverse}, for some $C_{\rm b}\ge 2$ only depending on $c_x$, 
    and for any $0\le q<Q$, it holds with probability at least $1-1/T$ that $\lambda_{\min }(\vX_{{q}}^{(m)\top}\vX_{{q}}^{(m)})\ge n_{m,q}\mu c_x/4$, for all $m \in[M]$ with $n_{m,q}\ge C_{\rm b}(\ln(MT)+d\ln(L\ln(K)/\mu))$.
\end{lemma}
\begin{proof}
For $0\le q< Q$ and $m\in[M]$, we let $\cT_q^{(m)}$ be the set of times when contexts $\vX_{q}^{m}$ are observed at instance $m$, \ie, $\vX_{q}^{m}=(x_{t,a_t^{(m)}}^{(m)})^\top_{t\in\cT_q^{(m)}}$. 
Clearly, we have $|\cT_q^{(m)}|=n_{m,q}$; 
and further $\{x_{t,a_t^{(m)}}^{(m)}: t\in\cT_q^{(m)}\}$ are independent, conditioned on $\widehat{\beta}^{(m)}_{q-1}$. The following analysis is conditional on  $\{\cT_q^{(m)}\}_{m\in [M]}$ and therefore on $\{n_{m,q}\}_{m\in[M]}$.

We first prove an upper bound on $\lambda_{\max}(\vX_{q}^{(m)\top}\vX_{q}^{(m)}/n_{m,q})$. 
For any $t\in[T]$, $a\in[K]$, $m\in[M]$
and any source vector $v \in\RR^d$, let $Z_{t,a}^{(m)}=\langle v, x_{t,a}\rangle^2$. 
Conditioned on $\widehat{\beta}^{(m)}_{q-1}$, for any $\delta>0$ and $\lambda>0$, we have 
\begin{align}
    &\PP\left(\sum_{t\in\cT_q^{(m)}}Z_{t,a_t^{(m)}}^{(m)}\ge n_{m,q}\delta\mid  \widehat{\beta}^{(m)}_{q-1}\right)\le e^{-\lambda n_{m,q} \delta}\EE\left[\exp\left(\lambda \sum_{t\in\cT_q^{(m)}}Z_{t,a_t^{(m)}}^{(m)}\right)\mid \widehat{\beta}^{(m)}_{{q-1}}\right]\nonumber\\
    =&e^{-\lambda n_{m,q}\delta}\prod_{t\in\cT_q^{(m)}}\EE\left[\exp\left(\lambda Z_{t,a_t^{(m)}}^{(m)}\right)\mid \widehat{\beta}^{(m)}_{{q-1}}\right]
    \le e^{-\lambda n_{m,q}\delta}\prod_{t\in\cT_q^{(m)}}\sum_{a \in [K]}\EE\left[\exp\left(\lambda Z_{t,a}^{(m)}\right)\right].\nonumber
\end{align}
Taking the expectation with respect to $ \widehat{\beta}^{(m)}_{q-1}$, we obtain 
\begin{align}\label{eqn:nvisdngs}
    \PP\left(\sum_{t\in\cT_q^{(m)}}Z_{t,a_t^{(m)}}^{(m)}\ge n_{m,q}\delta \right)\le e^{-\lambda n_{m,q}\delta}\prod_{t\in\cT_q^{(m)}}\sum_{a \in [K]}\EE\left[\exp\left(\lambda Z_{t,a}^{(m)}\right)\right].
\end{align}
Since $x_{t,a}^{(m)}$ is assumed to be $L$-sub-Gaussian  and $\|v\|_2=1$, $\langle v, x_{t,a}\rangle$ is also $L$-sub-Gaussian. As a result, $Z_{t,a}^{(m)}=\langle v, x_{t,a}^{(m)}\rangle^2$ is $(4\sqrt{2}L,4L)$-sub-exponential \citep{vershynin2018high}. 
By the sub-Gaussianity of $\langle v, x_{t,a}^{(m)}\rangle$, using Lemma \ref{lem:small-set-ingeral}, we have $\EE[Z_{t,a}^{(m)}]=\EE[\langle v, x_{t,a}^{(m)}\rangle^2]\le 2L(\ln(2)+1)$. Applying Bernstein's concentration inequality for sub-exponential variables, for $\delta\ge 2L(\ln(2)+1)\ge \EE[Z_{t,a}^{(m)}]$ we have 
\begin{align}
    e^{-\lambda \delta}\EE\left[\exp\left(\lambda Z_{t,a}^{(m)}\right)\right]\le &  e^{-\lambda (\delta-\EE[Z_{t,a}^{(m)}])}\EE\left[\exp\left(\lambda (Z_{t,a}-\EE[Z_{t,a}^{(m)}])\right)\right]\nonumber\\
    \le &\exp\left(-\min\left\{\frac{(\delta-\EE[Z_{t,a}^{(m)}])^2}{64L}, \frac{\delta-\EE[Z_{t,a}^{(m)}]}{32 L}\right\}\right)\nonumber\\
    \le &\exp\left(-\min\left\{\frac{(\delta-2L(\ln(2)+1))^2}{64L}, \frac{\delta-2L(\ln(2)+1)}{32 L}\right\}\right).\label{eqn:vndifgsd}
\end{align}
Combining $\eqref{eqn:vndifgsd}$ with \eqref{eqn:nvisdngs}, we obtain
for $\delta \ge 2L(\ln(2)+1)$,
\begin{align}
   &\PP\left(v^\top(\vX_{q}^{(m)\top}\vX_{q}^{(m)}/n_{m,q})v\ge \delta \right)=\PP\left(\sum_{t\in\cT_q^{(m)}}Z_{t,a_t^{(m)}}^{(m)}\ge n_{m,q}\delta \right)\nonumber\\
   \le&\exp\left(-\left(\min\left\{\frac{(\delta-2L(\ln(2)+1))^2}{64L}, \frac{\delta-2L(\ln(2)+1)}{32 L}\right\} + \ln(K)\right)n_{m,q}\right)\triangleq p_{q,\delta}^{(m)}\label{eqn:jvisdngsd}.
\end{align}

Now, we consider an $\ep$-net $\cN_d(\ep)$ of the source ball $\BB_d$ with cardinality at most $(1+2/\ep)^d$ \citep[Corollary 4.2.13]{vershynin2018high}. 
By applying a union bound to \eqref{eqn:jvisdngsd}, we have  with probability at least $1-(1+2/\ep)^d\sum_{m\in\cC_q}p_{q,\delta}^{(m)}$ that $v^\top(\vX_{q}^{(m)\top}\vX_{q}^{(m)}/n_{m,q})v\le \delta $ holds for any $v\in\cN_d(\ep)$ and $m\in\cC_q$. Now taking any source vector $u\in\RR^d$, there exists $v\in\cN_d(\ep)$ such that $\|u-v\|_2\le \ep$. Furthermore, by symmetry of $(\vX_{q}^{(m)\top}\vX_{q}^{(m)}/n_{m,q})$, we have $u^\top(\vX_{q}^{(m)\top}\vX_{q}^{(m)}/n_{m,q})v=v^\top(\vX_{q}^{(m)\top}\vX_{q}^{(m)}/n_{m,q})u$ and thus
\begin{align}
    &u^\top(\vX_{q}^{(m)\top}\vX_{q}^{(m)}/n_{m,q})u-v^\top(\vX_{q}^{(m)\top}\vX_{q}^{(m)}/n_{m,q})v= (u+v)^\top(\vX_{q}^{(m)\top}\vX_{q}^{(m)}/n_{m,q})(u-v)\nonumber\\
    \le &\ep \left\|\left(\vX_{q}^{(m)\top}\vX_{q}^{(m)}/n_{m,q}\right)(u+v)\right\|_2\le  2\ep \lambda_{\max}\left(\vX_{q}^{(m)\top}\vX_{q}^{(m)}/n_{m,q}\right)\nonumber.
\end{align}
Rearranging the above inequality gives
\begin{equation}\nonumber
    u^\top(\vX_{q}^{(m)\top}\vX_{q}^{(m)}/n_{m,q})u\le 2\ep \lambda_{\max}\left(\vX_{q}^{(m)\top}\vX_{q}^{(m)}/n_{m,q}\right)+\delta.
\end{equation}
Taking the supremum with respect to $u$, we obtain  
\begin{equation}\label{eqn:nmgodsfsd}
    \lambda_{\max}\left(\vX_{q}^{(m)\top}\vX_{q}^{(m)}/n_{m,q}\right)\le \frac{\delta }{1-2\ep}.
\end{equation}

Next we bound $\lambda_{\min}(\vX_{q}^{(m)\top}\vX_{q}^{(m)}/n_{m,q})$.
For any source vector $v\in\RR^d$, we have
\begin{equation*}
   v^\top(\vX_{q}^{(m)\top}\vX_{q}^{(m)}/n_{m,q}) v= v^\top \left(\frac{1}{n_{m,q}}\sum_{t\in\cT_q^{(m)}}x_{t,a_t^{(m)}}^{m}x_{t,a_t^{(m)}}^{m\top}\right)v \ge \frac{\mu}{n_{m,q}}\sum_{t\in\cT_q^{(m)}}\one(\langle v, x_{t,a_t^{(m)}}^{(m)} \rangle^2 \ge \mu).
\end{equation*}
Since $a_t^{(m)} = \arg\max_{a\in[K]}\langle \widehat{\beta}_{q-1}^{(m)} ,x_{t,a}^{(m)}\rangle  $ for any $t\in\cT_q^{(m)}$ and $m\in[M]$, by Condition \ref{asp:diverse}, we have 
\begin{equation*}
    \EE[\one(\langle v, x_{t,a_t^{(m)}}^{(m)} \rangle^2 \ge \mu)]=\PP(\langle v, x_{t,a_t^{(m)}}^{(m)} \rangle^2 \ge \mu)\ge c_x.
\end{equation*}
Therefore, by applying the Chernoff bound, we have
\begin{equation}\label{eqn:jvisdngsd2}
    \PP\left(v^\top(\vX_{q}^{(m)\top}\vX_{q}^{(m)}/n_{m,q}) v \le  {\mu c_x}/{2}\right)\le  e^{-c_xn_{m,q}/8}\triangleq \tilde{p}_{q}^{(m)}.
\end{equation}
By applying a union bound to \eqref{eqn:jvisdngsd2}, we have  with probability at least $1-(1+2/\ep)^d\sum_{m\in\cC_q}\tilde{p}_{q}^{(m)}$ that $v^\top(\vX_{q}^{(m)\top}\vX_{q}^{(m)}/n_{m,q})v\ge {\mu c_x}/{2} $ holds for any $v\in\cN_d(\ep)$ and $m\in\cC_q$. Taking any source vector $u\in\RR^d$, there exists $v\in\cN_d(\ep)$ such that $\|u-v\|_2\le \ep$. Furthermore, by the symmetry of $\vX_{q}^{(m)\top}\vX_{q}^{(m)}/n_{m,q}$, we have 
\begin{align}
    &u^\top(\vX_{q}^{(m)\top}\vX_{q}^{(m)}/n_{m,q})u-v^\top(\vX_{q}^{(m)\top}\vX_{q}^{(m)}/n_{m,q})v= 
    (u+v)^\top(\vX_{q}^{(m)\top}\vX_{q}^{(m)}/n_{m,q})(u-v) \nonumber\\
    \ge & - \ep \left\|\left(\vX_{q}^{(m)\top}\vX_{q}^{(m)}/n_{m,q}\right)(u+v)\right\|_2
    \ge   -2\ep \lambda_{\max}\left(\vX_{q}^{(m)\top}\vX_{q}^{(m)}/n_{m,q}\right).\nonumber
\end{align}
Rearranging the above inequality, we obtain 
\begin{equation}\label{eqn:vnosdfsdas}
    u^\top(\vX_{q}^{(m)\top}\vX_{q}^{(m)}/n_{m,q})u\ge\frac{\mu c_x}{2} - 2\ep \lambda_{\max}\left(\vX_{q}^{(m)\top}\vX_{q}^{(m)}/n_{m,q}\right).
\end{equation}
Taking the infimum in \eqref{eqn:vnosdfsdas} with respect to $u$ and using \eqref{eqn:nmgodsfsd}, we obtain  
\begin{equation}\label{eqnLvndivcxb}
    \lambda_{\min}\left(\vX_{q}^{(m)\top}\vX_{q}^{(m)}/n_{m,q}\right)\ge \frac{\mu c_x}{2} - \frac{2\ep \delta}{1-2\ep}.
\end{equation}
Finally, letting $\delta = 32\max\{L,\sqrt{L}\}(\ln(K)+1)$ and $\ep = {\mu c_x}/{(8 \delta+2\mu c_x)}$, we have $2\ep \delta/(1-2\ep)=\mu c_x/4$ and $p_{q,\delta}^{(m)} \le e^{-n_{m,q}}\le e^{-c_xn_{m,q}/8}=\tilde{p}_q^{(m)}$. Therefore, $\lambda_{\min}\left(\vX_{q}^{(m)\top}\vX_{q}^{(m)}/n_{m,q}\right)\ge {\mu c_x}/{4}$  holds for all $m\in\cC_q$ 
 with a probability of at least 
\begin{align*}
    &1-(1+2/\ep)^d \sum_{m\in\cC_q}(\tilde{p}_q^{(m)}+p_{q,\delta}^{(m)})\ge  1- 2(1+2/\ep)^d\sum_{m\in\cC_q}e^{-c_xn_{m,q}/8}\\
    =&1- 2\left(5+512\max\{L,\sqrt{L}\}(\ln(K)+1)/(\mu c_x)\right)^d\sum_{m\in\cC_q}e^{-c_xn_{m,q}/8}\\
    \ge &1- \exp\left(-\frac{c_x\min_{m\in\cC_q}n_{m,q}}{8}+\ln(2M)+d\ln\left(5+512\max\{L,\sqrt{L}\}(\ln(K)+1)/(\mu c_x)\right)\right).
\end{align*}
In particular, there exists $C_{\rm b}\ge 2$ depending only on $c_x$, such  that when $\min_{m \in \cC_q}n_{m,q}\ge C_{\rm b}(\ln(MT)+d\ln(L\ln(K)/\mu))$, the probability is lower bounded bounded by $1-1/T$.
\end{proof}

\subsection{Proof of Lemma \ref{lem:bandit-est}}\label{app:bandit-est}
\begin{proof}
Since $n_{m, 0}=\sum_{t\in\cH_{0}}\one(m\in\cS_t)$ and  $n_{m,q} = n_{m, q-1}\one\{m\notin \cC_{q-1}\}+\sum_{t\in\cH_{q}}\one(m\in\cS_t)\ge \sum_{t\in\cH_{q}}\one(\in\cS_t)$ for any $1\le q <Q$,
by using Bernstein's inequality \eqref{eqn:jgowmfqww}, we have for each $m\in[\tau]$,
\begin{align}\label{eqn:visdgs}
    \PP\left(n_{m,q}< \frac{p_m |\cH_{q}|}{2}\right)\le& \exp\left(-\frac{|\cH_{q}|p_m^2}{2(p_m(1-p_m)+p_m/2)}\right)\le \exp\left(-{|\cH_{q}|p_m}/{3}\right)\le\frac{1}{2MT},
\end{align}
where the last inequality holds because $C_{\rm b}\ge 2$ and thus $|\cH_q|\ge 2 C_{\rm b}\ln(MT)/p_{\tau}\ge   3\ln(2MT)/p_{\tau}$.  
Similarly, we have $\PP\left(n_{m,q}> {3p_m |\cH_{q}|}/{2}\right)\ge 1/(2MT)$ for all $m\in[\tau]$.
For each $0\le q<Q$, define the event 
\begin{equation*}
    \cF_{q}\triangleq \{p_m|\cH_{q}|/2\le n_{m,q} \le 3p_m|\cH_{q}|/2\text{ for all } m\in[\tau]\}.
\end{equation*}
 Using \eqref{eqn:visdgs} and applying the union bound over all $m\in[\tau]$, we have $ \PP(\cF_{q})\ge 1-1/T$. Furthermore, by the condition on $|\cH_q|$, we have for all $m\in[\tau]$ that
 \begin{equation*}
     p_m|\cH_{q}|/2\ge C_{\rm b}(\ln(MT)+d\ln(L\ln(K)/\mu).
 \end{equation*}
 Therefore, on the event $\cF_{q}$, we have $[\tau]\subseteq \cC_{q}$, with $\cC_{q}$ from \eqref{cq}.  
For each $0\le q<Q$, define the event 
\begin{equation*}
    \cE_{q}\triangleq \left\{\lambda_{\min}(\vX_{{q}}^{(m)\top}\vX_q^{m})\ge  n_{m,q} \mu c_x/4 \text{ for all } m\in\cC_{q}\right\}.
\end{equation*}  By Lemma \ref{lem:eig-lb}, we have  $\PP(\cE_{q})\ge 1- 1/T$.
On the event $\cE_{q}\cap \cF_q$, which holds with probability at least $1-2/T$,  using Condition \ref{asp:freq}, we have 
 \begin{align*}
     \min_{1\le m\le |\cC_q|}\frac{n_{1,q}\vee (n_{\cC_q,q}/m)}{n_{m,q}}\le& 6\min_{1\le m\le |\cC_q|}\frac{p_{1}\vee (p_{\cC_q}/m)}{p_{m}}
     \le 6c_{\rm f}=\widetilde{O}(1).
 \end{align*}
Therefore, by applying Theorem \ref{thm:ls-fixed-design} with $p=2$ and using definition of $\cE_q$, 
we have for any $m\in\cC_{q}$ that 
\begin{align}
    \EE\left[\|\widehat{\beta}^{m}_{q}-\beta^{(m)}\|_2^2\mid(\vX^{(m)}_q, Y^{(m)}_q)_{m\in\cC_q}\right]
    =&\widetilde{O}\left(\frac{1}{{\mu}}\left(\frac{s}{n_{m,q}}+\frac{d }{n_{\cC_q,q}}\right)\right).\label{eqn:vncxisu3zk}
\end{align}
Since event $\cF_{q}$ implies  $n_{m,q} \ge p_m |\cH_{q}|/2$ and $[\tau]\subseteq \cC_q$ for all $m\in\cC_q$, we thus have, 
using the definition of $\tau$,
\begin{equation}\label{eqn:vjiods1}
   n_{\cC_q,q}\ge n_{[\tau],q}\ge |\cH_q| \tau p_{\tau}/2= \widetilde{O}\left(|\cH_q|p_{[M]}\right) .
\end{equation}
Plugging \eqref{eqn:vjiods1} into \eqref{eqn:vncxisu3zk}, we reach the conclusion.
\end{proof}

\subsection{Proof of Theorem \ref{thm:bc-bandit}}\label{app:bc-bandit}
\begin{proof}
For any $t\in\cH_q$, $0\le q\le Q$, and $m\in [M]$,  it holds that 
\begin{align}
    &\max_{a\in[K]}\langle x_{t,a}^{(m)}-x_{t,a_t^{(m)}}^{(m)},\beta^{(m)}\rangle \one(m\in\cS_t)
    \le \max_{a\in[K], a^\prime \in[K]}\langle x_{t,a}^{(m)}-x_{t,a^\prime}^{(m)},\beta^{(m)}\rangle\one(m\in\cS_t)\nonumber\\
    =&\left(\max_{a\in[K]}\langle x_{t,a}^{(m)},\beta^{(m)}\rangle+
    \max_{a^\prime\in[K]}|\langle x_{t,a^\prime}^{(m)},\beta^{(m)}\rangle|\right)\one(m\in\cS_t)\le   2\max_{a\in[K]}|\langle x_{t,a}^{(m)}, \beta^{(m)}\rangle|\one(m\in\cS_t).\label{eqn:vnicsxnvx1}
\end{align}
Also, by the definition of $a_t^{(m)}$,
the instantaneous regret can be bounded as
\begin{align}
    &\max_{a\in[K]}\langle x_{t,a}^{(m)}-x_{t,a_t^{(m)}}^{(m)},\beta^{(m)}\rangle \one(m\in\cS_t)\le  \max_{a\in[K]}\langle x_{t,a}^{(m)}-x_{t,a_t^{(m)}}^{(m)},\beta^{(m)}-\widehat{\beta}^{m}_{q-1}\rangle \one(m\in\cS_t)\nonumber\\
    \le & \max_{a\in[K], a^\prime \in[K]}\langle x_{t,a}^{(m)}-x_{t,a^\prime}^{(m)},\beta^{(m)}-\widehat{\beta}^{m}_{q-1}\rangle\one(m\in\cS_t)\le  2\max_{a\in[K]}|\langle x_{t,a}^{(m)}, \beta^{(m)}-\widehat{\beta}^{m}_{q-1}\rangle|\one(m\in\cS_t).\label{eqn:vnicsxnvx2}
\end{align}
Combinig \eqref{eqn:vnicsxnvx1} with \eqref{eqn:vnicsxnvx2}, we obtain, for $m\in\cS_t$,
\begin{align}
    \max_{a\in[K]}\langle x_{t,a}^{(m)}-x_{t,a_t^{(m)}}^{(m)},\beta^{(m)}\rangle \le 2\left(\max_{a\in[K]}|\langle x_{t,a}^{(m)}, \beta^{(m)}\rangle|\right)\wedge\left(\max_{a\in[K]}|\langle x_{t,a}^{(m)}, \beta^{(m)}-\widehat{\beta}^{m}_{q-1}\rangle| \right) 
    \label{eqn:vbcixbvxa}
\end{align}
Taking the expectation of \eqref{eqn:vbcixbvxa} multiplied by 
$\one(m\in\cS_t)$,
conditioned on $\widehat{\beta}^{(m)}_{q-1}$,
by Condition \ref{asp:covariate-subg} and Lemma \ref{lem:max-ineq}, we have
\begin{align}
     &\EE\left[ \max_{a\in[K]}\langle x_{t,a}^{(m)}-x_{t,a_t^{(m)}}^{(m)},\beta^{(m)}\rangle \one(m\in\cS_t)\mid \widehat{\beta}^{m}_{q-1}\right]\nonumber\\
    \le &2\left(\|\beta^{(m)}\|_2\wedge\|\beta^{(m)}-\widehat{\beta}^{m}_{q-1}\|_2\right)\sqrt{2\ln(2K) L}\cdot \PP(m\in\cS_t)\nonumber\\
    \le &2p_m\left(1\wedge\|\beta^{(m)}-\widehat{\beta}^{m}_{q-1}\|_2\right)\sqrt{2\ln(2K) L}.\label{eqn:vniscnvxd-async}
\end{align}
Taking expectations in \eqref{eqn:vniscnvxd-async} with respect to $\widehat{\beta}_{q-1}^{(m)}$, we find 
\begin{equation}\label{eqn:vncsvnsLdsad}
     \EE\left[ \max_{a\in[K]}\langle x_{t,a}^{(m)}-x_{t,a_t^{(m)}}^{(m)},\beta^{(m)}\rangle \one(m\in\cS_t)\right]\le 2\sqrt{2\ln(2K)L}p_m\left(\EE[1\wedge\|\widehat{\beta}^{m}_{q-1}-\beta^{(m)}\|_2]\right).
\end{equation}
Given \eqref{eqn:vncsvnsLdsad}, the remaining key step
is to bound the estimation error 
$\EE[\|\widehat{\beta}^{m}_{q-1}-\beta^{(m)}\|_2]$ for all $0\le q\le Q$ and $m\in[M]$.  
Letting 
$\tilde{Q}= \lceil \log_2(C_{\rm b}(\ln(MT)+d\ln(L\ln(K)/\mu))/(H_0 p_{\tau})\rceil +3$ with $C_{\rm b}$ defined in Lemma \ref{lem:eig-lb}, we first bound $\EE[\|\widehat{\beta}^{m}_{q-1}-\beta^{(m)}\|_2]$ for $m\in[\tau]$. 

\vspace{2mm}
\noindent\textbf{Case 1. }When $0\le q<\tilde{Q}$, using 
$1\wedge\|\widehat{\beta}^{m}_{q-1}-\beta^{(m)}\|_2\le 1$
in \eqref{eqn:vncsvnsLdsad}, 
we have
\begin{equation}\label{eqn:vncsvnsL-pdsa}
    \sum_{t\in\cH_q}\EE\left[\max_{a\in[K]}\langle x_{t,a}^{(m)}-x_{t,a_t^{(m)}}^{(m)},\beta^{(m)}\rangle \one(m\in\cS_t)\right]=\widetilde{O}(p_m|\cH_q|\sqrt{L}).
\end{equation}

\vspace{2mm}
\noindent\textbf{Case 2. }
When $\tilde{Q}\le q\le Q$, we have 
by the definition of $\tilde Q$,
\begin{equation*}
    |\cH_{q-1}|=2^{q-2}H_0\ge 2C_{\rm b}(\ln(MT)+d\ln(L\ln(K)/\mu))/p_{\tau}.
\end{equation*}
Let $\cG_{q-1}$ be the event that $[\tau]\subseteq \cC_{q-1}$, with $\cC_{q-1}$ from \eqref{cq}. 
From Theorem \ref{thm:ls-fixed-design} with $p=2$, 
it follows for all $m\in\cC_{q-1}$ that
    \begin{equation*}
        \EE[1\wedge\|\widehat{\beta}^{(m)}_{q-1}-\beta^{(m)}\|_2^2 \mid (\vX^{(m)}_{q-1},Y^{(m)}_{q-1})_{m\in\cC_{q-1}}]=\widetilde{O}\left(\frac{1}{\mu |\cH_{q-1}|}\left(\frac{s}{p_m}+\frac{d}{p_{[M]}}\right)\right).
    \end{equation*}
Following the argument from Lemma \ref{lem:bandit-est}, 
we know  $\PP(\cG_{q-1})\ge 1-2/T$. Marginalizing over $(\vX^{(m)}_{q-1},Y^{(m)}_{q-1})_{m\in\cC_{q-1}}$ and using  Jensen's inequality, we have for any $m\in[\tau]$ that
\begin{equation}\label{eqn:bnicdnvxc}
    \EE[1\wedge\|\widehat{\beta}^{(m)}_{q-1}-\beta^{(m)}\|_2]=\widetilde{O}\left(\frac{1}{\sqrt{\mu |\cH_{q-1}|}}\sqrt{\frac{s}{p_m}+\frac{d}{p_{[M]}}}+\frac{1}{T}\right)
\end{equation}
where $1/T$ appears by considering the complement of $\cG_{q-1}$, via the  bound $1\wedge\|\widehat{\beta}^{(m)}_{q-1}-\beta^{(m)}\|_2\le 2$. Plugging \eqref{eqn:bnicdnvxc} into \eqref{eqn:vncsvnsLdsad}, we obtain
\begin{align}
    &\sum_{t\in\cH_q}\EE\left[\max_{a\in[K]}\langle x_{t,a}^{(m)}-x_{t,a_t^{(m)}}^{(m)},\beta^{(m)}\rangle\one(m\in\cS_t)\right]\nonumber\\
    =&\widetilde{O}\left(|\cH_q|p_m\left(\sqrt{\frac{L}{{\mu |\cH_{q-1}|}} \left(\frac{s}{p_m}+\frac{d}{p_{[M]}}\right)}+\frac{\sqrt{L}}{T}\right)\right)\nonumber\\
    =&\widetilde{O}\left(\sqrt{\frac{L|\cH_q|(p_m)^2}{{\mu }} \left(\frac{s}{p_m}+\frac{d}{p_{[M]}}\right)}+\frac{\sqrt{L}|\cH_q|p_m}{T}\right),\label{eqn:vnidnscxx}
\end{align}
where the last equation holds because $|\cH_q|\le 2|\cH_{q-1}|$.
Combining the bounds \eqref{eqn:vnidnscxx} with  \eqref{eqn:vncsvnsL-pdsa} for the two cases, we obtain for each $m\in[\tau]$ that 
\begin{align}
    \EE[R_T^{(m)}]=&\sum_{t=1}^{T}\EE\left[\max_{a\in[K]}\langle x_{t,a}^{(m)}-x_{t,a_t^{(m)}}^{(m)},\beta^{(m)}\rangle \one(m\in\cS_t)\right]\nonumber\\
    =& \left(\sum_{q=0}^{\tilde{Q}-1}+\sum_{q=\tilde{Q}}^Q\right)\sum_{t\in\cH_q}\EE\left[\max_{a\in[K]}\langle x_{t,a}^{(m)}-x_{t,a_t^{(m)}}^{(m)},\beta^{(m)}\rangle\one(m\in\cS_t)\right]\nonumber\\
    =&\widetilde{O}\left(\sqrt{L}p_m\sum_{q=0}^{\tilde{Q}-1}|\cH_{q}|+\sum_{q=\tilde{Q}}^Q\left(\sqrt{\frac{L|\cH_q|(p_m)^2}{{\mu }} \left(\frac{s}{p_m}+\frac{d}{p_{[M]}}\right)}+\frac{\sqrt{L}|\cH_q|p_m}{T}\right)\right).\label{eqn:cxLfvbdsufbsd}
\end{align}
By direct calculation, we have $\sum_{q=\tilde{Q}}^Q|\cH_q|/T\le 1$, while
\begin{equation*}
    \sum_{q=0}^{\tilde{Q}-1}|\cH_{q}|=O(2^{\tilde{Q}}H_0)= O(C_{\rm b}(\ln(MT)+d\ln(L\ln(K)/\mu))/(H_0 p_{\tau}) \times H_0)=\widetilde{O}(d/p_{\tau}),
\end{equation*}
and 
\begin{equation*}
    \sum_{q=\tilde{Q}}^Q\sqrt{|\cH_{q}|}=O\left(\sum_{q=\tilde{Q}}^Q2^{q/2}\sqrt{H_0}\right)=O\left(2^{Q/2}\sqrt{H_0}\right)=O\left(\sqrt{T}\right).
\end{equation*}
Therefore, from \eqref{eqn:cxLfvbdsufbsd},
we obtain
\begin{align*}
    \EE[R_T^{(m)}]=&\widetilde{O}\left(\sqrt{L}d p_m/p_{\tau}+\sqrt{\frac{L}{{\mu }} \left(s+\frac{d p_m}{p_{[M]}}\right)T p_m}\right)\\
    =&\widetilde{O}\left(\sqrt{L}d +\sqrt{\frac{L}{{\mu }} \left(s+\frac{d p_m}{p_{[M]}}\right)T p_m}\right),
\end{align*}
where the second equation is due to $p_m/p_{\tau}\le c_{\rm f}=\widetilde{O}(1)$.

Next we bound the estimation errors for $m \notin [\tau]$, and thus also the regret.  Let $Q_m=  \left(\lceil \log_2(C_{\rm b}(\ln(MT)+d\ln(L\ln(K)/\mu))/(H_0 p^{m})\rceil +3\right)\wedge Q$ for each $m\notin [\tau]$.
If $p_m$ is sufficiently small such that $Q_m=Q$, then we have $T p_m\le 4|\cH_{Q-1}|p_m=\widetilde{O}(d)$, which, combined with \eqref{eqn:vncsvnsLdsad}, directly implies 
\begin{align*}
\sum_{t=1}^{T}\EE\left[\max_{a\in[K]}\langle x_{t,a}^{(m)}-x_{t,a_t^{(m)}}^{(m)},\beta^{(m)}\rangle \one(m\in\cS_t)\right]=O\left(\sqrt{L}Tp_m\right)=O\left(\sqrt{L}\cdot d\wedge (Tp_m)\right).
\end{align*}
Otherwise if $Q_m<Q$, $Tp_m=\widetilde{\Omega}(d)$. In this case, we show that  $m\in\cC_{q-1}$ with high probability for $q\ge Q_m$. Using $n^{(m)}_{q-1} \ge \sum_{t\in\cH_{q-1}}\one(m\in\cS_t)$ for any $q\in[Q]$ and Bernstein's inequality \eqref{eqn:jgowmfqww}, we have for each $m\in[\tau]$,
\begin{align*}
    \PP\left(n^{(m)}_{q-1}< \frac{p_m |\cH_{q-1}|}{2}\right)\le &\exp\left(-{|\cH_{q-1}|p_m}/{3}\right) \le \frac{1}{MT}.
\end{align*}
Letting $\cF_{q-1}^{(m)}=\{n_{m,q-1}\ge p_m|\cH_{q-1}|/2\}$, we have $\PP(\cF_{q-1}^{(m)})\ge 1-1/(MT)$ and furthermore $\cF_{q-1}^{(m)}$ implies $m\in\cC_{q-1}$. Following the arguments in \eqref{eqn:bnicdnvxc}, \eqref{eqn:vnidnscxx}, we similarly have that for any $q\ge Q_m$,
$\EE\left[1\wedge\|\widehat{\beta}^{m}_{q-1}-\beta^{(m)}\|_2\right]$
is bounded by \eqref{eqn:vnidnscxx}.
Therefore, using $\sum_{q=0}^{{Q}^{(m)}-1}|\cH_{q}|=\widetilde{O}(d/p^{m})$ and $   \sum_{q=Q_m}^Q\sqrt{|\cH_{q}|}=O(\sqrt{T})$,  we can proceed as in \eqref{eqn:cxLfvbdsufbsd},
and then bound
\begin{align*}
     \sqrt{L}\sum_{q=0}^{Q_m-1}p_m|\cH_q|
    = 
    \widetilde{O}(\sqrt{L}\cdot d\wedge(Tp_m)).
\end{align*}
to reach the desired conclusion.

\end{proof}

\subsection{Proof of Theorem \ref{thm:lb-bandit}}\label{app:lb-bandit}
\begin{proof}
The proof strategy for the term $\Omega(\sqrt{(s+dp_m/p_{[M]})Tp_m})$ is similar to Theorem \ref{thm:lower1}:
we  prove $\Omega(\sqrt{sTp_m})$ and $\Omega(\sqrt{dTp_m/p_{[M]}})$ by considering two cases: 
the \emph{homogeneous} case where $\beta^{(1)}=\cdots =\beta^{(M)}=\beta^\star$ and the \emph{$s$-sparse} case where $\beta^\star =0$ and $\|\beta^{(m)}\|_0\le s$ for all $m\in[M]$.


\vspace{2mm}
\noindent
\textbf{The homogeneous case. } The lower bound of $\Omega(\sqrt{dT})$ for a single  linear contextual bandit is proved in  \citep{Han2020SequentialBL,chu2011contextual}. We will follow a similar method to prove the risk is bounded as $\Omega(\sqrt{dTp_m^2/p_{[M]}})$. 
Since $\beta^{(1)}=\cdots =\beta^{(M)}=\beta^\star$ in this case, we omit the superscript $m$ in $\beta^{(m)}$ for simplicity.

We consider a $2$-armed instance, \ie, $K=2$. 
Denote by $Q$ the uniform distribution over $\{\beta \in \RR^d:\|\beta\|_2= \Delta\}$ where $\Delta\in[0,1]$ will be specified later. 
We let $Q$ be the prior distribution for $\beta$,
and let $D\triangleq\cN(0, L\cdot I_{d})$ be the distribution of contexts.
By \citep[Lemma 1]{Ren2020DynamicBL}, 
this choice of context distribution satisfies Conditions \ref{asp:covariate-subg} and \ref{asp:diverse} with $c_x=\Theta(1)$.

Let $\cS_t$ be the set of activated bandits at the $t$-th round and denote by $P_{\beta ,x,t}$ the distribution of the observed rewards $\{\{y_\ell^{(m)} \one(m \in \cS_\ell) 
\}_{m=1}^M\}_{\ell=1}^{t}$ up to time $t$, conditioned on parameter $\beta$ and contexts $\{x_{\ell,a}^{(m)}:a\in[2], m\in[M]\}_{\ell=1}^{t}$. 
Since the event $\{m\in\cS_t\}$ is independent of the history and of the contexts $\{x_{t,a}^{(m)}:a\in[2]\}$ at the current round, we have
\begin{align}\label{eqn:vgnidsgs}
    \sup_{\beta}\EE[R_T^{(m)}(A)]\ge& \EE_{\beta \sim Q}[R_T^{(m)}(A)]=\sum_{t=1}^T\EE_Q\left[\EE_{D}\left[\EE_{P_{\beta,x,t-1}}\left[p_m\max_{a\in[2]}\langle x_{t,a}^{(m)}-x_{t,a_t^{(m)}}^{(m)},\beta\rangle  \right]\right]\right]
\end{align}
where the factor of $p_m$  in \eqref{eqn:vgnidsgs} is due to the integration over the randomness of $\{m\in\cS_t\}$.
Letting $d_t^{(m)}\triangleq x_{t,2}^{(m)}-x_{t,1}^{(m)}$ for all $m\in[M]$ and $t\ge 1$, we have 
\begin{equation*}
    \max_{a\in[2]}\langle x_{t,a}^{(m)}-x_{t,a_t^{(m)}}^{(m)},\beta\rangle=\one(a_t^{(m)}=1)\langle d_t^{(m)},\beta\rangle_++\one(a_t^{(m)}=2)\langle d_t^{(m)},\beta\rangle_-
\end{equation*}
where the subscripts $+$ and $-$ denote the positive and negative part respectively, \ie, $u_+=\max\{u,0\}$  and  $u_-=\max\{-u,0\}$ for any $u \in\RR$. 
Therefore, from \eqref{eqn:vgnidsgs}, we have
\begin{align}
\sup_{\beta}\EE[R_T^{(m)}(A)]\ge&p_m\sum_{t=1}^T\EE_Q\left[\EE_{D}\left[\EE_{P_{\beta,x,t-1}}\left[\one(a_t^{(m)}=1)\langle d_t^{(m)},\beta\rangle_++\one(a_t^{(m)}=2)\langle d_t^{(m)},\beta\rangle_-\right]\right]\right].\label{eqn:videvbvcx}
\end{align}
For any $1\le t\le T$, conditioned on $d_t^{(m)}$, we define two new measures $Q_t^{(m)+}$ and $Q_t^{(m)-}$ 
over $\RR^d$
via the Radon–Nikodym derivatives $\d Q_t^{(m)+}(\beta)=\langle d_t^{(m)},\beta\rangle_+/Z(d_t^{(m)})\d Q$ and $\d Q_t^{(m)-}(\beta)=\langle d_t^{(m)},\beta\rangle_-/Z(d_t^{(m)})\d Q$,
for all $\beta \in\RR^d$,
where  $Z(d_t^{(m)})=\EE_Q[\langle d_t^{(m)},\beta\rangle_+]=\EE_Q[\langle d_t^{(m)},\beta\rangle_-]$ is a normalization factor. 
Here $\EE_Q[\langle d_t^{(m)},\beta\rangle_+]=\EE_Q[\langle d_t^{(m)},\beta\rangle_-]$ due to the symmetry of $Q$.
Plugging the definitions of $Q_t^{(m)+}$ and $Q_t^{(m)-}$ into \eqref{eqn:videvbvcx}, and changing the integration order, we have
\begin{align}
    \sup_{\beta}\EE[R_T^{(m)}(A)]\ge&p_m\sum_{t=1}^T\EE_{D}\left[Z(d_t^{(m)})\left(\EE_{P_{\beta,x,t-1}\circ Q_t^{(m)+}}[\one(a_t=1)]+\EE_{P_{\beta,x,t-1}\circ Q_t^{(m)-}}[\one(a_t=2)]\right)\right]\nonumber\\
    \ge&p_m\sum_{t=1}^T\EE_{D}\left[Z(d_t^{(m)})\left(1-\texttt{TV}(P_{\beta,x,t-1}\circ Q_t^{(m)+},P_{\beta,x,t-1}\circ Q_t^{(m)-})\right)\right]\nonumber\\
    \ge& p_m\sum_{t=1}^T\EE_{D}\left[Z(d_t^{(m)})\left(1-\sqrt{\frac{1}{2}D_\mathrm{KL}(P_{\beta,x,t-1}\circ Q_t^{(m)+}\,\|\,P_{\beta,x,t-1}\circ Q_t^{(m)-})}\right)\right],\label{eqn:nvicdbve}
\end{align}
where the second inequality follows the definition of the total variation distance: $P_1(A)+P_2(A^c)\ge 1- \texttt{TV}(P_1,P_2)$ with $A=\{a_t=1\}$,  and the third inequality follows from 
Pinsker's inequality  $\texttt{TV}(P_1,P_2)\le\sqrt{D_\mathrm{KL}(P_1\,\|\,P_2)/2}$. 

Due to the distribution of the contexts, ${d}_t^{(m)}\neq 0$ with probability one.
Hence we can let $u_t^{(m)}={d}_t^{(m)}/\|{d}_t^{(m)}\|_2$ if ${d}_t^{(m)}\neq 0$; and the zero-probability set where $d_t^{(m)}=0$ does not affect the result.
Further, let
\begin{equation}\label{eqn:nbvindgde}
    \tilde{\beta}_t^{m} = \beta-2\langle u_t^{(m)},\beta\rangle u_t^{(m)}.
\end{equation}
Since $\langle d_t^{(m)},\tilde{\beta}_t^{m} \rangle=-\langle d_t^{(m)},\tilde{\beta}_t^{m}\rangle$,
we have  $\langle d_t^{(m)},\beta \rangle_-=\langle d_t^{(m)},\tilde{\beta}_t^{m}\rangle_+$. Furthermore, \eqref{eqn:nbvindgde} has the inverse transformation 
\begin{equation}\label{eqn:nbvindgde-inv}
    \beta = \tilde{\beta}_t^{m} -2\langle u_t^{(m)},\tilde{\beta}_t^{m} \rangle u_t^{(m)}.
\end{equation}
Since $\langle d_t^{(m)},\beta \rangle_-=\langle d_t^{(m)},\tilde{\beta}_t^{m} \rangle_+$ and $Q$ is reflection-invariant, one can equivalently obtain $\beta\sim Q_t^{(m)-}$ by first generating  $\tilde{\beta}_t^{(m)}\sim Q_t^{(m)+}$ and then calculating $\beta$ via \eqref{eqn:nbvindgde-inv}. 
Since $P_{\beta,x,t-1}\circ Q_t^{(m)-}$ 
means drawing $\beta$ from $Q_t^{(m)-}$ first and then gaining rewards given such $\beta$ and the  independently sampled contexts, it thus holds that 
\begin{align*}
    P_{\beta,x,t-1}\circ (\beta \sim Q_t^{(m)-})
    =&P_{\tilde{\beta}_t^{(m)} -2\langle u_t^{(m)},\tilde{\beta}_t^{(m)}\rangle u_t^{(m)} ,x,t-1}\circ(\tilde{\beta}_t^{(m)} \sim Q_t^{(m)+})\\
    \overset{d}{=}&P_{\beta-2\langle u_t^{(m)},\beta \rangle u_t^{(m)} ,x,t-1}\circ(\beta \sim Q_t^{(m)+}).
\end{align*}
In the second equation above  we changed the notation from 
$\tilde{\beta}_t^{(m)}\mapsto \beta$.
From this 
and \eqref{eqn:nvicdbve}, 
we have
\begin{align*}
    &\sup_{\beta}\EE[R_T^{(m)}(A)]\\
    \ge& p_m \sum_{t=1}^T\EE_{D}\left[Z(d_t^{(m)})\left(1-\sqrt{\frac{1}{2}D_\mathrm{KL}(P_{\beta,x,t-1}\circ Q_t^{(m)+}\,\|\,P_{\beta-2\langle u_t^{(m)},\beta \rangle u_t^{(m)} ,x,t-1} \circ Q_t^{(m)+}}\right)\right].
\end{align*}
By Lemma \ref{lem:convexity}, this is lower bounded by
\begin{equation}
    p_m \sum_{t=1}^T\EE_{D}\left[Z(d_t^{(m)})\left(1-\sqrt{\frac{1}{2}\EE_{\beta \sim Q_t^{(m)+}}\left[D_\mathrm{KL}(P_{\beta,x,t-1}\,\|\,P_{\beta-2\langle u_t^{(m)},\beta \rangle u_t^{(m)} ,x,t-1})\right]}\right)\right]. \label{eqn:vnixcnvsdf}
\end{equation}
Since the reward noise follows a $\cN(0,1)$ distribution, conditioned on the activation sets $\{\cS_\ell\}_{\ell=1}^{t-1}$, 
we have $P_{\beta,x,t-1}=\otimes_{\ell=1}^{t-1}\otimes_{r \in \cS_\ell}\cN(\langle \beta,x_{\ell,a_\ell^{(r)}}^{(r)}\rangle, 1)$. Furthermore, by the formula for the Kullback-Leibler divergence between two Gaussian distributions, we have 
\begin{align}
    &D_\mathrm{KL}\left(P_{\beta,x,t-1}\,\|\,P_{\beta-2\langle u_t^{(m)},\beta \rangle u_t^{(m)} ,x,t-1} \mid \{\cS_\ell\}_{\ell=1}^{t-1}\right)\nonumber\\
    =&\frac{1}{2}\sum_{\ell=1}^{t-1}\sum_{r \in \cS_\ell}\left(\langle \beta,x_{\ell,a_\ell^{(r)}}^{(r)}\rangle-\left\langle \beta-2\langle u_t^{(m)},\beta\rangle u_t^{(m)},x_{\ell,a_\ell^{(r)}}^{(r)}\right\rangle\right)^2\nonumber\\
    =&{2\langle u_t^{(m)},\beta\rangle^2}\sum_{\ell=1}^{t-1}\sum_{r\in \cS_{\ell}}\langle u_t^{(m)},x_{\ell,a_\ell^{(r)}}^{(r)}\rangle^2
    ={2\langle u_t^{(m)},\beta\rangle^2}\sum_{\ell=1}^{t-1}\sum_{r=1}^M\left\langle u_t^{(m)},x_{\ell,a_\ell^{(r)}}^{(r)}\right\rangle^2 \one(r\in \cS_{\ell}).\label{eqn:nvidgsd}
\end{align}
Since the events $\{r\in \cS_{\ell}\}$ for $\ell\in[t-1]$ and $r\in[M]$ are independent of contexts and the variable $\beta$, using \eqref{eqn:nvidgsd} and Lemma \ref{lem:convexity}, we obtain 
\begin{align}
    &D_\mathrm{KL}(P_{\beta,x,t-1}\,\|\,P_{\beta-2\langle u_t^{(m)},\beta \rangle u_t^{(m)} ,x,t-1})\nonumber\\
    \le&\EE_{\{\cS_\ell\}_{\ell=1}^{t-1}}\left[D_\mathrm{KL}\left(P_{\beta,x,t-1}\,\|\,P_{\beta-2\langle u_t^{(m)},\beta \rangle u_t^{(m)} ,x,t-1} \mid \{\cS_\ell\}_{\ell=1}^{t-1}\right)\right] \nonumber\\
    =&{2\langle u_t^{(m)},\beta\rangle^2}\sum_{\ell=1}^{t-1}\sum_{r\in[M]}\langle u_t^{(m)},x_{\ell,a_\ell^{(r)}}^{(r)}\rangle^2 \PP(r\in \cS_{\ell})
    ={2\langle u_t^{(m)},\beta\rangle^2}p_{[M]}\sum_{\ell=1}^{t-1}\langle u_t^{(m)},x_{\ell,a_\ell^{(r)}}^{(r)}\rangle^2 .\label{eqn:nvidgsddsad}
\end{align}

Therefore, combining \eqref{eqn:vnixcnvsdf} and \eqref{eqn:nvidgsddsad}, we have 
\begin{align}
    &\sup_{\beta}\EE[R_T^{(m)}(A)]\\
    \ge& p_m \sum_{t=1}^T\EE_{D}\left[Z(d_t^{(m)})\left(1-\sqrt{{\EE_{Q_t^{(m)+}}[\langle u_t^{(m)},\beta\rangle^2]}
    u_t^{m\top} \left(p_{[M]}\sum_{\ell=1}^{t-1}x_{\ell,a_\ell^{(r)}}^{(r)}x_{\ell,a_\ell^{(r)}}^{(r)\top}\right) u_t^{m}}\right)\right]\nonumber\\
    \ge & p_m\sum_{t=1}^T\EE_{D}\left[Z(d_t^{(m)})\left(1-\sqrt{{ p_{[M]}\EE_{Q_t^{(m)+}}[\langle u_t^{(m)},\beta\rangle^2]}
    u_t^{m\top} \sum_{\ell=1}^{t-1}\left(x_{\ell,1}^{(r)}x_{\ell,1}^{(r)\top}+x_{\ell,2}^{(r)}x_{\ell,2}^{(r)\top}\right) u_t^{(m)}}\right)\right].\label{eqn:vxcivsd}
\end{align}
Taking the expectation of \eqref{eqn:vxcivsd} with respect to $\{(x_{\ell,1}^r,x_{\ell,2}^r):r\in[M]\}_{\ell=1}^{t-1}$, 
each of which is distributed i.i.d.~according to $D=\cN(0, L\cdot I_{d\times d})$,
and using that $\|u_t^{(m)}\|_2=1$ with probability one,
we have 
that the above is lower bounded by
\begin{align}
 p_m\sum_{t=1}^T\EE_{(x_{t,1}^{(m)}, x_{t,2}^{(m)})}\left[Z(d_t^{(m)})\left(1-\sqrt{{2(t-1)L}p_{[M]}\EE_{Q_t^{(m)+}}[\langle u_t^{(m)},\beta\rangle^2]
    }\right)\right],\label{eqn:vncixvniefdbgvs}
\end{align}
where the  outer expectation is only over the randomness of $(x_{t,1}^{(m)}, x_{t,2}^{(m)})$.  We next calculate $\EE_{Q_t^+}[\langle u_t,\beta\rangle^2]$ and $Z(d_t^{(m)})$. By the definition of $Q_t^+$, we have
\begin{equation}\label{eqnL:vniefwasw}
    \EE_{Q_t^{(m)+}}[\langle u_t^{(m)},\beta\rangle^2]=\frac{\EE_{Q}\left[|\langle u_t^{(m)},\beta\rangle|^3\right]}{\EE_Q\left[|\langle u_t^{(m)}, \beta\rangle|\right]}.
\end{equation}
By the symmetry of $Q$, the distribution of  $\langle u_t^{(m)},\beta\rangle$, conditioned on any  $u_t^{(m)}$, is identical to that of the first coordinate of $\beta$. 
Therefore, using Lemma \ref{lem:moment}, we have 
\begin{equation*}
    \EE_{Q}\left[|\langle u_t^{(m)},\beta\rangle|^3\right]=\Delta^3\frac{\Gamma(\frac{d}{2})\Gamma(2)}{\Gamma(\frac{d+3}{2})\Gamma(\frac{1}{2})}\quad \text{and}\quad \EE_{Q}\left[|\langle u_t^{(m)},\beta\rangle|\right]=\Delta \frac{\Gamma(\frac{d}{2})\Gamma(1)}{\Gamma(\frac{d+1}{2})\Gamma(\frac{1}{2})},
\end{equation*}
and thus it follows from \eqref{eqnL:vniefwasw} that
\begin{equation}\label{eqn:vncixvniefdbgvs2}
     \EE_{Q_t^{(m)+}}[\langle u_t^{(m)},\beta\rangle^2]=\frac{2\Delta^2}{d+1}.
\end{equation}
Similarly, we have
\begin{align}
    &\EE_{(x_{t,1}^{(m)}, x_{t,2}^{(m)})}[Z(d_t^{(m)})]= \EE_{(x_{t,1}^{(m)}, x_{t,2}^{(m)})}[\EE_Q[\langle d_t^{(m)},\beta\rangle_+]]=\frac{1}{2}\EE_{(x_{t,1}^{(m)}, x_{t,2}^{(m)})}[\EE_Q[|\langle d_t^{(m)},\beta\rangle|]]\nonumber\\
    =& \frac{1}{2}\EE_{(x_{t,1}^{(m)}, x_{t,2}^{(m)})}\left[\|{d}_t^{(m)}\|_2 \right]\frac{\Delta\Gamma(\frac{d}{2})}{\Gamma(\frac{d+1}{2})\sqrt{\pi}}=\frac{1}{2}\EE\left[\|x_{t,1}^{(m)}- x_{t,2}^{(m)}\|_2 \right]\frac{\Delta\Gamma(\frac{d}{2})}{\Gamma(\frac{d+1}{2})\sqrt{\pi}}=\Omega(\sqrt{L}\Delta).\label{eqn:vncixvniefdbgvs3}
\end{align} 
Combining \eqref{eqn:vncixvniefdbgvs}, \eqref{eqn:vncixvniefdbgvs2}, and \eqref{eqn:vncixvniefdbgvs3}, we have 
\begin{align}\label{eqn:Lnvidswns}
    \sup_{\beta}\EE[R_T^{(m)}(A)]
    \ge&\Omega(p_m\sqrt{L}\Delta T)\cdot \left(1-\sqrt{\frac{4(t-1)L
   \Delta^2}{(d+1)}p_{[M]}}\right).
\end{align}
Choosing  $\Delta=\frac{ \sqrt{(d+1)}}{4\sqrt{(T-1)L\sum_{r\in[M]}p_r}}$,
which satisfies $\Delta\le 1$ by assumption, in
\eqref{eqn:Lnvidswns}, we finally establish
\begin{equation*}
    \sup_{\beta}\EE[R_T(A)]
    =\Omega\left(p_m T\sqrt{L}\Delta\right)=\Omega\left( \sqrt{dTp_m^2/p_{[M]}}\right).
\end{equation*}

\vspace{2mm}
\noindent
\textbf{The $s$-sparse case. }
In this case, we consider $\supp(\beta^{(1)}),\dots, \supp(\beta^{(M)})$ located in the first $s$ coordinates. 
If the supports of $\{\beta^{(m)}\}_{m\in[M]}$ are known to the algorithm,
then the structure of sparse heterogeneity and the common $\beta^\star$ would be non-informative for estimating $\{\beta^{(m)}\}_{m\in[M]}$. 
Therefore, in this case, one can obtain the lower bound $\Omega(\sqrt{sTp_m})$ by simply adapting the proof for the homogeneous case with $M=1$  in $s$ dimensions.

\end{proof}

\newpage

\section{Results on Contextual Bandits under {\sf Model-P}}\label{app:model-p}

Recall that in the single contextual bandit problem under {\sf Model-P}, we have a set of $K$ parameters $\{\beta^{(a)}\}_{a\in[K]}$ where $\beta^{(a)}$ is associated with arm $a$. 
When action $a$ is chosen,
a reward $y_{t,a}=\langle x_t,\beta^{(a)}\rangle +\ep_t$ is earned. 
We extend this to a multitask scenario as follows. 
We consider $M$ bandit instances, and each bandit $ m$ is associated with $K$ arms corresponding to  parameters $\{\beta^{(m,a)}\}_{a\in[K]}\subseteq\RR^d$, and an activation probability $p_m\in[0,1]$.
At any time $t$,
each bandit $m$ is independently activated with  probability $p_m$. 
The analyst observes an independent  $d$-dimensional context $x_{t}^{(m)}$ for $m$ in the set $\cS_t$ of activated bandit instances. 
Given all observed contexts, the analyst can select action $a_t^{(m)}\in[K]$ for each activated bandit instance $m\in\cS_t$ and earn the reward via $y_{t}^{(m)}=\langle x_{t}^{(m)},\beta^{(m, a_t^{(m)})}\rangle +\ep_{t}^{(m)}\in\RR$, 
where $\ep_{t}^{(m)}$ are i.i.d.~noise random variables.

To study the proposed multitask scenario under {\sf Model-P}, we impose the following conditions, which are parallel to Conditions \ref{asp:shb}, \ref{asp:covariate-subg}, and \ref{asp:diverse}.

\begin{assumption}[\sc Sparse heterogeneity \& Boundedness]\label{asp:shb-p}
    There is $s_a$ with $0\le s_a\le d$ such that  for each action $a\in[K]$, there is an unknown global parameter
$\beta^{\star, (a)}\in\RR^d$ with $\|\beta^{(m,a)}-\beta^{\star, (a)}\|_0\le s$ for any $m\in[M]$. Furthermore,  $\|\beta^{(m,a)}\|_2\le 1$ for all $a\in[K]$ and $m\in[M]$.
\end{assumption}

\begin{assumption}[\sc Sub-Gaussianity]\label{asp:covariate-subg-p}
    For each $t\in [T]$, the marginal distribution of $x_{t}$ is $L$-sub-Gaussian.
\end{assumption}

\begin{assumption}[\sc Diverse covariate]\label{asp:diverse-p}
    There are  positive constants $\mu$ and $c_x$, such that for any $\{\beta^{(m,a)}:a\in[K]\}\subseteq\RR^d$, source vector $v\in\RR^d$, and $m\in[M]$ , it holds that $\PP(\langle x^{(m)}_{t}, v\rangle^2\one(a^\star = a)\ge \mu)\ge c_x$ where $a^\star =\argmax_{a\in[K]} \langle x_{t}^{(m)} ,\beta^{(m,a)}\rangle$ and the probability $\PP(\cdot)$ is taken over the  distribution of $x_t^{(m)}$.
\end{assumption}

\begin{remark}
    Condition \ref{asp:diverse-p} ensures sufficient exploration even with a greedy algorithm. \citep[Lemma 14]{Ren2020DynamicBL} proves that Condition \ref{asp:diverse-p} holds when $\EE[x_tx_t^\top]\succeq 2\mu I_d$ and $p(x_t)\ge \nu p(-x_t)$ for some $\nu >0$ where $p(\cdot)$ is the density of $x_t$.
\end{remark}

\subsection{Algorithm \& Regret Analysis}

\begin{algorithm}[H]
	\caption{MOALRBandit: Collaborative Bandits with MOLAR estimates under {\sf Model-P}}
	\label{alg:bc-bandit-p}
	\begin{algorithmic}
		\STATE \noindent {\bfseries Input:} Time horizon $T$, $\widehat{\beta}^{(m)}_{a,-1}=0$  for $a\in[K]$ and $m\in[M]$, 
  initial batch size $H_0$ and batch $\cH_0=[H_0]$; number of batches $Q= \lceil \log_2(T/H_0)\rceil$,
  $\vX^{(m,a)}_{q}=\emptyset$, and $Y^{(m)}_{a,q}=\emptyset$ for $a\in[K]$, $m\in[M]$, and $0\le q\le Q$
        \FOR{$q=1,\dots, Q $}
        \STATE Define batch $\cH_q = \{t:2^{q-1}H_0<t \le \min\{2^q H_0, T\}\}$
        \ENDFOR
		\FOR{$t=1,\cdots,T$}
		\FOR{each bandit in parallel}
        \vspace{1mm}
        \STATE Bandit instance $m$ is activated with probability $p_m$
        \vspace{1mm}
        \IF{$t\in\cH_q$ \textbf{and} bandit instance $m$ is activated}
        \vspace{1mm}
		\STATE Choose, breaking ties randomly $a_t^{(m)}=\arg\max_{a\in[K]}\langle x_{t}^{(m)},\widehat{\beta}_{q-1}^{(m, a)}\rangle$, and gain reward $y_{t}^{(m)}$
  \vspace{1mm}
		\STATE Augment observations $\vX_{q}^{(m,a_t^{(m)})}\,\leftarrow\,[\vX_{q}^{(m,a_t^{(m)})\top}, x_{t}^{(m)}]^\top$ and $Y_{q}^{(m,a_t^{(m)})}\,\leftarrow\,[Y_{q}^{(m,a_t^{(m)})\top},y_t^{(m)}]^\top$
  \vspace{1mm}
        \ENDIF
		\ENDFOR
  \IF{$t= 2^q H_0$, \ie, batch $\cH_q$ ends}
        \vspace{1mm}
        \STATE Let $n_{m,q}=\sum_{a\in[K]}|Y_{q}^{(m,a)}|$ and $\cC_{q}=\{m\in[M]:n_{m,q}\ge {C}_{\rm b}^\prime( \ln(M KT ) + d \ln(L/\mu))\}$ with ${C}_{\rm b}^\prime$ defined in Lemma \ref{lem:eig-lb-p} 
        \FOR{$a\in[K]$}
        \vspace{1mm}
        \STATE Call MOLAR$(\{(\vX^{(m,a)}_{q},Y^{(m,a)}_{q})\}_{m\in\cC_{q}})$ to obtain $\{\widehat{\beta}^{(m,a)}_{q}\}_{m\in \cC_{q}}$
        \vspace{1mm}
        \FOR{$m\in [M] \backslash\cC_{a,q}$}
        \vspace{1mm}
        \STATE Let $\widehat{\beta}^{(m,a)}_{q}=\widehat{\beta}^{(m,a)}_{q-1}$, $\vX_{q+1}^{(m,a)}=\vX_{q}^{(m,a)}$, and $Y_{q+1}^{(m,a)}=Y_{q}^{(m,a)}$
        \vspace{1mm}
        \ENDFOR
        \ENDFOR
        \ENDIF
  \ENDFOR 
	\end{algorithmic}
\end{algorithm}

Algorithm \ref{alg:bc-bandit-p} describes a variant of the \oursb algorithm under {\sf Model-P}. While Algorithm \ref{alg:bc-bandit-p} follows the spirit of \oursb, the difference is that it requires applying \ours to all arms with the same index across all bandits due to the nature of {\sf Model-P}. 

We  consider the following individual regret metric: given a time horizon $T\ge 1$ and a specific algorithm $A$ that produces action trajectories $\{a_t^{(m)}\}_{t\in[T],m\in[M]}$, we define for each $m\in[M]$ that
\begin{equation*}
    R_T^{(m)}(A):=\sum_{t=1}^T\max_{a\in[K]}\langle x_{t}^{(m)}, \beta^{(m,a)}-\beta^{(m,a^{(m)}_t)}\rangle\one(m\in\cS_t).
\end{equation*}

Theorem \ref{thm:bc-bandit-p} establishes a corresponding
regret upper bound under Conditions \ref{asp:freq}, \ref{asp:gsn}, and \ref{asp:shb-p}-\ref{asp:diverse-p}. To this end, we first show Lemma \ref{lem:eig-lb-p}, which guarantees that for any given $a\in[K]$ and $0\le q<Q$, the contexts $\cX_{a,q}^{(m)}$ at the end of each batch for all $m\in[M]$ has lower bounded eigenvalues with high probability. Lemma \ref{lem:eig-lb-p} is similar to \citep[Lemma 18]{Ren2020DynamicBL} in the single-bandit and $s$-sparse regime.

\begin{lemma}\label{lem:eig-lb-p}
    Under Conditions \ref{asp:covariate-subg-p} and \ref{asp:diverse-p}, there is ${C}_{\rm b}^\prime$ only depending on $c_x$, such that for any $0\le q<Q$, it holds with probability at least $1-1/T$ that $\lambda_{\min }(\vX_{q}^{(m,a)\top}\vX_{q}^{(m,a)})\ge n_{m,q}\mu c_x/4$ for all $a\in[K]$ and $m\in [M]$ with $n_{m,q}\ge {C}_{\rm b}^\prime( \ln(M KT ) + d \ln(L/\mu))$.
\end{lemma}
\begin{proof}

The proof is similar to the proof of   Lemma \ref{lem:eig-lb}.
Hence we only sketch the key steps below.
For $0\le q< Q$ and $m\in[M]$, we let $\cT_{q}^{(m)}$  be the set of times when contexts $\vX_{q}^{m}$ are observed at instance $m$. 
Clearly, we have $|\cT_{q}^{(m)}|=n_{m,q}$;
and $\{x_{t}^{(m)}: t\in\cT_{q}^{(m)}\}$ are independent, conditioned on $\{\widehat{\beta}^{(m,a)}_{q-1}\}_{a\in[K]}$. The following analysis is conditional on $\{\cT_{q}^{(m)}\}_{m\in [M]}$ and therefore also on $\{n_{m,q}\}_{m\in[M]}$. 

By definition, we have for any $a\in[K]$ and $m\in[M]$,
\begin{equation*}
\vX_{q}^{(m,a)\top}\vX_{q}^{(m,a)}=\sum_{t\in\cT_q^{(m)}}x_t^{(m)}x_t^{(m)\top}\one(a_t^{(m)}=a).
\end{equation*}
We first give an upper bound for $\lambda_{\max}(\vX_{q}^{(m,a)\top}\vX_{q}^{(m,a)}/n_{m,q})$. 
For any source vector $v \in\RR^d$,  any $t\in[T]$, $a\in[K]$, and $m\in[M]$, let $Z_{t,a}^{(m)}=\langle v, x_{t}\rangle^2\one(a_t^{(m)}=a)$. 
Conditionally on $\{\widehat{\beta}^{(m,a)}_{q-1}\}_{a\in[K]}$, for any $\delta>0$ and $\lambda>0$, we have 
\begin{align}
    \PP\left(\sum_{t\in\cT_q^{(m)}}Z_{t,a}^{(m)}\ge n_{m,q}\delta\mid  \{\widehat{\beta}^{(m,a)}_{q-1}\}_{a\in[K]}\right)\le &e^{-\lambda n_{m,q} \delta}\EE\left[\exp\left(\lambda \sum_{t\in\cT_q^{(m)}}Z_{t,a}^{(m)}\right)\mid \{\widehat{\beta}^{(m,a)}_{q-1}\}_{a\in[K]}\right]\nonumber\\
    =&e^{-\lambda n_{m,q}\delta}\prod_{t\in\cT_q^{(m)}}\EE\left[\exp\left(\lambda Z_{t,a}^{(m)}\right)\mid \{\widehat{\beta}^{(m,a)}_{q-1}\}_{a\in[K]}\right]\nonumber.
\end{align}
Since $x_{t}$ is assumed to be $L$-sub-Gaussian  and $\|v\|_2=1$, 
$Z_{t,a}^{(m)}=\langle v, x_{t}\rangle^2\one(a_t^{(m)}=a)\le \langle v, x_{t}\rangle^2$ is $(4\sqrt{2}L,4L)$-sub-exponential \citep{vershynin2018high}.
Following the argument in  \eqref{eqn:nvisdngs} and \eqref{eqn:vndifgsd}, we obtain
\begin{align}
   &\PP\left(v^\top(\vX_{q}^{(m,a)\top}\vX_{q}^{(m,a)}/n_{m,q})v\ge \delta  \right)=\PP\left(\sum_{t\in\cT_{q}^{(m)}}Z_{t,a}^{(m)}\ge n_{m,q}\delta L \right)\nonumber\\
   \le&\exp\left(-\min\left\{\frac{(\delta-2L(\ln(2)+1))^2}{64L}, \frac{\delta-2L(\ln(2)+1)}{32 L}\right\}n_{m,q}\right)\triangleq p_{q,\delta}^{(m)}\nonumber
\end{align}
for any $\delta \ge 2L(\ln(2)+1)$. Now, following the $\ep$-net-arguments around \eqref{eqn:nmgodsfsd}, we have  with probability at least $1-(1+2/\ep)^d\sum_{m\in\cC_{q}}p_{q,\delta}^{(m)}$ that
\begin{equation}\label{eqn:nmgodsfsd-p}
\lambda_{\max}\left(\vX_{q}^{(m,a)\top}\vX_{q}^{(m,a)}/n_{m,q}\right)\le \frac{\delta }{1-2\ep}.
\end{equation}

Next we bound $\lambda_{\min}(\vX_{q}^{(m,a)\top}\vX_{q}^{(m,a)}/n_{m,q})$.
For any source vector $v\in\RR^d$, we have
\begin{align*}
   &v^\top(\vX_{q}^{(m,a)\top}\vX_{q}^{(m,a)}/n_{m,q}) v= v^\top \left(\frac{1}{n_{m,q}}\sum_{t\in\cT_{q}^{(m)}}x_{t}^{m}x_{t}^{m\top}\one(a_t^{(m)}=a)\right)v \\
   \ge& \frac{\mu}{n_{m,q}}\sum_{t\in\cT_{q}^{(m)}}\one(\langle v, x_{t}^{(m)} \rangle^2 \ge \mu)\one(a_t^{(m)}=a).
\end{align*}
Since $a_t^{(m)} = \arg\max_{a\in[K]}\langle \widehat{\beta}_{q-1}^{(m, a)} ,x_{t}^{(m)}\rangle  $ for any $t\in\cT_q^{(m)}$ and $m\in[M]$, by Condition \ref{asp:diverse-p}, we have 
\begin{equation*}
    \EE[\one(\langle v, x_{t}^{(m)} \rangle^2 \ge \mu)\one(a_t^{(m)}=a)]=\PP(\langle x^{(m)}_{t}, v\rangle^2\one(a_t^{(m)} = a)\ge \mu)\ge c_x.
\end{equation*}
Therefore, by applying a Chernoff bound, we have
\begin{equation}\label{eqn:jvisdngsd2-p}
    \PP\left(v^\top(\vX_{q}^{(m,a)\top}\vX_{q}^{(m,a)}/n_{m,q}) v \le  {\mu c_x}/{2}\right)\le  e^{-c_xn_{m,q}/8}\triangleq \tilde{p}_{q}^{(m)}.
\end{equation}
Then, using an $\ep$-net argument and applying a union bound to \eqref{eqn:jvisdngsd2-p}, we have  with probability at least $1-(1+2/\ep)^d\sum_{m\in\cC_q}\tilde{p}_{q}^{(m)}$ that $v^\top(\vX_{q}^{(m,a)\top}\vX_{q}^{(m,a)}/n_{m,q})v\ge {\mu c_x}/{2} $ holds for any $v\in\cN_d(\ep)$ and $m\in\cC_q$.
Therefore, following the argument around \eqref{eqnLvndivcxb}, we obtain  
\begin{equation*}\nonumber
\lambda_{\min}\left(\vX_{q}^{(m,a)\top}\vX_{q}^{(m,a)}/n_{m,q}\right)\ge \frac{\mu c_x}{2} - \frac{2\ep \delta}{1-2\ep}.
\end{equation*}
Finally, letting $\delta = 32\max\{L,\sqrt{L}\}$ and $\ep = {\mu c_x}/{(8 \delta+2\mu c_x)}$, we have $2\ep \delta/(1-2\ep)= \mu c_x/4$ and $p_{q,\delta}^{(m)} \le e^{-n_{a,q}^{(m)}}\le e^{-c_xn_{m,q}/8}=\tilde{p}_{q}^{(m)}$. Therefore,
$\lambda_{\min}\left(\vX_{q}^{(m,a)\top}\vX_{q}^{(m,a)}/n_{m,q}\right)\ge {\mu c_x}/{4}$ holds for all $a\in[K]$ and $m\in\cC_q$ with probability at least 
\begin{align*}
    &1-(1+2/\ep)^d K\sum_{m\in\cC_q}(\tilde{p}_{q}^{(m)}+p_{q,\delta}^{(m)})\ge  1- 2(1+2/\ep)^dK\sum_{m\in\cC_q}e^{-c_xn_{m,q}/8}\\
    \ge &1- \exp\left(-\frac{c_x\min_{m\in\cC_q}n_{m,q}}{8}+\ln(2MK)+d\ln\left(5+512\max\{L,\sqrt{L}\}/(\mu c_x)\right)\right).
\end{align*}
In particular, there is ${C}_{\rm b}^\prime\ge 3$ depending only on $c_x$, such  that when $\min_{m \in \cC_q}n_{a,q}^{(m)}\ge {C}_{\rm b}^\prime(\ln(MKT)+d\ln(L/\mu))$, the probability is lower bounded by $1-1/T$. 
\end{proof}

Given 
Lemma \ref{lem:eig-lb-p},
using Theorem \ref{thm:ls-fixed-design}, we can bound the $\ell_2$ estimation error $\EE[\max_{a\in[K]}\|\widehat{\beta}_t^{(m,a)}-\beta^{(m,a)}\|_2]$ for all $m\in\cC_q$ at the end of  batch $\cH_q$ as follows.
\begin{lemma}\label{lem:bandit-est-p}
    Under Conditions \ref{asp:freq}, \ref{asp:gsn}, and \ref{asp:shb-p}-\ref{asp:diverse-p}, for any $0\le q<Q$, letting $\tau =\arg\min_{m\in[M]}(p^{1}\vee \sum_{\ell \in [M]} p^\ell /m))/p^{m}$, if $|\cH_q|\ge 2{C}_{\rm b}^\prime( \ln(M KT ) + d \ln(L/\mu))/p_{\tau}$ with $C_{\rm b}^\prime$ defined in Lemma  \ref{lem:eig-lb-p}, it holds with probability at least $1-2/T$ 
    that for all $a\in[K]$ and $q\in\cC_q$,
    \begin{equation*}
        \EE[\max_{a\in[K]}\|\widehat{\beta}^{(m,a)}_{q}-\beta^{(m,a)}\|_2^2 \mid (\vX^{(m,a)}_{q},Y^{(m,a)}_{q})_{a\in[K], m\in\cC_q}]=\widetilde{O}\left(\frac{1}{\mu |\cH_q|}\left(\frac{s}{p_m}+\frac{d}{p_{[M]}}\right)\right),
    \end{equation*}
    where the expectation is taken with respect to the randomness of the noise, and
    logarithmic factors as well as quantities depending only on $c_x$, $c_{\rm f}$  are absorbed into $\widetilde{O}(
    \cdot)$.
\end{lemma}
\begin{proof}
The proof is essentially the same as the proof of Lemma \ref{lem:bandit-est}. We thus omit the proof.
\end{proof}

Based on Lemma \ref{lem:bandit-est-p}, we can bound the individual regret as follows.
\begin{theorem}\label{thm:bc-bandit-p}
Under Conditions  \ref{asp:freq}, \ref{asp:gsn}, and \ref{asp:shb-p}-\ref{asp:diverse-p}, 
for any $T\ge 1$ and $1\le H_0\le d$,
the expected regret of \oursb under {\sf Model-P}, for any $T\ge 1$, is bounded as
\begin{equation*}
    \EE[R_T^{(m)}]=
   \widetilde{O}\left( d\wedge( Tp_m)+
    \sqrt{\left(s+\frac{dp_m}{p_{[M]}}\right)Tp_m}\right),
\end{equation*}
where logarithmic factors as well as quantities depending only on $c_x$, $c_{\rm f}$  are absorbed into $\widetilde{O}(
    \cdot)$.
\end{theorem}

\begin{proof}
For any $t\in\cH_q$, $0\le q\le Q$, and $m\in [M]$, we have 
for $m\in\cS_t$,
\begin{align}
    &\max_{a\in[K]}\langle x_{t}^{(m)},\beta^{(m,a)}-\beta^{(m,a_t^{(m)})}\rangle \le \max_{a, a^\prime\in[K]}\langle x_{t}^{(m)},\beta^{(m,a)}-\beta^{(m,a^\prime)}\rangle \le  2\max_{a\in[K]}|\langle x_{t}^{(m)}, \beta^{(m,a)}\rangle|.\label{eqn:nvicnvidffs1}
\end{align}
Also, from the definition of $a_t^{(m)}$, we have $\langle x_t^{(m)},\widehat{\beta}_{q-1}^{(m, a)}\rangle\le \langle x_t^{(m)},\widehat{\beta}_{q-1}^{(m,a_t^{(m)}}\rangle$ for any $a\in[K]$. Therefore,
the instantaneous regret can be bounded, for $m\in\cS_t$, as 
\begin{align}
    &\max_{a\in[K]}\langle x_{t}^{(m)},\beta^{(m,a)}-\beta^{(m,a_t^{(m)})}\rangle =  \max_{a\in[K]}\langle x_{t}^{(m)},\beta^{(m,a)}-\widehat{\beta}^{(m, a)}_{q-1}+\widehat{\beta}^{(m, a)}_{q-1}-\widehat{\beta}^{(m, a_t^{(m)})}_{q-1}\rangle \nonumber\\
    \le &\max_{a\in[K]}\langle x_{t}^{(m)},\beta^{(m,a)}-\widehat{\beta}^{(m, a)}_{q-1}+\widehat{\beta}^{(m, a_t^{(m)})}_{q-1}-\widehat{\beta}^{(m, a_t^{(m)})}_{q-1}\rangle \le  2\max_{a\in[K]}|\langle x_{t}^{(m)}, \beta^{(m,a)}-\widehat{\beta}^{(m, a)}_{q-1}\rangle|.\label{eqn:nvicnvidffs2}
\end{align}
By combining \eqref{eqn:nvicnvidffs1}, \eqref{eqn:nvicnvidffs2} and further using Condition \ref{asp:covariate-subg-p} and Lemma \ref{lem:max-ineq},   we obtain
\begin{align}
     &\EE\left[ \max_{a\in[K]}\langle x_{t}^{(m)},\beta^{(m,a)}-\beta^{(m,a_t^{(m)})}\rangle \one(m\in\cS_t)\mid \{\widehat{\beta}^{(m, a)}_{q-1}\}_{a\in[K]}\right]\nonumber\\
    \le &2\max_{a\in[K]}\left(\|\beta^{(m,a)}\|_2\wedge\|\beta^{(m,a)}-\widehat{\beta}^{(m, a)}_{q-1}\|_2\right)\sqrt{2\ln(2K) L}\PP(m\in\cS_t)\nonumber\\
    \le & 2p_m\max_{a\in[K]}\left(\|\beta^{(m,a)}\|_2\wedge\|\beta^{(m,a)}-\widehat{\beta}^{(m, a)}_{q-1}\|_2\right)\sqrt{2\ln(2K) L}.\label{eqn:vniscnvxd-p}
\end{align}
Taking expectations of \eqref{eqn:vniscnvxd-p} with respect to $\{\widehat{\beta}_{q-1}^{(m, a)}\}_{a \in [K]}$, we find 
\begin{equation}\label{eqn:vncsvnsLdsad-p}
     \EE\left[\max_{a\in[K]}\langle x_{t}^{(m)},\beta^{(m,a)}-\beta^{(m,a_t^{(m)})}\rangle \one(m\in\cS_t)\right]\le 2\sqrt{2\ln(2K)L}p_m\EE[\max_{a\in[K]}1\wedge\|\widehat{\beta}^{(m, a)}_{q-1}-\beta^{(m,a)}\|_2].
\end{equation}
Given \eqref{eqn:vncsvnsLdsad-p}, it remains to bound the estimation error $\EE[\max_{a\in[K]}\|\widehat{\beta}^{(m, a)}_{q-1}-\beta^{(m,a)}\|_2]$ for all $0\le q\le Q$ and $m\in[M]$. Therefore, the rest follows the argument in the proof of Theorem \ref{thm:bc-bandit} and uses Lemma \ref{lem:bandit-est-p}.

\end{proof}

\subsection{Lower Bound}

We also establish the regret lower bound under {\sf Model-P} as follows.  
\begin{theorem}\label{thm:lb-bandit-p}
Given any $1\le s\le d$ and $\{p_m\}_{m\in[M]}\subseteq[0,1]$, for any $m\in[M]$, when $T\ge \max\{(d+1)/p_{[M]}, (s+1)/p_m\}/(16L)+1$,  there are  $\{\beta^{(m,a)}\}_{a\in[K],m\in[M]}$ satisfying Condition \ref{asp:shb-p} and distributions of contexts satisfying Condition
\ref{asp:covariate-subg-p} and \ref{asp:diverse-p}, 
such that  for any online Algorithm $A$,
\begin{equation*}
    \EE[R_T^{(m)}(A)]=\Omega\left(\sqrt{\left(s+\frac{dp_m}{p_{[M]}}\right)Tp_m}\right).
\end{equation*}
\end{theorem}
\begin{proof}
We consider the two-armed case where $K=2$ and $\beta^{(m,2)}=-\beta^{(m,1)}$ for all $m\in[M]$.
We  prove $\Omega(\sqrt{sTp_m})$ and $\Omega(\sqrt{dTp_m/p_{[M]}})$ by considering two cases: the \emph{homogeneous} case where $\beta^{(1,1)}=\cdots =\beta^{(M,1)}=\beta^{\star(1)}$ and the \emph{$s$-sparse} case where $\beta^{\star(1)} =0$ and $\|\beta^{(m,1)}\|_0\le s$ for all $m\in[M]$.

\vspace{2mm}
\noindent
\textbf{The homogeneous case. }  Since $\beta^{(1,1)}=\cdots =\beta^{(M,1)}=\beta^{\star(1)}$  in this case, for simplicity, we omit the superscript $m$ and use $\beta_a$ to denote $\beta^{(m,a)}$ for any $k\in[K]$. 
Denote by $Q$ the uniform distribution over $\{\beta \in \RR^d:\|\beta\|_2= \Delta\}$ where $\Delta\in[0,1]$ will be determined below. 
We let $Q$ be the prior distribution of $\beta_1$, 
and let $D\triangleq\cN(0, L I_{d})$ be the distribution of contexts. 
Let $\cS_t$ be the set of activated bandits at the $t$-th round, 
and denote by $P_{\beta_1,x,t}$ the distribution of the observed rewards
$\{y_\ell^{(m)} \one(m \in \cS_t) : m\in[M]\}_{\ell=1}^{t}$ 
up to time $t$, 
conditioned on  $\beta_1$ and the contexts $\{x_{\ell}^{(m)}:a\in[2], m\in[M]\}_{\ell=1}^{t}$. 
Since the event $\{m\in\cS_t\}$ is independent of the history and of the contexts $\{x_{t,a}^{(m)}:a\in[2]\}$ at the current round, 
we have
\begin{align}\label{eqn:vgnidsgs-p}
    \sup_{\beta_1}\EE[R_T^{(m)}(A)]\ge& \EE_{\beta_1 \sim Q}[R_T^{(m)}(A)]=\sum_{t=1}^T\EE_Q\left[\EE_{D}\left[\EE_{P_{\beta_1,x,t-1}}\left[p_m\max_{a\in[2]}\langle x_{t}^{(m)},\beta_a-\beta_{a_t^{(m)}}\rangle  \right]\right]\right]
\end{align}
where the factor of $p_m$  in \eqref{eqn:vgnidsgs} is due to the integration over the randomness of $\{m\in\cS_t\}$.
Since $\beta_2=-\beta_1$, we have 
\begin{equation*}
    \max_{a\in[2]}\langle x_{t}^{(m)},\beta_a-\beta_{a_t^{(m)}}\rangle=2\one(a_t^{(m)}=1)\langle x_t^{(m)},\beta_1\rangle_-+2\one(a_t^{(m)}=2)\langle x_t^{(m)},\beta_1\rangle_+
\end{equation*}
where the subscripts $+$ and $-$ denote the positive and negative parts, respectively. 
Therefore, from \eqref{eqn:vgnidsgs-p}, we have that $\sup_{\beta}\EE[R_T^{(m)}(A)]$ is lower bounded by
\begin{align}
2p_m\sum_{t=1}^T\EE_Q\left[\EE_{D}\left[\EE_{P_{\beta_1,x,t-1}}\left[\one(a_t^{(m)}=1)\langle x_t^{(m)},\beta_1\rangle_-+\one(a_t^{(m)}=2)\langle x_t^{(m)},\beta_1\rangle_+\right]\right]\right].\label{eqn:videvbvcx-p}
\end{align}
For any $1\le t\le T$, conditionally on $x_t^{(m)}$, we define two  measures $Q_t^{(m)+}$ and $Q_t^{(m)-}$ via the Radon–Nikodym derivatives $\d Q_t^{(m)+}(\beta_1)=\langle x_t^{(m)},\beta_1\rangle_+/Z(x_t^{(m)})\d Q$ and $\d Q_t^{(m)-}(\beta_1) $ $=$ $\langle x_t^{(m)},\beta_1\rangle_-/Z(x_t^{(m)})$ $\d Q$, 
where  $Z(x_t^{(m)})=\EE_Q[\langle x_t^{(m)},\beta_1\rangle_+]=\EE_Q[\langle x_t^{(m)},\beta_1\rangle_-]$ is a normalization factor. 
Here $\EE_Q[\langle d_t^{(m)},\beta_1\rangle_+]=\EE_Q[\langle d_t^{(m)},\beta_1\rangle_-]$ due to the symmetry of $Q$.
Plugging the definitions of $Q_t^{(m)+}$ and $Q_t^{(m)-}$ into \eqref{eqn:videvbvcx-p}, and changing the integration order, we have
\begin{align}
    \sup_{\beta}\EE[R_T^{(m)}(A)]\ge&2p_m\sum_{t=1}^T\EE_{D}\left[Z(x_t^{(m)})\left(\EE_{P_{\beta_1,x,t-1}\circ Q_t^{(m)-}}[\one(a_t=1)]+\EE_{P_{\beta_1,x,t-1}\circ Q_t^{(m)+}}[\one(a_t=2)]\right)\right]\nonumber\\
    \ge&2p_m\sum_{t=1}^T\EE_{D}\left[Z(x_t^{(m)})\left(1-\texttt{TV}(P_{\beta_1,x,t-1}\circ Q_t^{(m)-},P_{\beta_1,x,t-1}\circ Q_t^{(m)+})\right)\right]\nonumber\\
    \ge& 2p_m\sum_{t=1}^T\EE_{D}\left[Z(x_t^{(m)})\left(1-\sqrt{\frac{1}{2}D_\mathrm{KL}(P_{\beta_1,x,t-1}\circ Q_t^{(m)+}\,\|\,P_{\beta_1,x,t-1}\circ Q_t^{(m)-})}\right)\right]\label{eqn:nvicdbve-p}
\end{align}
where the second inequality follows the definition of the total variation distance: $P_1(A)+P_2(A^c)\ge 1- \texttt{TV}(P_1,P_2)$ with $A=\{a_t=1\}$,  and the third inequality follows from Pinsker's inequality  $\texttt{TV}(P_1,P_2)\le\sqrt{D_\mathrm{KL}(P_1\,\|\,P_2)/2}$. 

Due to the distribution of $x_t^{(m)}$, $x_t^{(m)}\neq 0$ with probability one; 
hence we can set $u_t^{(m)}=x_t^{(m)}/\|x_t^{(m)}\|_2$ and the zero probability event where  $x_t^{(m)}=0$ does not affect the result.
We also let $
    \tilde{\beta}_{1,t}^{(m)} = \beta_1-2\langle u_t^{(m)},\beta_1\rangle u_t^{(m)} $,
and we have \begin{equation}\label{eqn:nbvindgde-inv-p}
        \beta_1 = \tilde{\beta}_{1,t}^{(m)} -2\langle u_t^{(m)},\tilde{\beta}_{1,t}^{(m)} \rangle u_t^{(m)}.
    \end{equation}
Since $\langle x_t^{(m)},\beta_1 \rangle_-=\langle x_t^{(m)},\tilde{\beta}_{1,t}^{(m)} \rangle_+$ and $Q$ is reflection-invariant, one can equivalently obtain $\beta_1\sim Q_t^{(m)-}$ by first generating  $\tilde{\beta}_{1,t}^{(m)}\sim Q_t^{(m)+}$ and then calculating $\beta_1$ via \eqref{eqn:nbvindgde-inv-p}. It thus holds that
\begin{align*}
    P_{\beta_1,x,t-1}\circ (\beta_1 \sim Q_t^{(m)-})
    =&P_{\tilde{\beta}_{1,t}^{(m)} -2\langle u_t^{(m)},\tilde{\beta}_{1,t}^{(m)}\rangle u_t^{(m)} ,x,t-1}\circ(\tilde{\beta}_{1,t}^{(m)} \sim Q_t^{(m)+})\\
    \overset{d}{=}&P_{\beta_1-2\langle u_t^{(m)},\beta_1 \rangle u_t^{(m)} ,x,t-1}\circ(\beta_1 \sim Q_t^{(m)+}).
\end{align*}
Following the above argument and \eqref{eqn:nvicdbve-p}, we have
\begin{align*}
    2 p_m \sum_{t=1}^T\EE_{D}\left[Z(x_t^{(m)})\left(1-\sqrt{\frac{1}{2}D_\mathrm{KL}(P_{\beta_1,x,t-1}\circ Q_t^{(m)+}\,\|\,P_{\beta_1-2\langle u_t,\beta_1\rangle u_t ,x,t-1} \circ Q_t^{(m)+}}\right)\right].
\end{align*}
By Lemma \ref{lem:convexity}, $  \sup_{\beta}\EE[R_T^{(m)}(A)]$ is further lower bounded by 
\begin{align}
    &\sup_{\beta}\EE[R_T(A)]\nonumber\\
    \ge& p_m \sum_{t=1}^T\EE_{D}\left[Z(x_t^{(m)})\left(1-\sqrt{\frac{1}{2}\EE_{\beta_1 \sim Q_t^{(m)+}}\left[D_\mathrm{KL}(P_{\beta_1,x,t-1}\,\|\,P_{\beta_1-2\langle u_t^{(m)},\beta_1 \rangle u_t^{(m)} ,x,t-1})\right]}\right)\right]. \label{eqn:vnixcnvsdf-p}
\end{align}
Since the reward noise follows the distribution $\cN(0,1)$, conditioned on the activation sets $\{\cS_\ell\}_{\ell=1}^{t-1}$, we have $P_{\beta_1,x,t-1}=\otimes_{\ell=1}^{t-1}\otimes_{r \in \cS_\ell}\cN(\langle \beta_{a_\ell^{(r)}},x_{\ell}^{(r)}\rangle, 1)$. Furthermore, by the formula for the divergence between two Gaussian distributions, we have 
\begin{align}
    &D_\mathrm{KL}\left(P_{\beta_1,x,t-1}\,\|\,P_{\beta_1-2\langle u_t^{(m)},\beta \rangle u_t^{(m)} ,x,t-1} \mid \{\cS_\ell\}_{\ell=1}^{t-1}\right)\nonumber\\
    =&\frac{1}{2}\sum_{\ell=1}^{t-1}\sum_{r \in \cS_\ell}\left(\langle \beta_{a_\ell^{(r)}},x_{\ell, }^r\rangle-\langle \beta_{a_\ell^{(r)}}-2\langle u_t^{(m)},\beta_{a_{\ell}^{(r)}}\rangle u_t^{(m)},x_{\ell}^{(r)}\rangle\right)^2\nonumber\\
    =&{2\langle u_t^{(m)},\beta_1\rangle^2}\sum_{\ell=1}^{t-1}\sum_{r\in \cS_{\ell}}\langle u_t^{(m)},x_{\ell}^{(r)}\rangle^2
    ={2\langle u_t^{(m)},\beta_1\rangle^2}\sum_{\ell=1}^{t-1}\sum_{r\in[M]}\langle u_t^{(m)},x_{\ell}^{(r)}\rangle^2 \one(r\in \cS_{\ell}).\label{eqn:nvidgsd-p}
\end{align}
Since the events $\{r\in \cS_{\ell}\}$ for $\ell\in[t-1]$ and $r\in[M]$ are independent of the contexts and of $\beta_1$, using \eqref{eqn:nvidgsd-p} and Lemma \ref{lem:convexity}, we obtain 
\begin{align}
    &D_\mathrm{KL}(P_{\beta_1,x,t-1}\,\|\,P_{\beta_1-2\langle u_t^{(m)},\beta_1 \rangle u_t^{(m)} ,x,t-1})\nonumber\\
    \le&\EE_{\{\cS_\ell\}_{\ell=1}^{t-1}}\left[D_\mathrm{KL}\left(P_{\beta_1,x,t-1}\,\|\,P_{\beta_1-2\langle u_t^{(m)},\beta_1 \rangle u_t^{(m)} ,x,t-1} \mid \{\cS_\ell\}_{\ell=1}^{t-1}\right)\right] \nonumber\\
    =&{2\langle u_t^{(m)},\beta_1\rangle^2}\sum_{\ell=1}^{t-1}\sum_{r\in[M]}\langle u_t^{(m)},x_{\ell}^{(r)}\rangle^2 \PP(r\in \cS_{\ell})
    ={2\langle u_t^{(m)},\beta_1\rangle^2}p_{[M]}\sum_{\ell=1}^{t-1}\langle u_t^{(m)},x_{\ell}^{(r)}\rangle^2 .\label{eqn:nvidgsddsad-p}
\end{align}

Therefore, combining \eqref{eqn:vnixcnvsdf-p} and \eqref{eqn:nvidgsddsad-p}, we have 
\begin{align}
    &\sup_{\beta}\EE[R_T^{(m)}(A)]\nonumber\\
    \ge& 2p_m \sum_{t=1}^T\EE_{D}\left[Z(x_t^{(m)})\left(1-\sqrt{{\EE_{Q_t^{(m)+}}[\langle u_t^{(m)},\beta_1\rangle^2]}
    u_t^{m\top} \left(p_{[M]}\sum_{\ell=1}^{t-1}x_{\ell}^{(r)} x_{\ell}^{r\top}\right) u_t^{m}}\right)\right].\label{eqn:vxcivsd-p}
\end{align}
Taking expectations in \eqref{eqn:vxcivsd-p} with respect to $\{x_{\ell}^{(r)}:r\in[M]\}_{\ell=1}^{t-1}$, 
 $  \sup_{\beta}\EE[R_T^{(m)}(A)]$ is lower bounded by 
\begin{align}
    2p_m\sum_{t=1}^T\EE_{x_{t}^{(m)}}\left[Z(x_t^{(m)})\left(1-\sqrt{{2(t-1)L}p_{[M]}\EE_{Q_t^{(m)+}}[\langle u_t^{(m)},\beta_1\rangle^2]
    }\right)\right],\label{eqn:vncixvniefdbgvs-p}
\end{align}
where the  outer expectation is only over the randomness of $x_{t}^{(m)}$.  We next calculate $\EE_{Q_t^+}[\langle u_t,\beta_1\rangle^2]$ and $Z(x_t^{(m)})$.
By \eqref{eqn:vncixvniefdbgvs2}, we have
\begin{equation}\label{eqn:vncixvniefdbgvs2-p}
     \EE_{Q_t^{(m)+}}[\langle u_t^{(m)},\beta_1\rangle^2]=\frac{2\Delta^2}{d+1}.
\end{equation}
We also have
\begin{align}
    &\EE_{x_{t}^{(m)}}[Z(x_t^{(m)})]= \EE_{x_{t}^{(m)}}[\EE_Q[\langle x_t^{(m)},\beta_1\rangle_+]]=\frac{1}{2}\EE_{x_{t}^{(m)}}[\EE_Q[|\langle u_t^{(m)},\beta\rangle|]]\nonumber\\
    =& \frac{1}{2}\EE_{x_{t}^{(m)}}\left[\|x_t^{(m)}\|_2 \right]\frac{\Delta\Gamma(\frac{d}{2})}{\Gamma(\frac{d+1}{2})\sqrt{\pi}}=\Omega(\sqrt{L}\Delta).\label{eqn:vncixvniefdbgvs3-p}
\end{align} 
Combining \eqref{eqn:vncixvniefdbgvs-p}, \eqref{eqn:vncixvniefdbgvs2-p}, and \eqref{eqn:vncixvniefdbgvs3-p}, we have 
\begin{align}\label{eqn:Lnvidswns-p}
    \sup_{\beta}\EE[R_T^{(m)}(A)]
    \ge&\Omega(p_m\sqrt{L}\Delta)\sum_{t=1}^T\left(1-\sqrt{\frac{4(t-1)L
   \Delta^2}{(d+1)}p_{[M]}}\right).
\end{align}
Plugging  $\Delta=\frac{\sqrt{(d+1)}}{4\sqrt{(T-1)L\sum_{r\in[M]}p_r}}\le 1$ into \eqref{eqn:Lnvidswns-p}, we finally establish
\begin{equation*}
    \sup_{\beta}\EE[R_T(A)]
    =\Omega\left(p_m T\sqrt{L}\Delta\right)=\Omega\left(\sqrt{dTp_m^2/p_{[M]}}\right).
\end{equation*}

\vspace{2mm}
\noindent
\textbf{The $s$-sparse case. }
In this case, we consider $\supp(\beta^{(1,1)}),\dots, \supp(\beta^{(M,1)})$ located in the first $s$ coordinates. 
If the supports of $\{\beta_1^{(m)}\}_{m\in[M]}$ are known by the analyst, then the structure of sparse heterogeneity and the common $\beta^{\star(1)}$ is non-informative for estimating $\{\beta^{(m,1)}\}_{m\in[M]}$. 
Therefore, in this case, we can obtain the lower bound $\Omega(\sqrt{sTp_m})$ by simply adapting the proof for the homogeneous case with $M=1$ in $s$ dimensions.

\end{proof}

\clearpage

\section{Experimental Details}\label{app:experi}

\subsection{Synthetic Experimental Details}\label{sec:experi-detail}
To ensure a fair comparison of methods, 
we set the regularization parameter $\lambda_m$ of LASSO/ LASSOB, RM/RMB, and TNB, and the threshold parameter $\gamma_m$ of \ours/\oursb guided by theoretical results to obtain optimal rates with respect to $n,d$, and $M$ \citep{xu2021multitask,bastani2020online,cella2022multi}. Our simulation setup follows conventions from the most closely related statistical literature \citep{chen2021statistical,Chan2017TheMB,Camerlenghi2019NonparametricBM,xu2021multitask,bastani2020online}. 

Specifically, we set $\lambda_m = c_\lambda \sqrt{\ln(d)/n_m}$ for LASSO/LASSOB and RM/RMB,  $\gamma_m=c_\gamma \ln( (n_{[M]}/n_m)\wedge d)/n_m)$ for \ours/\oursb, $\lambda_m= c_\lambda\sqrt{(M+d)/n_m}$ for TNB.
We only tune the numerical coefficient $c_\lambda$ and $c_\gamma$ on a pre-specified grid $\{0.05\sigma, 0.35\sigma, 0.7\sigma, \sigma, 2\sigma\}$, where $\sigma$ is the standard deviation of noise. 
We tune these numerical coefficients to lead to the best $\ell_1$ estimation errors on independently generated data with $n=5,000$.

Note that $\sigma=0.1$ in our offline linear regression setup. 
After tuning, we take $c_\lambda=0.005$ for LASSO; we take $c_\lambda=0.035$ for RM and set the trimming-related parameters to the default values 
$\zeta =0.1$, $\eta=0.1$ suggested by \cite{xu2021multitask}; we take $c_\gamma=0.1$ for \ours with the option of hard thresholding, respectively.

In contextual bandits, the noise scale is set as $\sigma = 0.5$.
We initialize the first batch size $|H_0|=1$ to use data efficiently.
In the bandit case, the parameters associated with the reported results are $c_\lambda=0.025$ for LASSOB, $c_\lambda =1 $ for TNB, $(\zeta,\eta,c_\lambda)=(0.1,0.1,0.175)$ for RMB, and $c_\gamma=0.5$  for \oursb with the option of hard thresholding, respectively.

\subsection{PISA Experimental Details}\label{app:pisa_more}
The PISA2012  dataset \citep{oecd2019teaching} consists of $485,490$ student records collected from $68$ countries\footnote{It is accessible at \url{https://www.oecd.org/pisa/data/pisa2012database-downloadabledata.htm} (official website) and \url{https://s3.amazonaws.com/udacity-hosted-downloads/ud507/pisa2012.csv.zip} (an exterior csv format).}. 
However, records associated with many of these countries have more than half the data missing and contain constant features.
Thus, we restrict our experiment to the $M=15$ countries with the largest sample sizes. 
The sample sizes in these countries range from $7,038$ to $33,806$, while 
the data contains about $500$ features. 

Since many features are highly correlated or are constant across records, to avoid ill-conditioning,
we pre-process the data as follows.
We
create dummy variables to indicate missing values and then fill missing values with zeroes.
We apply the LASSO globally to select features, with the regularization hyperparameter selected automatically via $10$-fold cross-validation. 
We then filter out features with pairwise correlations higher than $0.6$ among the selected ones, doing this sequentially in the order given by the PISA dictionary. 
This finally leaves us with $57$ features.

\setcounter{figure}{5}
\begin{figure}[H]
    \centering
    \includegraphics[height = 0.25\textwidth]{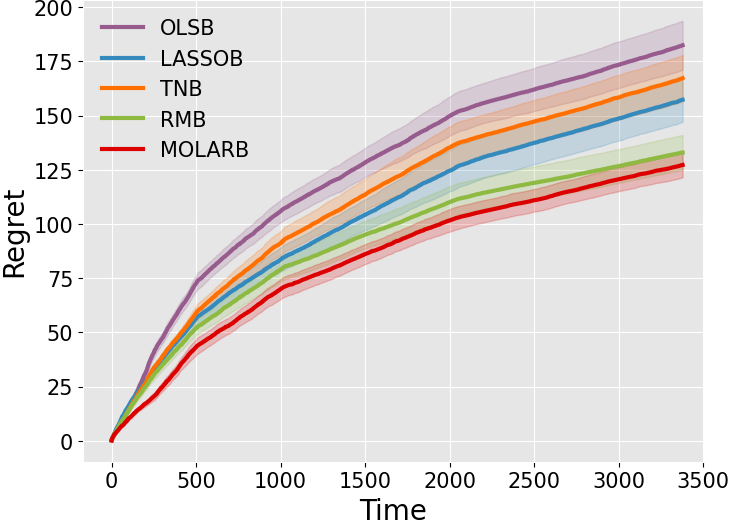}
    \hspace{-3.5mm}
    \includegraphics[height = 0.25\textwidth]{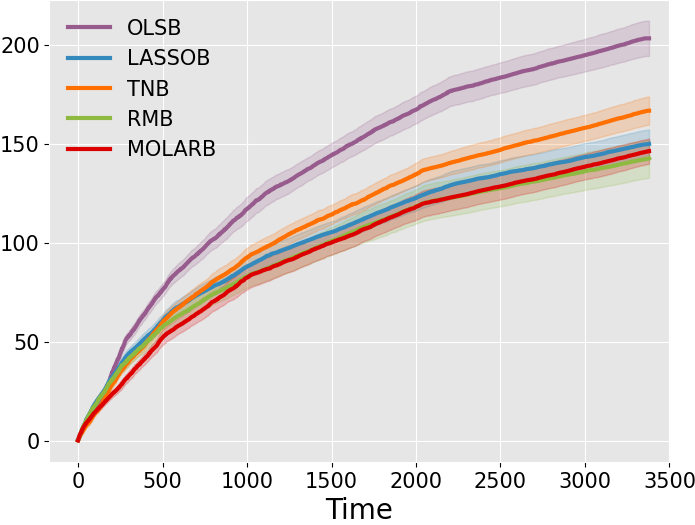}
    \hspace{-3.5mm}
    \includegraphics[height = 0.25\textwidth]{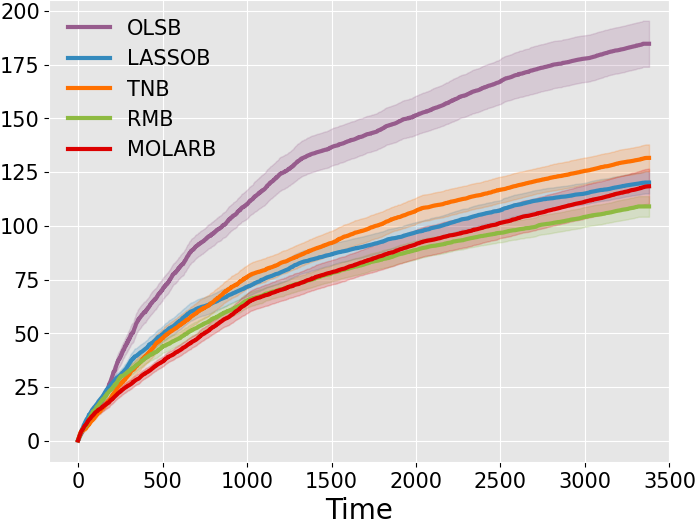}
    \includegraphics[height = 0.25\textwidth]{images/cumreg_PISA_top15_split_auto_m=3.png}
    \hspace{-3.5mm}
    \includegraphics[height = 0.25\textwidth]{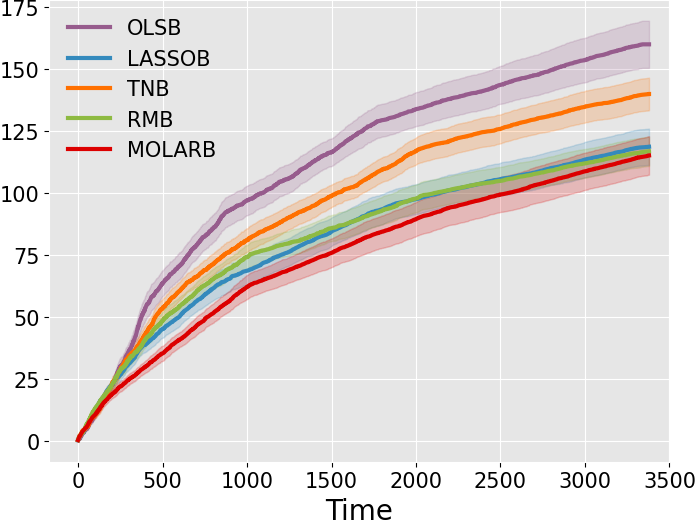}
    \hspace{-3.5mm}
    \includegraphics[height = 0.25\textwidth]{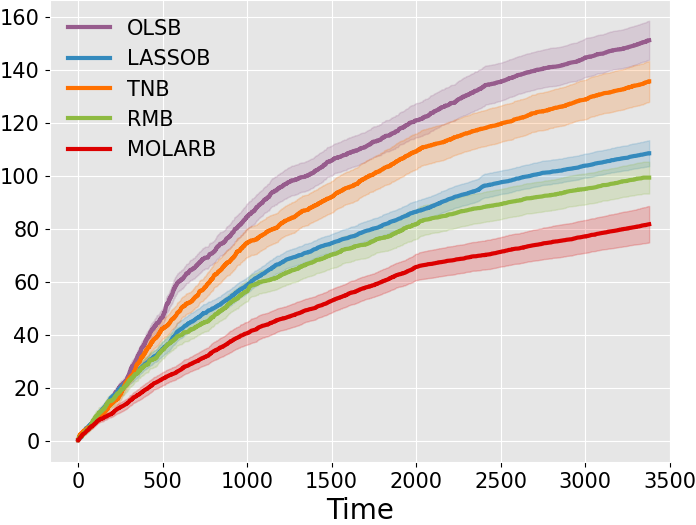}
    \includegraphics[height = 0.25\textwidth]{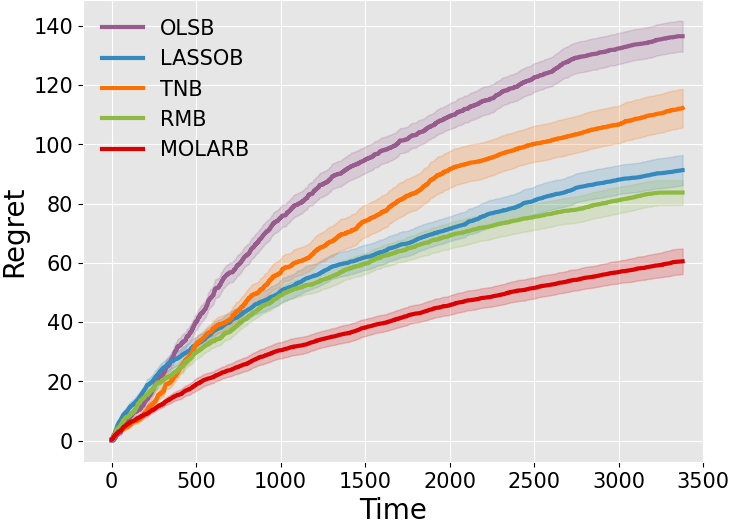}
    \hspace{-3.5mm}
    \includegraphics[height = 0.25\textwidth]{images/cumreg_PISA_top15_split_auto_m=7.png}
    \hspace{-3.5mm}
    \includegraphics[height = 0.25\textwidth]{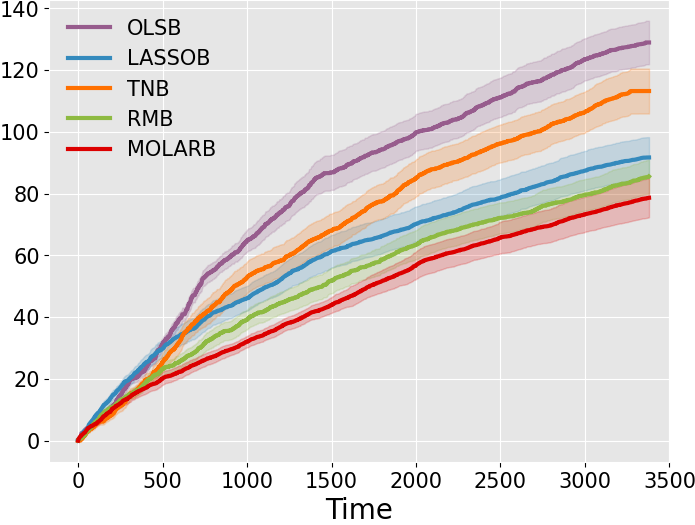}
    \includegraphics[height = 0.25\textwidth]{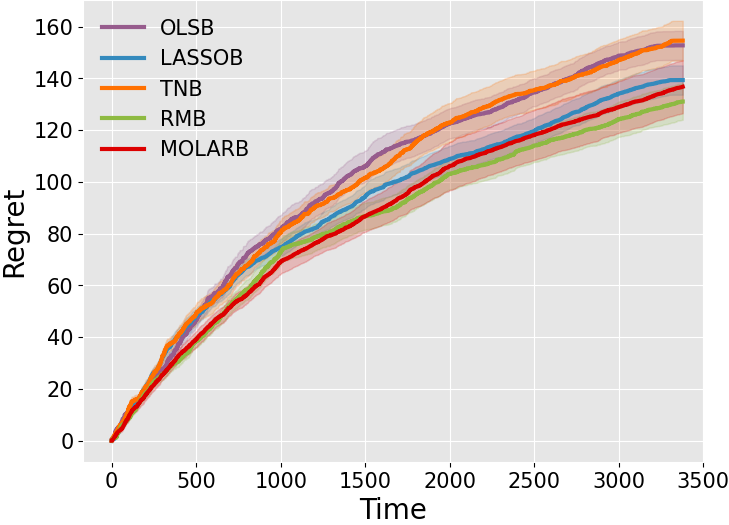}
    \hspace{-3.5mm}
    \includegraphics[height = 0.25\textwidth]{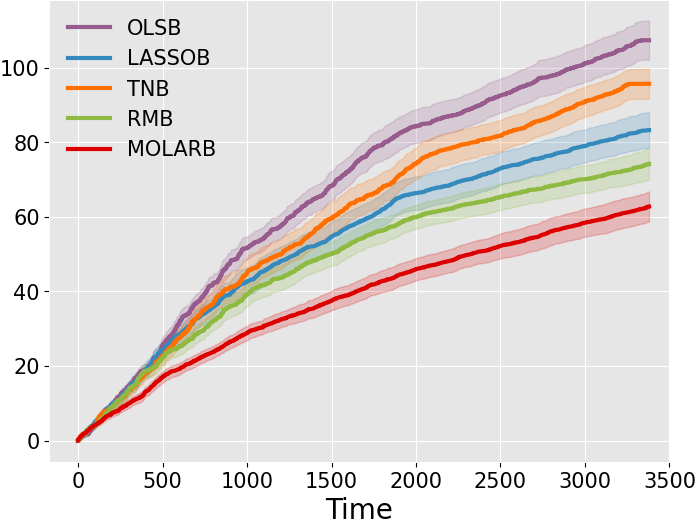}
    \hspace{-3.5mm}
    \includegraphics[height = 0.25\textwidth]{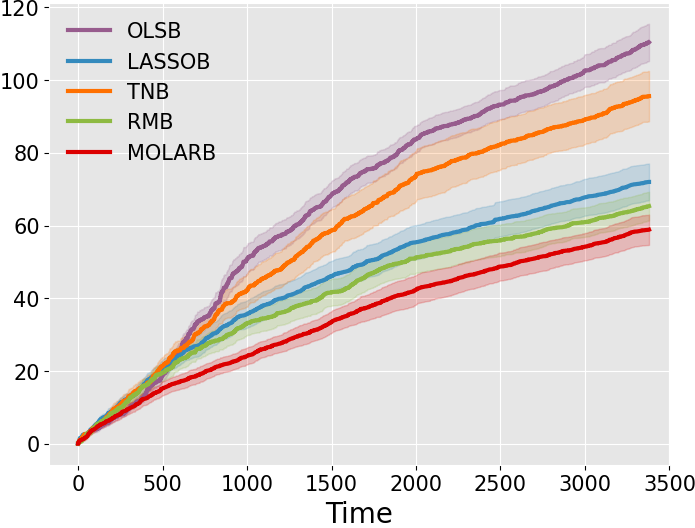}
    \includegraphics[height = 0.25\textwidth]{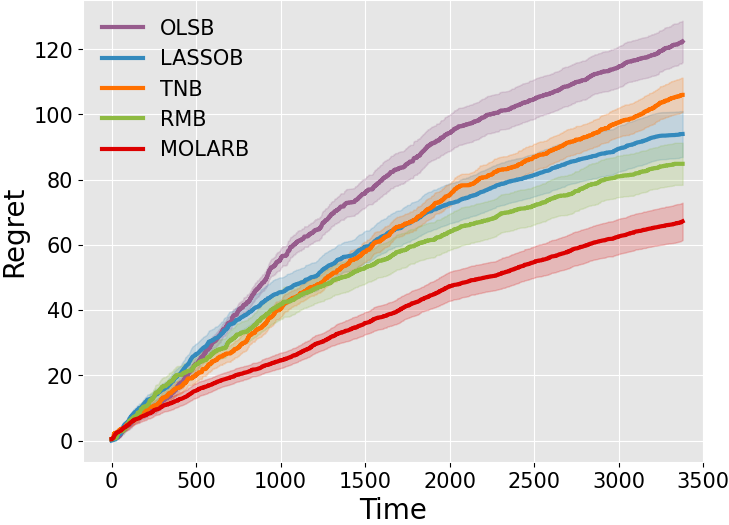}
    \hspace{-3.5mm}
    \includegraphics[height = 0.25\textwidth]{images/cumreg_PISA_top15_split_auto_m=13.png}
    \hspace{-3.5mm}
    \includegraphics[height = 0.25\textwidth]{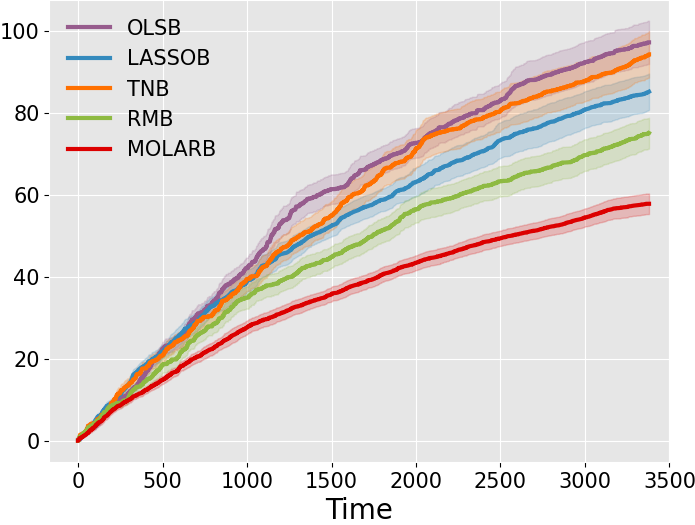}
    \caption{Regret $R_T^{(m)}$ of the top largest $15$ countires of the PISA dataset. The
    shaded regions depict the corresponding
$95\%$ normal confidence intervals based on the standard errors from twenty independent trials.}
    \label{fig:pisa-on-full}
\end{figure}

To visibly judge the heterogeneity of the processed datasets, we compute the  OLS estimates $\{\widehat{\beta}^{(m)}\}_{m=1}^M$ for all countries individually, along with a shared estimate $\widehat{\beta}^\star$ obtained by taking the covariate-wise median over the OLS estimates. 
We then plot the differences $|\delta^{(m)}|=|\widehat{\beta}^{(m)}-\widehat{\beta}^\star|$ with values below $6 \sqrt{[(\vX^{(m)\top}\vX^{(m)})^{-1}]_{k,k}}$ set as zero. Note that $\sqrt{[(\vX^{(m)\top}\vX^{(m)})^{-1}]_{k,k}}$ corresponds to the magnitude of the variation of OLS estimates under standard Gaussian noises.  The final result is shown in Figure \ref{fig:coef_plt}. 
Note that Figure \ref{fig:coef_plt} is plotted before splitting the data for the experiments.
In Figure \ref{fig:coef_plt}, we see that the differences in coefficients seem to be consistent with being sparse.

We take $K=2$ since our arms are to predict the student with a better mathematics score. 
To enable evaluation, we randomly split the data into a large test set, a small training set, and a small validation set with proportions ($90\%$, $5\%$, $5\%$)  and ($80\%$, $15\%$, $5\%$)  for offline and online experiments, respectively. 
We use the individual OLS estimates of the test set as proxies for true parameters and evaluate methods on the training set with parameters tuned based on the validation set. 

Once again, to ensure a fair comparison of methods, 
we set the hyperparameters $\lambda$ and $\gamma$ in the same manner as in the synthetic results, as suggested by theoretical results to achieve optimal rates of convergence. 
We only tune the numerical coefficients over a pre-specified grid $\{0.05, 0.35, 0.7, 1, 2\}$, and report the optimal results.
We run the bandit methods on this processed data in the same manner as in the simulations. Here we run MOLAR and MOLARB with the option of soft thresholding.  

We repeat experiments starting from data splitting for $100$ and $20$ times for offline and online experiments, respectively. 
The regrets for all $15$ countries (Mexico, Italy, Spain, Canada, Brazil, Australia, UK, UAE, Switzerland, Qatar, Colombia, Finland, Belgium, Denmark, and Jordan from top to bottom, left to right) are in Figure \ref{fig:pisa-on-full}.

\subsection{Ablation Studies}

\subsubsection{Robustness Examinations for \ours}\label{app:robust_offline} 
We first examine the robustness of \ours to the choice of $c_\gamma$. To this end, we repeat the experiments in Figure~\ref{fig:synthetic-ls} and simulate \ours with empirically estimated noise variances through
\begin{equation*}
    \hat{\sigma}_m:=\sqrt{\|\vX^{(m)}\widehat \beta_{{\rm ind}}^{(m)}-Y^{(m)}\|_2^2/(n_m-d)}
\end{equation*}
and varying $c_\gamma$.
The results are depicted in Figure~\ref{fig:synthetic-ls-rob}. By comparing \ours with the individual OLS estimates, we observe that \ours exhibits advantages for all values of the threshold. In particular, \ours is robust for slightly large $c_\gamma$).

\begin{figure}[h]
    \centering
    \hspace{-3mm}
    \includegraphics[height = 0.22\textwidth]{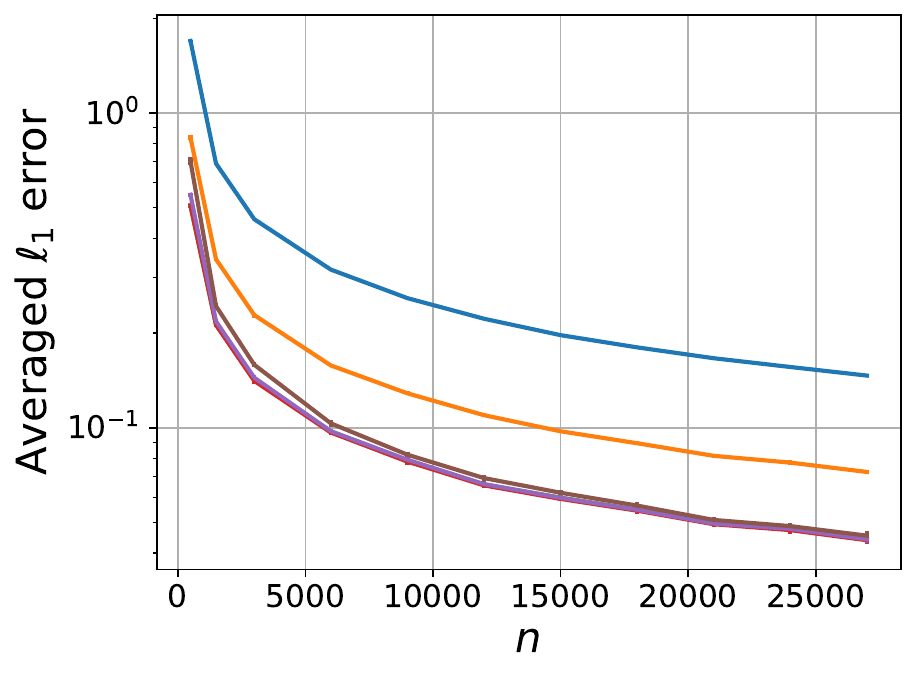}
    \hspace{-3mm}
    \includegraphics[height = 0.22\textwidth]{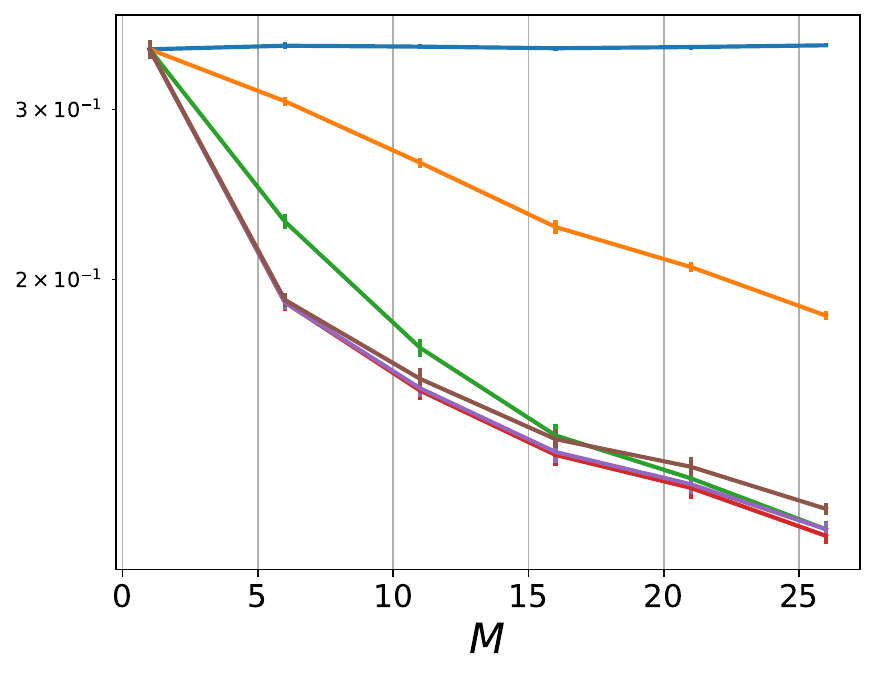}
    \hspace{-3mm}
    \includegraphics[height = 0.22\textwidth]{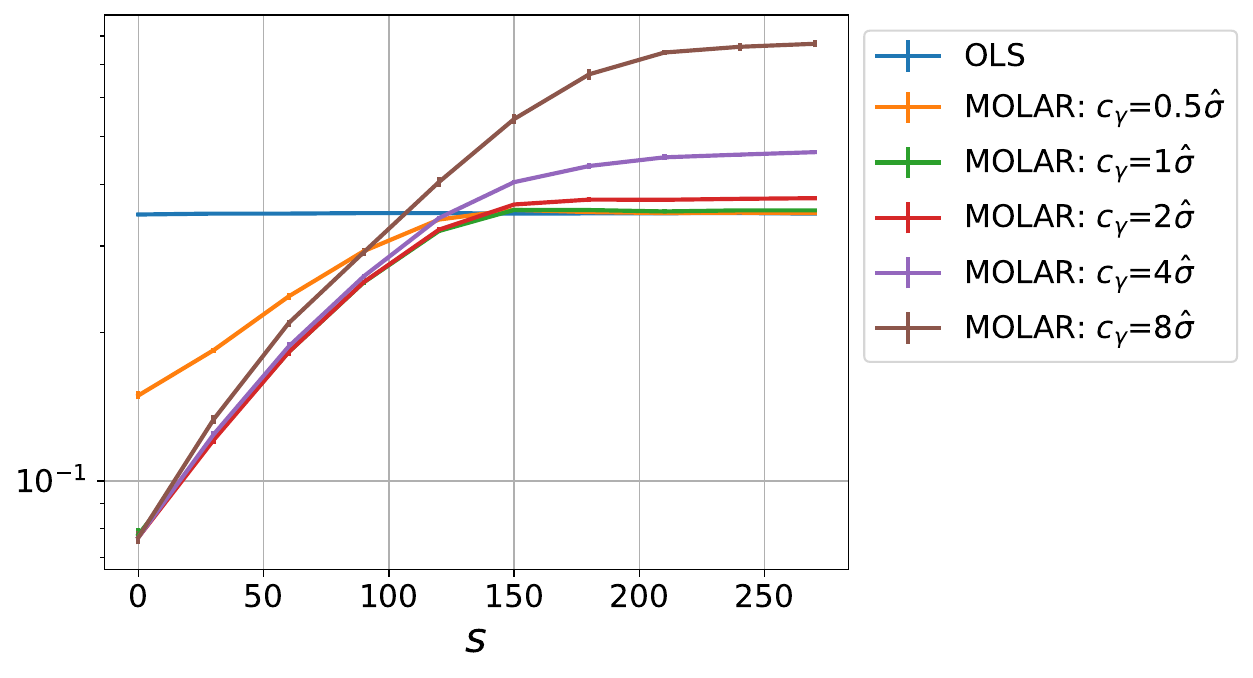}
    \hspace{-3mm}
    \caption{Average $\ell_1$ estimation error for \ours with varying thresholding parameters. (Left):  Fixing $s = 20, \,M = 30$ and varying $n$.
  (Middle): Fixing $s = 20,\,n = 5,000$ and varying $M$.
  (Right): Fixing  $M = 30, \,n = 5,000$ and varying $s$.
  The standard error bars are obtained from ten independent trials. 
  }
    \label{fig:synthetic-ls-rob}
\end{figure}

\subsubsection{Correlated Covariates \& Disparate Sample Sizes}\label{app:correlated-disparse}
To supplement the experiments in Figures~\ref{fig:synthetic-ls} and \ref{fig:synthetic-ls-rob} where $\Sigma^{(m)}=I_d$ and $n_m=n$ for all $m\in[M]$, we also conduct similar experiments for correlated covariates with disparate tasks-wise covariances and sample sizes.  
Here, for each task $m\in[M]$, we select the covariance matrix as $\Sigma^{(m)}=Q^{(m)}\mathrm{diag}((1+4(k-1)/d)_{k\in[d]})Q^{(m)\top}$ where $Q^{(m)}\in\RR^{d\times d}$ is a randomly generated orthonormal matrix. Since $\Sigma^{(m)}\neq I_d$, the covariates are correlated. 

For each pre-specified $n$, we determine the task-wise sample sizes $\{n_m\}_{m=1}^M$ by first drawing a Dirchlet random vector $(z_1,\dots,z_d)$ with $0\leq z_k\leq 1$ and $\sum_{k=1}^dz_k=1$, 
and then round $Mn z_k$ to obtain the sample size $n_m$. By doing this splitting, 
we roughly maintain $n_{[M]}=\sum_{m=1}^Mn_m\approx Mn$ but introduce significant disparity among $\{n_m\}_{m=1}^M$. 
We thus apply the weighted median 
to obtain the global estimate $\widehat\beta^\star$ with $w_m=n_m$ for all $m\in[M]$ and set other hyperparameters the same as in Section~\ref{sec:experi-detail}. The results are shown in Figure~\ref{fig:synthetic-ls-correlated}. Again, we observe a significant advantage of \ours over other baseline approaches.

\begin{figure}[h]
    \centering
    \hspace{-3mm}
    \includegraphics[height = 0.25\textwidth]{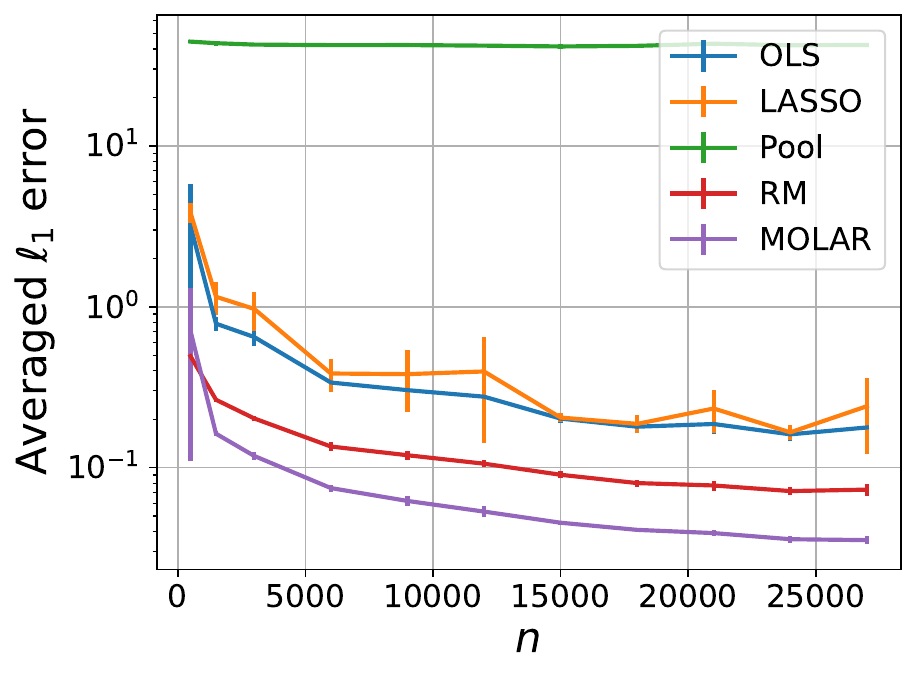}
    \hspace{-3mm}
    \includegraphics[height = 0.25\textwidth]{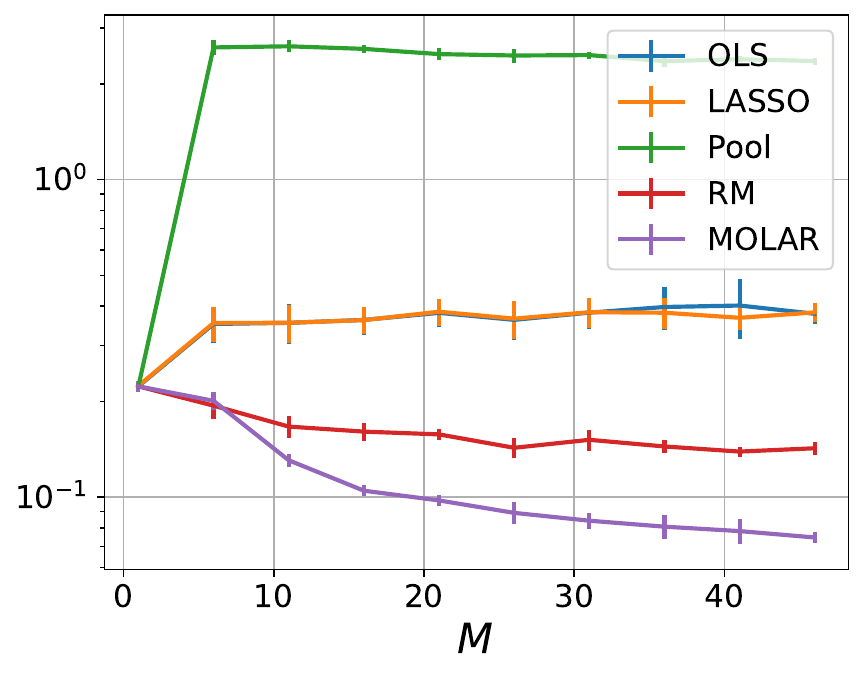}
    \hspace{-3mm}
    \includegraphics[height = 0.25\textwidth]{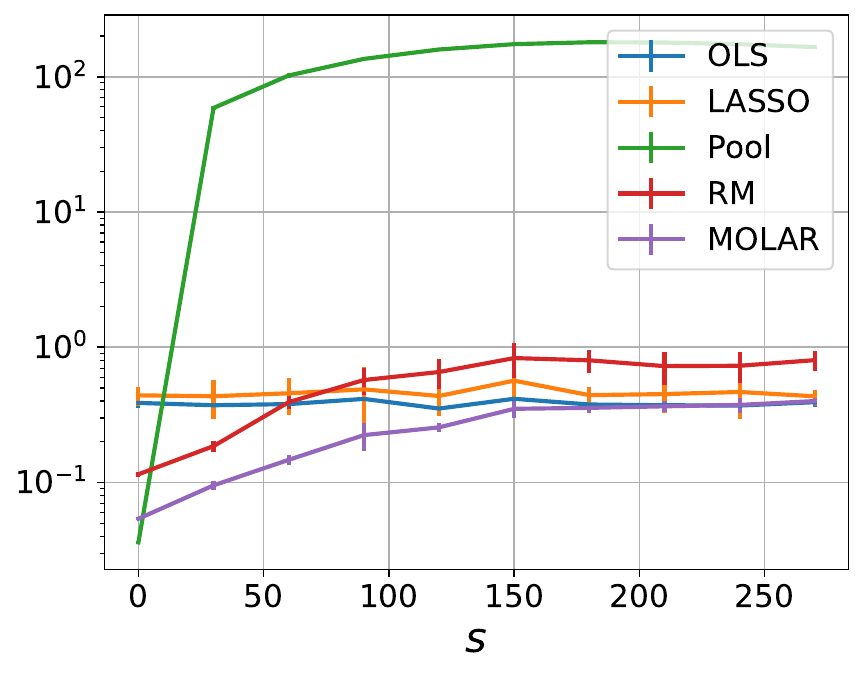}
    \hspace{-3mm}
    \caption{Average $\ell_1$ estimation error for multitask linear regression under correlated covariates with disparate task-wise covariances $\{\Sigma^{(m)}\}_{m=1}^M$ and sample sizes $\{n_m\}_{m=1}^M$. (Left):  Fixing $s = 20, \,M = 30$ and varying $n$.
  (Middle): Fixing $s = 20,\,n = 5,000$ and varying $M$.
  (Right): Fixing  $M = 30, \,n = 5,000$ and varying $s$.
  The standard error bars are obtained from ten independent trials. 
  }
    \label{fig:synthetic-ls-correlated}
\end{figure}

\subsubsection{Comparisong with LASSO Debiasing}\label{sec:compare_LASSO_debias}
The multitask approach proposed by~\cite{xu2021multitask}, while being sub-optimal, shares a similar algorithmic design by utilizing task-wise LASSO to debias a global trimmed-mean estimate. 
Given the optimality of \ours, one might be curious about which part (\ie, the median estimate or the coordinate-wise debiasing) is the key to attaining the optimality.

Experimentally, we compare our MOLAR, the robust multitask approach (RM)~\cite{xu2021multitask} which debiases over the global trimmed mean estimate with LASSO, and the counterpart that debiases over the global median estimate with LASSO, denoted by ML,  as in Figure~\ref{fig:synthetic-ML-compare}. The experimental setup is the same as in Figure 2 in the manuscript, except for the increased parameter $s$ for better visualization. By comparing ML with RM, we see that the median estimate performs consistently better than the trimmed mean estimate. On the other hand,  MOLAR outperforms ML, which suggests the empirical superiority of coordinate-wise debiasing. 
We have added the experiments and the below discussion in Appendix J.3.3 in our latest revision.

\begin{figure}[h]
    \centering
    \hspace{-3mm}
    \includegraphics[height = 0.25\textwidth]{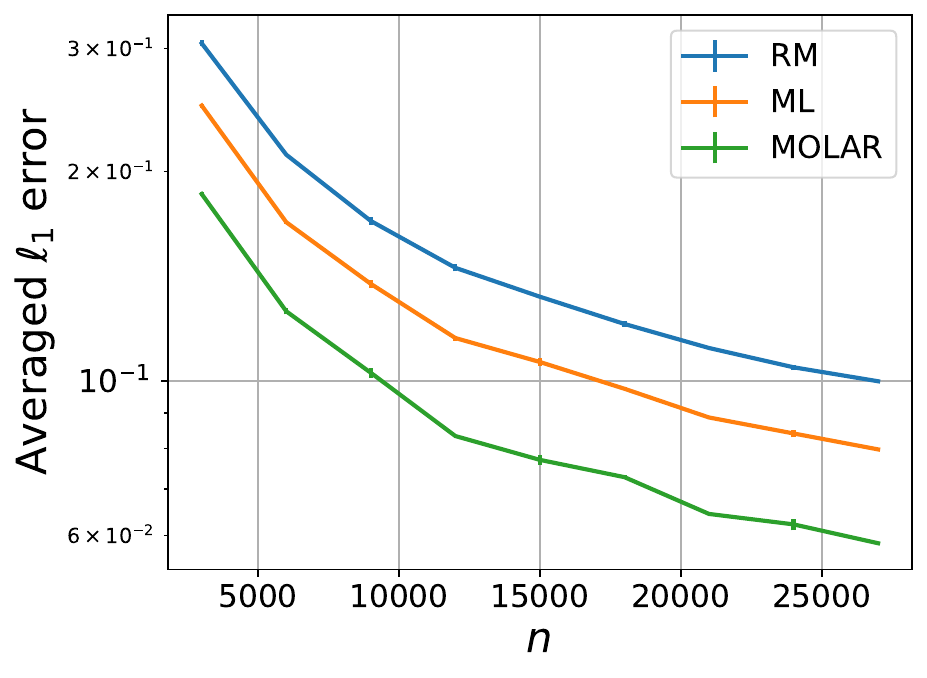}
    \hspace{-3mm}
    \includegraphics[height = 0.25\textwidth]{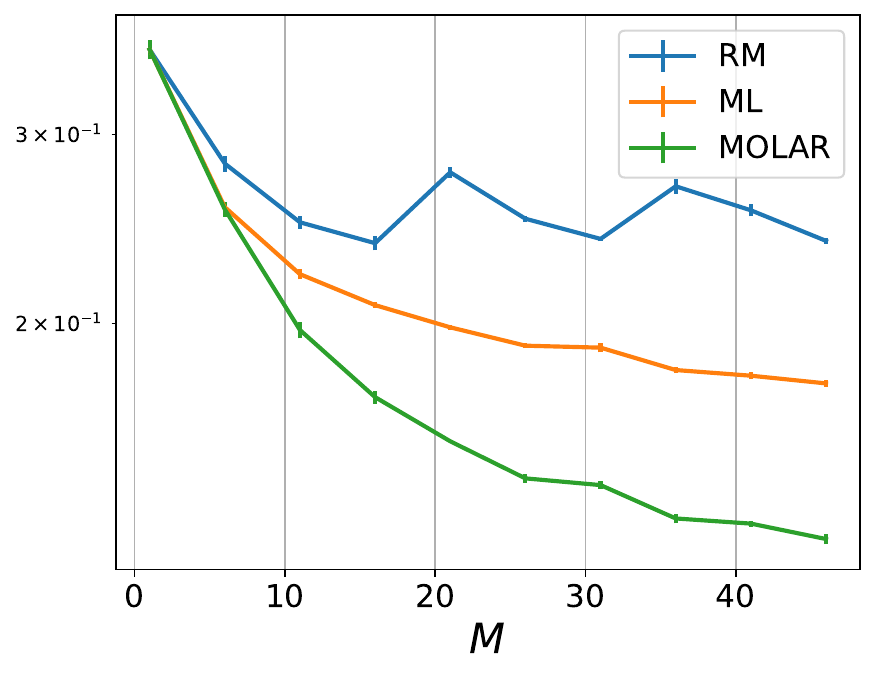}
    \hspace{-3mm}
    \includegraphics[height = 0.25\textwidth]{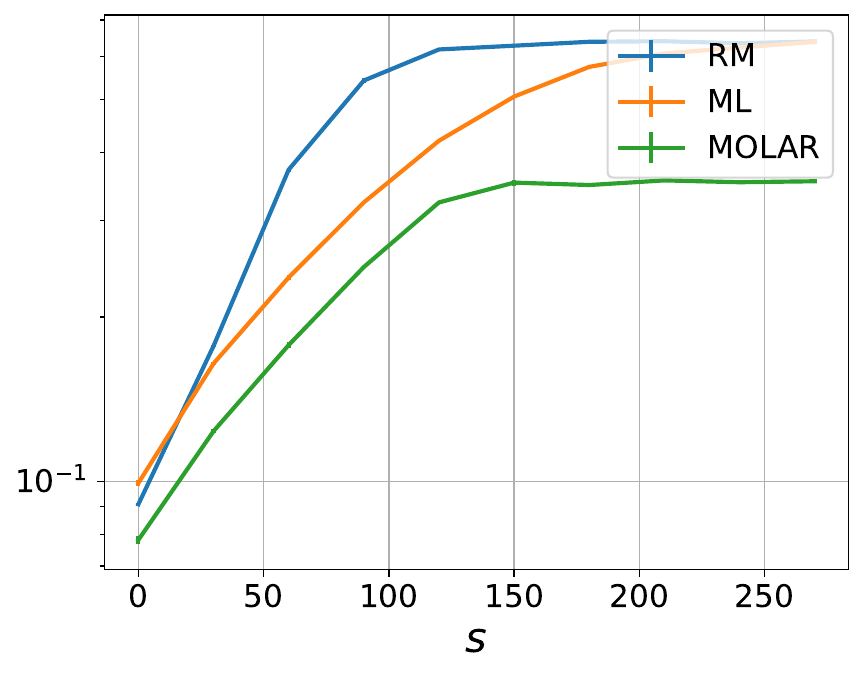}
    \hspace{-3mm}
    \caption{Average $\ell_1$ estimation error for \ours with varying thresholding parameters. (Left):  Fixing $s = 40, \,M = 30$ and varying $n$.
  (Middle): Fixing $s = 40,\,n = 5,000$ and varying $M$.
  (Right): Fixing  $M = 30, \,n = 5,000$ and varying $s$.
  The standard error (barely visible) bars are obtained from ten independent trials. 
  }
    \label{fig:synthetic-ML-compare}
\end{figure}

Theoretically, using LASSO-based debiasing over the global median estimate can reach the optimal rate for the $\ell_1$ error bound but may not do so for the $\ell_2$ error bound.   We take the special case of equal sample sizes and variances (\ie, $n_m=n$ and $\sigma_m = \sigma$ for all $m\in[M]$) as an  example. 

For the $\ell_1$ case, following Proposition 1, we have for all $k\in \cI_{1/5}$ that  $|\widehat{\beta}^{\star}_{k}-{\beta}^{\star}_{k}|=\widetilde{O}_P\left(\frac{|\cB_k|\sigma }{M\sqrt{n}} + \frac{\sigma}{\sqrt{Mn}}\right).$ Therefore, taking $\cG:= \cI_{1/5}$, the condition that $\|\widehat \beta^\star_{\cG}-\beta^\star_{\cG}\|_1=\widetilde{O}_P\left(\frac{s\sigma}{\sqrt{n}}+\frac{d\sigma}{\sqrt{Mn}}\right)$ for Corrolary E.1 is satisfied, where $|\cG^c|\leq 5s=O(s)$ due to the pigeon-hole principle. Therefore,  the optimal $\ell_1$ error $\|\widehat{\beta}^{(m)}-\beta^{(m)}\|_1=\widetilde{O}_P\left(\frac{s \sigma}{\sqrt{n}}+\frac{d\sigma}{\sqrt{Mn}}\right)$ for task-wise parameters can be obtained by using LASSO debiasing. In other words, for the $\ell_1$ case, it is mainly the median that makes optimality achievable.

For the $\ell_2$ error bound, the optimal error may not be achievable via LASSO, even if applied to the median estimate, because it amplifies the estimation error of the global estimate $\widehat \beta^\star$. Specifically, following the proof of Proposition E.1, with probability at least $1-\exp(-c_1\mu n)-(nd)^{-1}$ it holds that if $\|\widehat{\beta}^{\star}_{\cG}-\beta^{\star}_{\cG}\|_1\leq |\cJ_m|^{1/2}\|\widehat \beta^{(m)}-\beta^{(m)}\|_2$ where $\cJ_m\triangleq \cG^c\cup \mbox{supp}(\beta^{(m)}-\beta^\star)$, we have
\begin{equation}
    2\lambda_m|\cJ_m|^{1/2}\|\widehat\beta^{(m)}-\beta^{(m)}\|_2+2\lambda_m \|\widehat{\beta}^{\star}_{\cG}-\beta^{\star}_{\cG}\|_1\geq \frac{c_1}{6}\|\widehat \beta^{(m)}-\beta^{(m)}\|_2^2+\frac{\lambda_m}{2}\|\widehat \beta^{(m)}-\beta^{(m)}\|_1.
\end{equation}
Otherwise, we of course have $\|\widehat{\beta}^{\star}_{\cG}-\beta^{\star}_{\cG}\|_1> |\cJ_m|^{1/2}\|\widehat \beta^{(m)}-\beta^{(m)}\|_2$.
Consequently, we have 
\begin{align}
    \|\widehat \beta^{(m)}-\beta^{(m)}\|_2^2=&O_P\left(\lambda_m^2|\cJ_m|+\lambda_m \|\widehat{\beta}^{\star}_{\cG}-\beta^{\star}_{\cG}\|_1+\frac{\|\widehat{\beta}^{\star}_{\cG}-\beta^{\star}_{\cG}\|_1^2}{|\cJ_m|}\right)\\
    =&\widetilde O_P\left(\frac{(s+|\cG^c|)\sigma^2}{n}+\frac{\|\widehat{\beta}^{\star}_{\cG}-\beta^{\star}_{\cG}\|_1\sigma}{\sqrt{n}}+\frac{\|\widehat{\beta}^{\star}_{\cG}-\beta^{\star}_{\cG}\|_1^2}{s+|\cG^c|}\right).
\end{align}
Under the same condition $\|\widehat \beta^\star_{\cG}-\beta^\star_{\cG}\|_1=\widetilde{O}_P\left(\frac{s\sigma}{\sqrt{n}}+\frac{d\sigma}{\sqrt{Mn}}\right)$ and $|\cG^c|=O(s)$, this bound becomes 
\begin{align}
    \|\widehat \beta^{(m)}-\beta^{(m)}\|_2^2
    =&\widetilde O_P\left(\frac{s\sigma^2}{n}+\frac{(\sqrt{M}+ d/s) d\sigma^2}{Mn}\right),
\end{align}
which turns out to be suboptimal compared to the lower bound $\frac{s\sigma^2}{n}+\frac{d\sigma^2}{Mn}$.
We believe this issue is due to the dense shift $\beta^\star- \widehat \beta^\star$ in the $\ell_1$ penalty $\| \beta - \widehat \beta^\star\|_1=\| \beta -\beta^{\star} +\beta^\star- \widehat \beta^\star\|_1$ where only $\beta -\beta^{\star} $ is expected to be sparse (and the target is to estimate $\beta^{(m)}$). The $\ell_1$ form of the shift term leads to the suboptimality in controlling the $\ell_2$ error for the regularized estimate. 
Our claim is generally consistent with related literature on transfer learning~\cite{bastani2021predicting,li2023estimation} where LASSO debiasing is unable to simultaneously attain optimal rates for both the $\ell_1$ and $\ell_2$ cases. Therefore, for the $\ell_2$ case, both the median and the coordinate-wise debaising play a role in the optimality.

\subsubsection{Robustness Checks for \oursb}
We also provide a robustness check for \oursb by varying  $c_\gamma$ and $|H_0|$.
We investigate the cumulative expected regret of  \oursb while varying the first batch size $|H_0|\in\{1, 5, 10\}$, and the numerical coefficient $c_\gamma\in\{0.175, 0.35, 0.5,1\}$. The results, presented in Figure \ref{fig:synthetic-b-rob}, are computed in the same setup as the synthetic bandit simulations. We find that the cumulative regret performance of \oursb is not substantially impacted by changing the parameters by up to an order of magnitude.
This suggests that  \oursb is quite robust to $|H_0|$ and $c_\gamma$ in this range.

\subsubsection{Usage of Historical Batches}
Since \oursb empties all batch-wise buffers of contexts after using them to update estimates $\{\widehat\beta^{(m)}\}_{m=1}^M$, we also compare \oursb with a variant where all historical contexts in each arm are maintained. 
In this variant, we still use each brand-new batch $\{(\vX_q^{(m)},Y_q^{(m)})\}_{m=1}^M$ to collaboratively learn the global estimate $\widehat \beta^\star_q$,
yet the step of covariate-wise shrinkage in obtaining the task-wise estimates $\widehat\beta^{(m)}$ leverages all previous batches  $(\vX_{[q]}^{(m)},Y_{[q]}^{(m)})$ in the $m$-th bandit. The results of this variant are marked as ``use\_hist'' in Figure \ref{fig:synthetic-b-rob}. We find that \oursb does not lose significant sample efficiency, compared to this variant. 
This finding is consistent with our theoretical results
that \oursb is minimax optimal in this multi-task setup.

\begin{figure}[h]
    \centering
    \hspace{-3mm}
    \includegraphics[height = 0.24\textwidth]{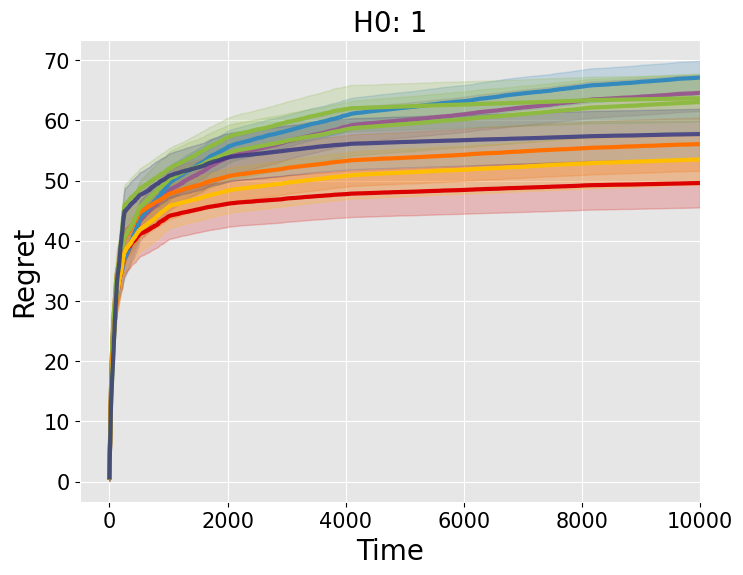}
    \hspace{-3mm}
    \includegraphics[height = 0.24\textwidth]{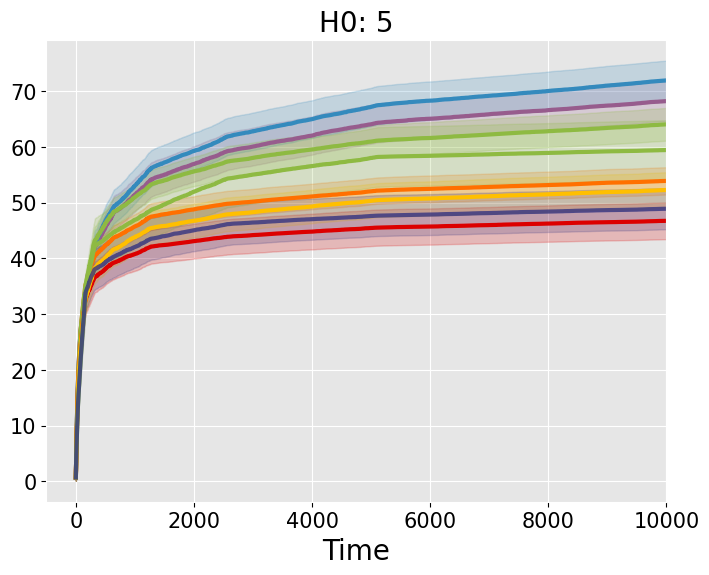}
    \hspace{-3mm}
    \includegraphics[height = 0.24\textwidth]{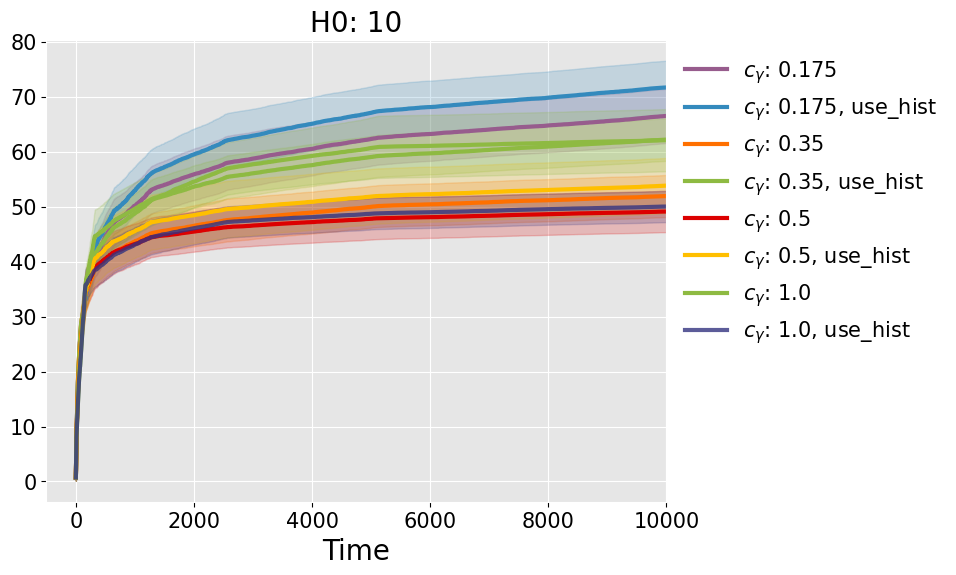}
    \hspace{-3mm}
    \caption{Regret $R_T^{(m)}$ accumulated by \oursb of an instance with activation probability $0.91$  with varying $|H_0|$ and tuning  coefficient $c_\gamma$, where
    shaded regions depict the corresponding
95\% normal confidence intervals based on standard errors calculated over twenty independent trials. 
    }
    \label{fig:synthetic-b-rob}
\end{figure}

\end{document}